\documentclass[11pt]{article}

\usepackage[margin=1in]{geometry}
\usepackage{amsmath,amssymb,amsthm,mathtools}
\usepackage{graphicx}
\usepackage{float}
\usepackage[authoryear,round,sort]{natbib}
\usepackage[colorlinks=true,citecolor=blue,linkcolor=blue,urlcolor=blue]{hyperref}
\usepackage{enumitem}
\usepackage{algorithm}
\usepackage{algpseudocode}

\newtheorem{theorem}{Theorem}[section]
\newtheorem{lemma}[theorem]{Lemma}
\newtheorem{corollary}[theorem]{Corollary}
\newtheorem{proposition}[theorem]{Proposition}
\newtheorem{definition}[theorem]{Definition}

\newtheorem*{informaltheorem}{Informal Theorem}
\newtheorem*{informalproposition}{Informal Proposition}

\newcommand{\calH}{\mathcal H}
\newcommand{\R}{\mathbb R}
\newcommand{\eps}{\varepsilon}
\newcommand{\conv}{\operatorname{conv}}
\newcommand{\spanop}{\operatorname{span}}
\newcommand{\Root}{\operatorname{Root}}
\newcommand{\edge}{\operatorname{edge}}

\title{Defensive Boosting for Online Probabilistic Forecasting}
\author{Georgy Noarov \and Aaron Roth}
\date{}
\hypersetup{
  pdftitle={Defensive Boosting for Online Probabilistic Forecasting},
  pdfauthor={Georgy Noarov and Aaron Roth}
}

\begin{document}
\maketitle

\begin{abstract}
We study online probabilistic forecasting of binary outcomes chosen by an adaptive adversary. Given an online learning algorithm for a weak hypothesis class $\calH$, we would like to efficiently obtain two incomparable guarantees that existing online boosting techniques provide separately. Online gradient boosting competes in Brier score with the best predictor induced by the span of $\calH$ on every sequence --- but promises nothing when the span does not contain an accurate predictor.  Online weak-to-strong boosting drives classification error to zero under a weak-learning condition, but promises little when that condition fails.

We give a simple defensive forecasting algorithm, the \emph{Defensive Booster}, that obtains both guarantees.  On every adaptive sequence, its Brier score is competitive with the best prediction induced by the span of $\calH$ at the same rate as online gradient boosting; simultaneously, whenever the realized transcript satisfies the smooth weak-learning condition, its Brier score and randomized classification error satisfy the same rate guarantee as online classification boosting. 
This is achieved by operationalizing the \emph{dual view} of boosting: When the Defensive Booster's randomized classification error is persistently high, its mistake weights form a smooth reweighting on which every weak hypothesis has low edge, yielding an ex-post \emph{hard-core} certificate that the weak-learning condition fails. 
We also develop a strongly adaptive variant, which satisfies both guarantees and provides local hard-core certificates on every time interval. The Defensive Booster is very efficient: it accesses just one weak-class learner, whereas the prior online boosting methods we compare against maintain large weak-learner ensembles. Experiments on synthetic and real data streams demonstrate its strong predictive performance (sometimes substantially improving over all prior baselines) coupled with orders-of-magnitude faster runtime.
\end{abstract}

\section{Introduction}
\label{sec:intro}

On each round $t$ of an online probabilistic forecasting problem, an
adversary reveals a context $x_t\in\mathcal X$, the learner announces a
probability $p_t\in[0,1]$ for the binary outcome $Y_t\in\{0,1\}$, and the
outcome is then revealed.  The sequence may be arbitrary and adaptive
to the learner's past predictions; we make no distributional assumptions. The forecast is scored by the Brier score
$(Y_t-p_t)^2$: the squared-error proper scoring rule, minimized in
expectation by the true conditional probability that $Y_t = 1$.  

We study this problem through the lens of boosting. The learner is
given an online learning algorithm for simple, ``weak'' probability
predictors and wants to make forecasts that are more accurate than any weak
learner can make alone.  There are two established ways to turn an
online weak learner into a stronger one, and they come with incomparable
guarantees.  To put both guarantees in common notation, encode a weak
prediction $q(x)\in[0,1]$ as $h(x)=2q(x)-1\in[-1,1]$, and let $\calH$
denote the resulting class; likewise, encode $Y_t$ as $\sigma_t=2Y_t-1$.
This normalization is useful for binary classification because it expresses
classification advantage as correlation between $h(x_t)$ and $\sigma_t$.

\emph{Online gradient boosting}
\citep{BeygelzimerHazanKaleLuo2015OnlineGradient} treats boosting as online
convex optimization over combinations of weak hypotheses.  Run with squared
loss, it guarantees Brier score competitive with the best predictor in the
convex hull of $\calH$, or more generally in the norm-bounded span.  This guarantee is assumption free in that it holds on \emph{every} sequence --- but of course there is no guarantee that there is an accurate predictor in the span.

\emph{Online weak-to-strong boosting}
\citep{ChenLinLu2012OnlineBoosting,BeygelzimerKaleLuo2015OptimalAdaptive}
instead obtains the ``AdaBoost phenomenon'' in the online setting under a
smooth weak-learning condition: every sufficiently smooth
reweighting of the examples, meaning one whose weight is not concentrated on
too few examples, admits a hypothesis with edge $\gamma$ over random
guessing.  With the encoding above, the edge of $h$ under weights $w_t$ is its
normalized weighted correlation
$|\sum_t w_t\sigma_t h(x_t)|/\sum_t w_t$.  For a binary-valued weak learner,
edge $\gamma$ is equivalent, after possibly flipping its sign, to weighted
classification error at most $(1-\gamma)/2$.  When the condition
holds down to the smoothness needed for the target accuracy, boosting drives
classification error to zero.  The resulting
classification accuracy can far exceed what squared-loss competition with
the span of $\calH$ alone guarantees.  But when the weak-learning
condition fails, these algorithms promise little, and
their natural output is a weighted vote over an ensemble of predictors rather than a probability.

\emph{This paper asks whether a single, natural, efficient online algorithm, outputting probability
forecasts, can enjoy both guarantees at once: the unconditional comparator
guarantee of gradient boosting, and the conditional weak-to-strong
guarantee of classification boosting. }

\subsection{Our results}
\label{sec:intro-results}

We answer affirmatively with a simple, efficient algorithm --- the
\emph{Defensive Booster} (Algorithm~\ref{alg:defensive}) ---
built as a black-box reduction from an online learning algorithm for the weak class $\calH$.  Write
$B_T=T^{-1}\sum_{t\le T}(Y_t-p_t)^2$ for its average Brier score.  Our main
guarantees, stated informally, follow. First, we unconditionally obtain the same guarantee as online gradient boosting, at the same rate:

\begin{informaltheorem}[Gradient-boosting-style span guarantee; Theorem~\ref{thm:span-guarantee}]
On every adaptive sequence, for every $f$ in the
$\Lambda$-norm-bounded span of $\calH$, define
$q_f(x)=(1+f(x))/2$.  Then
\[
  B_T
  \le
  \frac1T\sum_{t=1}^T
  \left(Y_t-q_f(x_t)\right)^2
  +
  O\!\left(\frac{\Lambda}{\sqrt T}\right).
\]
\end{informaltheorem}

Here $q_f=(1+f)/2$ reverses the affine encoding above and, for a general
span comparator, is an unrestricted real-valued score.  The actual bound is
second-order --- the regret term scales with the forecaster's own Brier
score rather than with $T$ --- yielding a fast $O(1/T)$ bound in the
realizable span case (Corollary~\ref{cor:low-loss-span}).

Next, we simultaneously obtain the guarantee of online weak-to-strong boosting, transforming a smooth weak-learning condition into perfect classification with the same $\gamma,\eps$ dependence that is optimal in the weak-online-learning model of \citet{BeygelzimerKaleLuo2015OptimalAdaptive}:

\begin{informaltheorem}[Weak-to-strong boosting guarantee; Corollary~\ref{cor:second-order-boosting}]
If the realized transcript satisfies the $(\rho,\gamma)$-smooth weak-learning
condition --- every reweighting $w_t\in[0,1]$ of the realized rounds with
average weight at least $\rho$ admits some $h\in\calH$
with normalized edge at least $\gamma$ ---
then the forecaster's Brier score and randomized classification error $\frac1T\sum_{t=1}^T |Y_t-p_t|$
are both at most $\max\{\rho,\ \tilde O(1/(\gamma^2T))\}$.  Consequently,
for any target $\eps>0$, if the weak-learning condition holds with $\rho=O(\eps)$, then
both errors are at most $\eps$ after
$T=\tilde O(1/(\gamma^2\eps))$ rounds.
\end{informaltheorem}

Thresholding each forecast at $1/2$ gives a deterministic classifier whose
average $0/1$ error is at most twice the randomized error.

The weak-to-strong guarantee has a complementary certificate
(Theorem~\ref{thm:main-hard-core}).  If the forecaster's error remains large
for long enough, its mistake weights $w_t=|Y_t-p_t|$ form a
\emph{hard-core witness}: a smooth reweighting of the realized rounds on
which every weak hypothesis has low edge.  Thus persistent error explicitly
certifies that the weak-learning condition fails on the realized transcript.

Both guarantees can also be made strongly adaptive.  A variant using
$O(\log T)$ active copies of the same weak-class oracle satisfies both
guarantees, up to polylogarithmic factors, simultaneously on every contiguous
interval (Section~\ref{sec:interval-boosting}).  On each interval it competes
with the best span comparator for that interval; if the smooth weak-learning
condition holds on the interval, it obtains the strong-learning guarantee
there.  It also localizes the certificate above: whenever error remains large
on an interval, the mistake weights restricted to that interval form a local
hard-core witness.  Because the interval may be chosen after observing the
transcript, this identifies where and when the smooth weak-learning condition
fails, even if the condition holds on the full sequence.

The two guarantees are genuinely different: neither implies the other, and
prior work that provides either guarantee in isolation does not provide both.

\begin{informalproposition}[Separation; Appendix~\ref{sec:separation}]
Neither guarantee implies the other.  In one direction, for arbitrarily small
constants $\gamma>0$, there are binary-valued weak classes and transcripts on
which every reweighting has edge at least $\gamma$ --- so the weak-to-strong
guarantee forces vanishing Brier score and randomized error --- yet every
 score induced by the span has squared loss bounded below by a
constant.  For every fixed coefficient-norm budget, a constant lower bound
also remains after clipping the scores to valid probabilities.  Conversely,
there are transcripts with an arbitrarily small-loss comparator in the span
but a smooth reweighting on which every weak hypothesis has zero edge, so the
smooth weak-learning condition fails.
\end{informalproposition}

We also evaluate the Defensive Booster empirically
(Section~\ref{sec:experiments}): on synthetic datasets engineered to favor
either gradient boosting or weak-to-strong classification boosting, and on
four real binary prediction datasets.  The stronger baseline family depends
on the instance: gradient boosting can substantially outperform
weak-to-strong boosting, and vice versa.  On every instance, the Defensive
Booster is competitive with, and often outperforms, the stronger baseline,
while the ensemble baselines take $20$--$66\times$ as much time per round:
it maintains one online weak learner, whereas they each maintain $100$.

A more naive alternative is to run instances of each kind of comparison
booster in parallel and combine their probability forecasts with an online
aggregator such as multiplicative weights.  We include this method as a
baseline.  It must run every constituent booster, while the Defensive Booster
runs one weak-class learner.  Moreover, aggregation by Brier loss gives a
weaker classification-error guarantee: a Brier guarantee controls randomized
classification error only through
$T^{-1}\sum_t|Y_t-p_t|\le\sqrt{B_T}$, so an $O(\eps)$ Brier guarantee yields
only $O(\sqrt\eps)$ randomized error, whereas the Defensive Booster directly
guarantees $O(\eps)$ under the smooth weak-learning condition.  The
aggregation scheme also provides no hard-core witness.

Finally, we note that our algorithm also handles arbitrary bounded real-valued outcomes, and its
squared-loss span guarantee holds unchanged, just as it does in the binary setting.
Appendix~\ref{sec:bounded-outcomes} describes this in more detail, and evaluates this extension on three
chronological regression streams.

\subsection{Technique: playing the dual side of the boosting game}
\label{sec:intro-technique}

\paragraph{The boosting game and its two views.}
In the batch setting, weak-to-strong boosting can be viewed as a zero-sum
game between a \emph{learner} player, who plays distributions over the weak
class $\calH$, and a \emph{data} player, who plays reweightings of the
dataset \citep{FreundSchapire1996Game}.  A winning strategy for the learner
player is a distribution over weak hypotheses whose weighted majority vote
attains perfect classification.  A winning strategy for the data player is
a \emph{hard-core distribution}: a reweighting of the data on which no weak
hypothesis has nontrivial edge \citep{Impagliazzo1995HardCore}.  The weak
learning assumption says that the data player has no sufficiently smooth
winning strategy, and minimax duality then supplies a winning majority vote
for the learner player.  

Existing online weak-to-strong boosting operationalizes the \emph{primal}
view: run many copies of the weak learner in parallel and learn a weighted
combination of their predictions
\citep{ChenLinLu2012OnlineBoosting,BeygelzimerKaleLuo2015OptimalAdaptive}.
We operationalize the \emph{dual} view.  Our forecaster never forms an
ensemble: it maintains one online learner for $\calH$ and two scalar
adaptive-gradient states, for a per-round cost of one oracle call plus
$O(1)$ arithmetic.  It makes probability forecasts that are --- in particular ---
\emph{multiaccurate} with respect to $\calH$
\citep{HebertJohnsonKimReingoldRothblum2018,KimGhorbaniZou2019}: no weak
hypothesis correlates with its forecast residuals.  Multiaccuracy implies
that weighting each round by the forecaster's randomized prediction error
yields a hard-core distribution for $\calH$.  This is the
same correlation-to-hard-core principle underlying the complexity-theoretic
regularity lemma of \citet{TrevisanTulsianiVadhan2009Regularity} --- see also recent work deriving related guarantees from strengthenings of multiaccuracy like calibrated multiaccuracy and multicalibration 
\citep{CasacubertaGopalanKanadeReingold2025,CasacubertaDworkVadhan2024}.
If the prediction error is high, the weights are large on average and so this
hard-core distribution is smooth in the standard boosting sense.  If the
smooth weak-learning condition holds --- weak hypotheses
have nontrivial edge on every smooth reweighting --- no such distribution
exists.  By contrapositive, the randomized classification error must be
low.  The remainder of this subsection makes each step concrete.

\paragraph{Defensive forecasting with orthogonality auditors.}
Using the encoding above, write $\mu_t=2p_t-1$ for the signed forecast and
$r_t=\sigma_t-\mu_t=2(Y_t-p_t)$ for the residual.  The design principle is
\emph{defensive forecasting}
\citep{VovkTakemuraShafer2005DefensiveForecasting}: rather than minimizing
a loss, the forecaster chooses $p_t$ so that a designated family of
statistical tests --- \emph{auditors} --- cannot accumulate evidence
that the forecasts differ from true probabilities.  We use two kinds of
auditors.  A weak-class auditor enforces multiaccuracy with respect to
$\calH$: hypotheses $h\in\calH$ should have small empirical correlation
with the residuals $r_t$.  A self-auditor enforces self-orthogonality:
the forecast $\mu_t$ itself should have small empirical correlation with
its own residuals.  Self-orthogonality is implied by, but substantially weaker
than \emph{calibration}, which is important, as calibration is impossible to obtain in the online setting at the rates we desire \citep{QiaoValiant2021Sidestepping,DaganDaskalakisFishelsonGolowichKleinbergOkoroafor2025Calibration}. An adaptive-gradient
procedure maintains a convex combination of the auditors, which we call the
\emph{aggregated auditor}.  On each round the
forecast $\mu_t$ is chosen by a one-dimensional \emph{root rule}: a point where
the aggregated auditor gain, viewed as a function of the forecast, changes
sign, so that the realized gain is nonpositive no matter how the label is
realized (Lemma~\ref{lem:root-sign}); this is a simple instantiation of the
more general framework recently introduced by
\citet{FarinaPerdomo2026BlackBox}.  On every sequence,
$\sup_{h\in\calH}|\sum_t h(x_t)r_t|$ --- multiaccuracy --- and
$|\sum_t \mu_tr_t|$ --- self-orthogonality --- are each at most
$A\sqrt{S_T}+B$, with its own constants determined by the corresponding
online-learning primitive, where $S_T=\sum_tr_t^2$ is the residual energy
(Theorem~\ref{thm:residual-certificate}).

\paragraph{High error yields a smooth hard-core witness.}
The residuals describe the forecaster's own randomized classification
mistakes.  Identify the forecast $p_t$ with the randomized classifier that
predicts $1$ with probability $p_t$.  Its conditional mistake probability is
$w_t=|Y_t-p_t|$, and these mistake weights satisfy the key identity
$w_t\sigma_t=r_t/2$.  Interpret $w_t$ as weights for a reweighting of the
transcript.  Their ``density'' $\rho_w=T^{-1}\sum_t w_t$ is exactly the
randomized classification error, while the identity converts multiaccuracy
into an edge bound for the mistake weighting: no weak hypothesis correlates
nontrivially with it.  So if the forecaster's randomized error is large
enough for $w$ to be smooth, then $w$ is a smooth reweighting of the
transcript on which no weak learner has nontrivial ``edge'' over random
guessing --- exactly the kind of winning strategy for the data player that
the smooth weak-learning condition rules out.  By contrapositive, the
forecaster must have sufficiently low error to avoid this contradiction,
and this is what turns the weak-learning condition into a strong-learning
guarantee.  The self-bounding form of the certificate (error term
$\sqrt{S_T}$, with $S_T\le4T\rho_w$) turns this qualitative contradiction
into the target $1/(\gamma^2\eps)$ rate
(Theorem~\ref{thm:main-hard-core}).  Self-orthogonality is not used in this
part of the argument; the hard-core conclusion follows from multiaccuracy
alone.

\paragraph{Adding self-orthogonality gives the span guarantee.}
On the other hand, it is known that multiaccuracy with respect to $\calH$
together with self-orthogonality gives squared error competitive with
every model in the span of $\calH$: these are exactly the first-order
optimality conditions for squared loss (see e.g. its use in
\citep{KearnsRothRyu2025Networked}).  For a span comparator $f$, the excess
squared loss of our forecasts over $f$ is controlled by two correlation
terms: the residuals against $f$, which multiaccuracy bounds because $f$ is
a linear combination of weak hypotheses, and the residuals against our own
forecasts, which self-orthogonality bounds.  Self-orthogonality is one
additional scalar constraint, enforced by one additional auditor at no cost
in the rate, and the result is Brier loss competitive with the best
predictor in the span on every sequence (Theorem~\ref{thm:span-guarantee}).

\paragraph{Empirical preview.}
Figure~\ref{fig:intro-experiment} previews the real-data behavior for both
binary probability forecasting and bounded regression; the
main comparisons are in Section~\ref{sec:experiments}, and the full protocol
and additional experiments are in
Appendices~\ref{sec:experimental-appendix}--\ref{sec:adaptive-diagnostics}.
The binary baselines are
\textsc{OGB} \citep{BeygelzimerHazanKaleLuo2015OnlineGradient} for online
gradient boosting; \textsc{Online BBM} and \textsc{AdaBoost.OL}
\citep{BeygelzimerKaleLuo2015OptimalAdaptive} and
\textsc{OSBoost} \citep{ChenLinLu2012OnlineBoosting} for online weak-to-strong
boosting; and a Brier-loss aggregator that combines the four ensemble
forecasts by multiplicative weights.  The regression comparison is with OGB,
the baseline providing the corresponding squared-loss span guarantee.

\begin{figure}[!t]
\centering
\includegraphics[width=.96\linewidth]{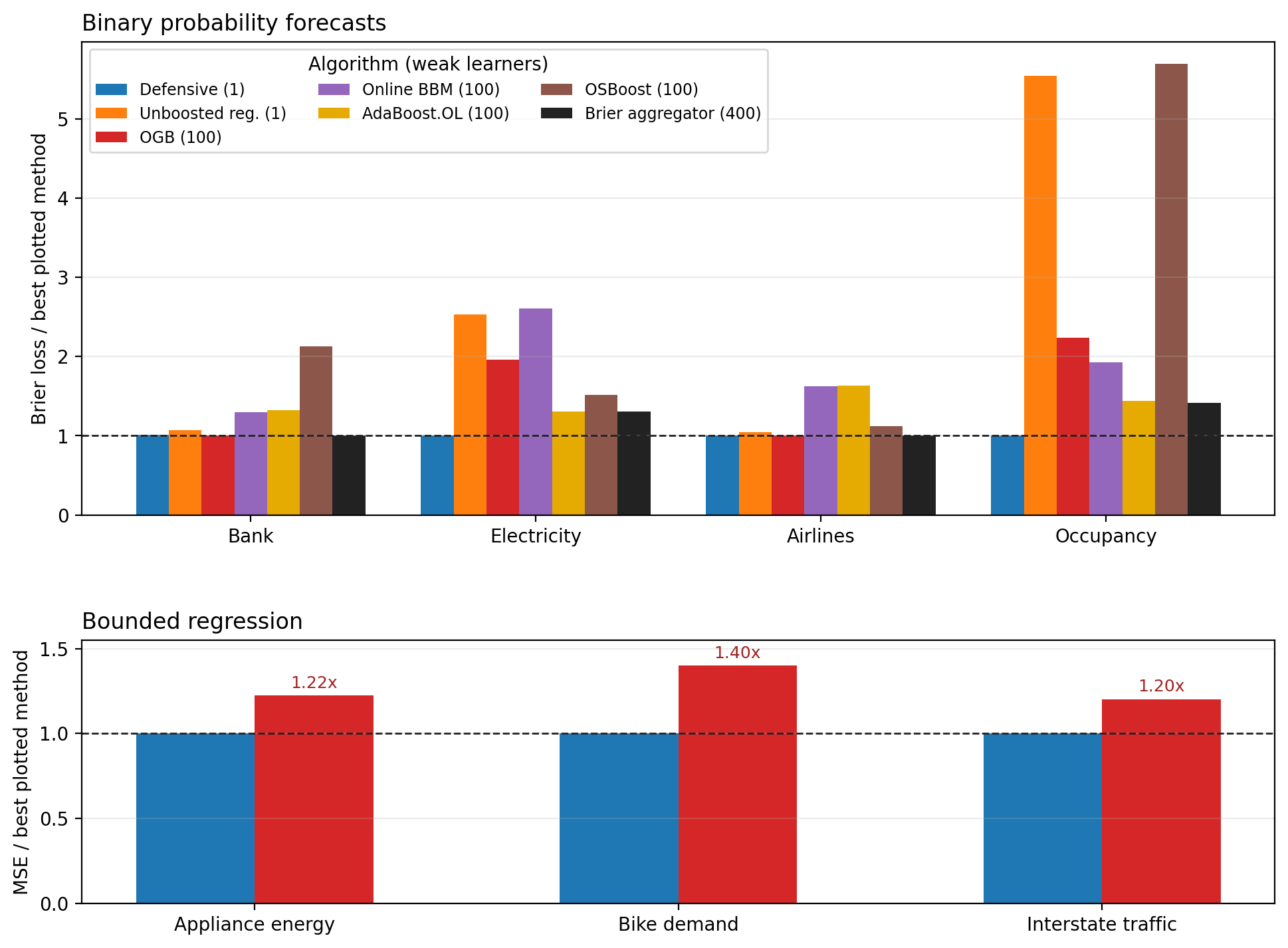}
\caption{Empirical preview (details in Section~\ref{sec:experiments}).
\emph{Top:} For each of four naturally ordered real binary streams, each
method's Brier loss is divided by the smallest loss among the methods shown.
The Defensive Booster
has the lowest Brier loss on Electricity and Occupancy, where it substantially
outperforms even the $400$-learner aggregator.  It essentially ties the
aggregator and OGB on Airlines and is within $.0010$ of the aggregator on
Bank, while maintaining one weak-class learner.  The hard-label unboosted
classifier is omitted for scale and reported in
Table~\ref{tab:real-brier-main}.  \emph{Bottom} (see Appendix~\ref{sec:regression-experiments}): On three chronological
bounded-regression streams, mean squared error is divided by the smaller loss
of the Defensive Booster and OGB.  The Defensive Booster reduces mean squared
error by $17$--$29\%$ while maintaining one weak learner rather than $100$.
In both panels lower is better, dashed lines mark the best plotted value, and
parentheses give the number of maintained online weak learners.}
\label{fig:intro-experiment}
\end{figure}

\paragraph{Organization.}
Section~\ref{sec:setting} defines the setting, the weak-class oracle, and the
fixed scalar update.  Section~\ref{sec:algorithm} gives the algorithm and the
multiaccuracy/self-orthogonality certificate.
Section~\ref{sec:main-guarantees} derives the Brier/span guarantee, the
hard-core mistake weighting, and the weak-to-strong corollary.
Section~\ref{sec:interval-boosting} gives a strongly adaptive variant whose
guarantees hold on every interval.
Section~\ref{sec:experiments} reports synthetic and real-data experiments,
and Section~\ref{sec:related-work} discusses related work.
Appendix~\ref{sec:separation} proves that the two guarantees are incomparable.
Appendix~\ref{sec:experimental-appendix} gives complete details for the binary
prediction experiments.
Appendix~\ref{sec:bounded-outcomes} extends the algorithm and its squared-loss span guarantee to bounded real-valued outcomes, and reports the
regression experiments. 
Finally, Appendix~\ref{sec:adaptive-diagnostics} evaluates
the strongly adaptive variant.

\section{Setting and algorithmic ingredients}
\label{sec:setting}

The interaction is adversarial and sequential.  On round $t$, the environment
reveals a context $x_t\in\mathcal X$.  The learner predicts
$p_t\in[0,1]$, interpreted as the probability of outcome $1$.  The
environment then reveals $Y_t\in\{0,1\}$, and the learner suffers Brier loss
$(Y_t-p_t)^2$.  The sequence may be adaptive to the learner's past
predictions.  We use binary outcomes for the probability and classification
interpretations; Appendix~\ref{sec:bounded-outcomes} generalizes to outcomes in
$[0,1]$.

The weak class $\calH$ uses the $[-1,1]$ encoding introduced above and
consists of functions $h:\mathcal X\to[-1,1]$.  We write
\[
  \sigma_t=2Y_t-1,
  \qquad
  \mu_t=2p_t-1,
  \qquad
  r_t=\sigma_t-\mu_t=2(Y_t-p_t).
\]
Thus $\mu_t$ is the $[-1,1]$-scaled forecast corresponding to the probability forecast
$p_t$.

\begin{definition}[Multiaccuracy and self-orthogonality]
\label{def:orthogonality}
For a realized forecast sequence, let
$r_t=\sigma_t-\mu_t$ denote the residual.  The forecasts are
$\alpha$-\emph{multiaccurate with respect to $\calH$} if
\[
  \sup_{h\in\calH}
  \left|
    \sum_{t=1}^T h(x_t)r_t
  \right|
  \le \alpha .
\]
They are $\beta$-\emph{self-orthogonal} if
\[
  \left|
    \sum_{t=1}^T \mu_t r_t
  \right|
  \le \beta .
\]
\end{definition}
After the affine encoding above, these are empirical versions of the
squared-loss orthogonality conditions used in e.g. \citet{KearnsRothRyu2025Networked}: multiaccuracy tests the residuals against an
external class of functions, while self-orthogonality tests the residuals
against the forecast itself.

We assume $\calH$ is symmetric: if $h\in\calH$, then $-h\in\calH$.
Otherwise one replaces $\calH$ by $\calH\cup(-\calH)$ and runs the weak
learner on both signs.

\begin{definition}[Norm-bounded span]
\label{def:norm-bounded-span}
For $\Lambda\ge0$, the $\Lambda$-norm-bounded span of $\calH$ is
\[
  \spanop_\Lambda(\calH)
  =
  \left\{
    f=\sum_{j=1}^m \alpha_j h_j:
    m<\infty,\ h_j\in\calH,\ \sum_{j=1}^m|\alpha_j|\le\Lambda
  \right\}.
\]
\end{definition}
The convex hull of $\calH$ is contained in $\spanop_1(\calH)$.
For every span comparator $f$, write
\[
  q_f(x)=\frac{1+f(x)}{2}.
\]
This is the affine rescaling of $f$ to the $\{0,1\}$ outcome scale on which
Brier loss is measured.  Since $f$ need not take values in $[-1,1]$, $q_f$
need not lie in $[0,1]$.

The algorithm has one problem-dependent online primitive: a weak-class
oracle for $\calH$.  We ask for a \emph{second-order} regret guarantee,
meaning that regret scales with the square root of the cumulative squared
coefficients rather than with $\sqrt T$.  The coefficients supplied to the
oracle will be the forecaster's residuals, so this gives a certificate whose
error scales with the residual energy
$S_T=\sum_t(\sigma_t-\mu_t)^2$ itself.  This self-bounding structure produces
the $1/(\gamma^2\eps)$ weak-to-strong sample complexity in
Corollary~\ref{cor:second-order-boosting}; a first-order $\sqrt T$ guarantee
would give only $1/(\gamma^2\eps^2)$.  The former dependence matches the rate
shown optimal in the prior weak-online-learning model
\citep{BeygelzimerKaleLuo2015OptimalAdaptive}, although our oracle model is
stronger.

\begin{definition}[Second-order weak-class oracle]
\label{def:weak-class-oracle}
A second-order weak-class oracle for $\calH$ has fixed constants
$a_{\calH},b_{\calH}\ge0$.  On round $t$, after observing $x_t$ but before
seeing a coefficient $c_t\in[-2,2]$, it outputs
$\widehat h_t\in[-1,1]$.  For every horizon $T$ and every resulting sequence,
it guarantees
\begin{equation}
\label{eq:weak-second-order}
  \sup_{h\in\calH}
  \sum_{t=1}^T c_t h(x_t)
  -
  \sum_{t=1}^T c_t\widehat h_t
  \le
  a_{\calH}\sqrt{\sum_{t=1}^T c_t^2}+b_{\calH}.
\end{equation}
The loss-minimization convention is obtained by replacing $c_t$ by $-c_t$.
\end{definition}

The assumption has standard instantiations.  If $\calH$ is finite, a
second-order experts algorithm with one expert per $h\in\calH$ gives
$a_{\calH}=O(\sqrt{\log|\calH|})$ and
$b_{\calH}=O(\log|\calH|)$
\citep{CesaBianchiMansourStoltz2007SecondOrder}.  More generally, adaptive
or scale-free online linear optimization over $\conv(\calH)$ gives
data-dependent regret in terms of cumulative gradient norms
\citep{Zinkevich2003OnlineConvexProgramming,KivinenSmolaWilliamson2004OnlineKernels,OrabonaPal2018ScaleFree}.
For example, for the RKHS ball
$\calH_B=\{x\mapsto\langle u,\phi(x)\rangle:\|u\|\le B\}$ with
$\|\phi(x)\|\le\kappa$ and $B\kappa\le1$, these methods give
$O(B\kappa\sqrt{\sum_t c_t^2})$ regret up to lower-order terms.

All remaining online machinery is class independent.  We use two copies of
the following fixed one-dimensional routine.

\begin{definition}[Scalar adaptive OGD]
\label{def:scalar-ogd}
Initialize $a_1=0$ and $V_0=4$.  On round $t$, output $a_t\in[-1,1]$.
After observing a coefficient $g_t\in[-2,2]$, update
\[
  a_{t+1}
  =
  \Pi_{[-1,1]}\!\left(a_t+\frac{g_t}{\sqrt{V_{t-1}}}\right),
  \qquad
  V_t=V_{t-1}+g_t^2,
\]
where $\Pi_{[-1,1]}$ denotes Euclidean projection onto $[-1,1]$.
\end{definition}

\begin{lemma}[Second-order scalar regret]
\label{lem:scalar-ogd}
Scalar adaptive OGD satisfies, for every $a\in[-1,1]$,
\[
  \sum_{t=1}^T(a-a_t)g_t
  \le
  4\sqrt{4+\sum_{t=1}^T g_t^2}
  \le
  4\sqrt{\sum_{t=1}^T g_t^2}+8.
\]
\end{lemma}
The proof is given in Appendix~\ref{app:scalar-regret}.

\section{The Defensive Booster}
\label{sec:algorithm}

The Defensive Booster combines the weak-class oracle with two copies of the
fixed scalar adaptive-OGD routine.  The only forecasting step is a
one-dimensional root rule on the signed mean $\mu_t$.
The labels of the two scalar states describe their roles: $\mathsf S$ controls
the self-auditor used to establish self-orthogonality, while $\mathsf A$
aggregates the weak-class and self auditors.

\begin{definition}[Root rule]
\label{def:root-rule}
For a continuous function $F:[-1,1]\to\R$, let $\Root(F)$ be any point
selected as follows:
\begin{enumerate}[label=(\roman*),leftmargin=2em]
\item if $F$ has a zero in $[-1,1]$, return any such zero;
\item if $F$ is positive throughout $[-1,1]$, return $1$;
\item if $F$ is negative throughout $[-1,1]$, return $-1$.
\end{enumerate}
\end{definition}
Continuity ensures that exactly one of these cases applies.

\begin{lemma}[Root sign property]
\label{lem:root-sign}
For every continuous $F:[-1,1]\to\R$, every signed label
$\sigma\in[-1,1]$, and $\mu=\Root(F)$,
\[
  F(\mu)(\sigma-\mu)\le0.
\]
\end{lemma}

\begin{proof}
If $F(\mu)=0$, the claim is immediate.  If $F$ is positive throughout the
interval, then $\mu=1$ and $\sigma-\mu\le0$.  If $F$ is negative throughout,
then $\mu=-1$ and $\sigma-\mu\ge0$.
\end{proof}

\begin{algorithm}[H]
\caption{Defensive Booster}
\label{alg:defensive}
\begin{algorithmic}[1]
\State Initialize the weak-class oracle over $\calH$ and two independent
       copies $\mathsf S$ (the self-auditor state) and $\mathsf A$
       (the auditor-aggregation state) of scalar adaptive OGD
       (Definition~\ref{def:scalar-ogd}).
\For{$t=1,2,\ldots$}
  \State Observe $x_t$.  Obtain $\widehat h_t\in[-1,1]$ from the weak
         oracle, $\theta_t\in[-1,1]$ from $\mathsf S$, and
         $\lambda_t\in[-1,1]$ from $\mathsf A$.
  \State Set
         $q_{H,t}=(1+\lambda_t)/2$ and
         $q_{S,t}=(1-\lambda_t)/2$.
  \State Form $F_t(\mu)=q_{H,t}\widehat h_t+q_{S,t}\theta_t\mu$ and set
         $\mu_t=\Root(F_t)$.
         \Comment{Definition~\ref{def:root-rule}}
  \State Forecast $p_t=(1+\mu_t)/2$.
  \State Observe $Y_t\in\{0,1\}$; set $\sigma_t=2Y_t-1$ and
         $r_t=\sigma_t-\mu_t$.
  \State Set
         $z_{H,t}=\widehat h_tr_t$,
         $z_{S,t}=\theta_t\mu_tr_t$,
         $u_t=\mu_tr_t$, and
         $v_t=(z_{H,t}-z_{S,t})/2$.
  \State Update the weak-class oracle with $c_t=r_t$, update $\mathsf S$
         with $u_t$, and update $\mathsf A$ with $v_t$.
\EndFor
\end{algorithmic}
\end{algorithm}

Since $F_t$ is affine, the root is computed in constant time: return
$-q_{H,t}\widehat h_t/(q_{S,t}\theta_t)$ when this ratio is defined and lies
in $[-1,1]$, return $0$ when $F_t$ is identically zero, and otherwise return
the endpoint prescribed by Definition~\ref{def:root-rule}.  The per-round
cost is one oracle prediction/update plus $O(1)$ arithmetic.  The resulting
forecast need not be a linear combination or weighted vote of weak
hypotheses: the algorithm predicts a probability directly rather than
maintaining an explicit ensemble.  The algorithm can be viewed as a simple one-dimensional, deterministic instance of the online-learning and
variational-inequality framework of \citet{FarinaPerdomo2026BlackBox}: their
forecast-dependent variational inequality is solved here by an exact root of
the affine function $F_t$.

\paragraph{Multiaccuracy and self-orthogonality guarantees.}
The weak-class and self auditors have gains
$z_{H,t}=\widehat h_tr_t$ and $z_{S,t}=\theta_t\mu_tr_t$.
The scalar state $\mathsf A$ chooses their convex weights: writing
$v_t=(z_{H,t}-z_{S,t})/2$, its two endpoint comparators
$\lambda=1$ and $\lambda=-1$ correspond exactly to always selecting the
weak-class auditor and the self auditor, respectively.  The root rule makes
the resulting weighted gain nonpositive on every round, regardless of the
label.  Lemma~\ref{lem:scalar-ogd} therefore forces each auditor's cumulative
gain to be small.  The weak-class oracle transfers this to multiaccuracy,
while the scalar state $\mathsf S$ transfers the self-auditor bound to
self-orthogonality.  Since $|u_t|,|v_t|\le|r_t|$, every error term scales with
$\sqrt{S_T}$.

\begin{theorem}[Second-order multiaccuracy and self-orthogonality]
\label{thm:residual-certificate}
For every adaptive sequence with $Y_t\in\{0,1\}$, let
\[
  S_T=\sum_{t=1}^T (\sigma_t-\mu_t)^2
      =4\sum_{t=1}^T (Y_t-p_t)^2 .
\]
The Defensive Booster (Algorithm~\ref{alg:defensive}) is
$(A_H\sqrt{S_T}+B_H)$-multiaccurate with respect to $\calH$:
\[
  \sup_{h\in\calH}
  \left|
    \sum_{t=1}^T h(x_t)(\sigma_t-\mu_t)
  \right|
  \le
  A_H\sqrt{S_T}+B_H
\]
and $(A_S\sqrt{S_T}+B_S)$-self-orthogonal:
\[
  \left|
    \sum_{t=1}^T \mu_t(\sigma_t-\mu_t)
  \right|
  \le
  A_S\sqrt{S_T}+B_S,
\]
where
\[
  A_H=a_{\calH}+4,\qquad B_H=b_{\calH}+8,
  \qquad
  A_S=8,\qquad B_S=16.
\]
\end{theorem}

\begin{proof}
Let
\[
  \bar z_t=\frac{z_{H,t}+z_{S,t}}2,
  \qquad
  v_t=\frac{z_{H,t}-z_{S,t}}2.
\]
The gain selected by $\mathsf A$ is
\[
  \bar z_t+\lambda_tv_t
  =
  q_{H,t}z_{H,t}+q_{S,t}z_{S,t}
  =
  F_t(\mu_t)(\sigma_t-\mu_t)
  \le0
\]
by Lemma~\ref{lem:root-sign}.  Competing in
Lemma~\ref{lem:scalar-ogd} with $\lambda=1$ and $\lambda=-1$ gives,
respectively,
\[
  \sum_{t=1}^T z_{H,t}
  -
  \sum_{t=1}^T(\bar z_t+\lambda_tv_t)
  \le
  4\sqrt{4+\sum_{t=1}^T v_t^2}
\]
and the same bound with $z_{S,t}$ in place of $z_{H,t}$.
Since $|v_t|\le|r_t|$ and the selected cumulative gain is nonpositive,
\begin{equation}
\label{eq:auditor-gains}
  \sum_{t=1}^T z_{H,t},
  \quad
  \sum_{t=1}^T z_{S,t}
  \le
  4\sqrt{S_T}+8.
\end{equation}

The weak-class oracle is updated with $c_t=r_t$.  Hence, for every
$h\in\calH$,
\[
  \sum_{t=1}^T h(x_t)r_t
  \le
  \sum_{t=1}^T\widehat h_tr_t
  +
  a_{\calH}\sqrt{S_T}+b_{\calH}
  \le
  (a_{\calH}+4)\sqrt{S_T}+b_{\calH}+8.
\]
Symmetry of $\calH$ gives the absolute-value multiaccuracy bound.

For self-orthogonality, $\mathsf S$ is updated with
$u_t=\mu_tr_t$.  For every $\theta\in[-1,1]$,
Lemma~\ref{lem:scalar-ogd} and \eqref{eq:auditor-gains} give
\[
  \sum_{t=1}^T\theta\mu_tr_t
  \le
  \sum_{t=1}^T\theta_t\mu_tr_t
  +
  4\sqrt{4+\sum_{t=1}^T\mu_t^2r_t^2}
  \le
  8\sqrt{S_T}+16.
\]
Taking the supremum over $\theta\in[-1,1]$ yields
$|\sum_t\mu_tr_t|$.
\end{proof}

\section{Main guarantees}
\label{sec:main-guarantees}

This section derives the paper's guarantees from the two inequalities in
Theorem~\ref{thm:residual-certificate}.  They are parallel consequences,
not consequences of one another: the hard-core statement uses multiaccuracy
alone, whereas the span statement also uses self-orthogonality.
Section~\ref{sec:span-guarantee} gives the Brier/span-regret guarantee, which
holds on every sequence.  Section~\ref{sec:hard-core} shows that the
multiaccuracy bound controls the edge of the forecaster's own mistake
weighting.
Section~\ref{sec:smooth-boosting} turns the edge bound into the weak-to-strong
boosting statement.  Appendix~\ref{sec:separation} gives
examples showing that the two guarantees are incomparable.

\subsection{Brier/span-regret guarantee}
\label{sec:span-guarantee}

\begin{theorem}[Brier/span guarantee]
\label{thm:span-guarantee}
For every adaptive binary sequence, the predictions of the Defensive
Booster (Algorithm~\ref{alg:defensive}) satisfy, for every
$f\in\spanop_\Lambda(\calH)$,
\[
  B_T
  :=
  \frac1T\sum_{t=1}^T (Y_t-p_t)^2
  \le
  \frac1T\sum_{t=1}^T (Y_t-q_f(x_t))^2
  +
  \frac{\Lambda A_H+A_S}{\sqrt T}\sqrt{B_T}
  +
  \frac{\Lambda B_H+B_S}{2T}.
\]
In particular, since $B_T\le1$,
\[
  B_T
  \le
  \frac1T\sum_{t=1}^T (Y_t-q_f(x_t))^2
  +
  \frac{\Lambda A_H+A_S}{\sqrt T}
  +
  \frac{\Lambda B_H+B_S}{2T}.
\]
The same bound holds for every $f\in\conv(\calH)$ with $\Lambda=1$.
\end{theorem}

\begin{proof}
Let $r_t=\sigma_t-\mu_t$.  Convexity of
$\mu\mapsto(\sigma_t-\mu)^2$ gives
\[
  (\sigma_t-\mu_t)^2
  \le
  (\sigma_t-f(x_t))^2+2r_t(f(x_t)-\mu_t).
\]
If $f=\sum_{j=1}^m\alpha_jh_j$ with
$\sum_j|\alpha_j|\le\Lambda$, Theorem~\ref{thm:residual-certificate} gives
\[
  \left|\sum_{t=1}^T f(x_t)r_t\right|
  \le
  \Lambda(A_H\sqrt{S_T}+B_H).
\]
The same theorem gives
$|\sum_t\mu_tr_t|\le A_S\sqrt{S_T}+B_S$.  Substituting these two bounds into
the signed convexity inequality gives
\[
  S_T
  \le
  \sum_{t=1}^T(\sigma_t-f(x_t))^2
  +
  2(\Lambda A_H+A_S)\sqrt{S_T}
  +
  2(\Lambda B_H+B_S).
\]
Divide by $4T$ and use
$S_T=4TB_T$ and
$(\sigma_t-f(x_t))^2=4(Y_t-q_f(x_t))^2$.  Since
$B_T\le1$, the second bound follows.
\end{proof}

\begin{corollary}[Low-loss span guarantee]
\label{cor:low-loss-span}
In the setting of Theorem~\ref{thm:span-guarantee}, write
\[
  B_f=\frac1T\sum_{t=1}^T (Y_t-q_f(x_t))^2,
  \qquad
  C=\Lambda A_H+A_S,
  \qquad
  D=\Lambda B_H+B_S .
\]
Then
\[
  B_T
  \le
  B_f+C\sqrt{\frac{B_f}{T}}+\frac{3C^2}{2T}+\frac{3D}{4T}.
\]
In particular, if $B_f=0$, then $B_T=O((C^2+D)/T)$.
\end{corollary}
The proof is given in Appendix~\ref{app:full-horizon-derivations}.

\subsection{The hard-core mistake weighting}
\label{sec:hard-core}

\begin{definition}[Reweighting, smoothness, and edge]
\label{def:smooth-edge}
A \emph{reweighting} of the realized transcript is a sequence
$w_t\in[0,1]$.  Its \emph{density} is
\[
  \rho(w)=\frac1T\sum_{t=1}^T w_t .
\]
For $\rho>0$, the reweighting is \emph{$\rho$-smooth} if
$\rho(w)\ge\rho$.
When $\sum_t w_t>0$, its \emph{normalized edge} against $\calH$ is
\[
  \edge_{\calH}(w)
  =
  \sup_{h\in\calH}
  \left|
    \frac{\sum_{t=1}^T w_t\sigma_th(x_t)}{\sum_{t=1}^T w_t}
  \right|.
\]
For $\rho,\gamma>0$, the transcript satisfies the
$(\rho,\gamma)$-\emph{smooth weak-learning condition} if every
$\rho$-smooth reweighting $w$ has $\edge_{\calH}(w)\ge\gamma$.
\end{definition}

When $\calH\subseteq\{-1,+1\}^{\mathcal X}$ is a class of binary-valued
classifiers, let
\[
  \operatorname{err}_w(h)
  =
  \frac{\sum_{t=1}^T w_t\mathbf 1\{h(x_t)\ne\sigma_t\}}
       {\sum_{t=1}^T w_t}
\]
denote the weighted error of $h\in\calH$.  Its normalized weighted correlation
is $1-2\operatorname{err}_w(h)$.  Because
$\operatorname{err}_w(-h)=1-\operatorname{err}_w(h)$,
\[
  \left|1-2\operatorname{err}_w(h)\right|
  =
  1-2\min\{\operatorname{err}_w(h),\operatorname{err}_w(-h)\}.
\]
Thus the absolute value in the edge compares $h$ with the classifier $-h$
obtained by flipping all of $h$'s predictions.  In particular,
$\edge_{\calH}(w)\ge\gamma$ means that, for some $h\in\calH$, either $h$ or
$-h$ has weighted error at most $(1-\gamma)/2$.  If $\calH$ is closed under
negation, both orientations are members of $\calH$.

\paragraph{Relation to smooth distributions.}
The condition in Definition~\ref{def:smooth-edge} is ex post: it is a
property of the realized transcript, not an input to the forecaster.  If $w$
is $\rho$-smooth, then its normalization
\[
  D_w(t)=\frac{w_t}{\sum_s w_s}
\]
is a distribution on the rounds satisfying $D_w(t)\le 1/(\rho T)$.
Conversely, any distribution $D$ on $[T]$ with
$D(t)\le 1/(\rho T)$ is represented by the $\rho$-smooth reweighting
$w_t=\rho T D(t)$.  Thus bounded smooth reweightings are exactly the
unnormalized form of the smooth distributions used by SmoothBoost
\citep{Servedio2003SmoothBoosting} and in boosting-based hard-core
constructions
\citep{KlivansServedio2003BoostingHardCore,BarakHardtKale2009UniformHardcore}.
The Defensive Booster does not maintain a distribution over past rounds or
train separate weak learners on different reweightings.  It sends the current
signed residual $2(Y_t-p_t)$ to one weak-class learner, as required to obtain
multiaccuracy.  The sequence $w_t=|Y_t-p_t|$ is interpreted only after the
fact as the witness analyzed below.

\begin{theorem}[Hard-core mistake weighting]
\label{thm:main-hard-core}
Let $p_1,\ldots,p_T$ be the probability forecasts of the Defensive
Booster, and let
\[
  B_T=\frac1T\sum_{t=1}^T(Y_t-p_t)^2 .
\]
Define the randomized mistake weights
\[
  w_t=Y_t(1-p_t)+(1-Y_t)p_t=|Y_t-p_t|,
  \qquad
  \rho_w=\frac1T\sum_{t=1}^T w_t .
\]
Then $w_t\in[0,1]$,
\[
  B_T=\frac1T\sum_{t=1}^T w_t^2\le\rho_w,
\]
and every $h\in\calH$ satisfies
\[
  \left|
    \frac1T\sum_{t=1}^T w_t\sigma_th(x_t)
  \right|
  \le
  \frac{A_H\sqrt{TB_T}+B_H/2}{T}.
\]
Consequently, if $\rho_w>0$ then
\[
  \edge_{\calH}(w)
  \le
  \frac{A_H\sqrt{TB_T}+B_H/2}{T\rho_w}.
\]
\end{theorem}

\begin{proof}
Since
$Y_t\in\{0,1\}$ and $p_t\in[0,1]$, $w_t\in[0,1]$ and
$(Y_t-p_t)^2=w_t^2$.  Hence
$B_T=T^{-1}\sum_tw_t^2\le\rho_w$.

The key identity is
\[
  w_t\sigma_t
  =
  Y_t-p_t
  =
  \frac{\sigma_t-\mu_t}{2}.
\]
The multiaccuracy part of Theorem~\ref{thm:residual-certificate} therefore implies, for every
$h\in\calH$,
\[
  \left|
    \frac1T\sum_{t=1}^T w_t\sigma_th(x_t)
  \right|
  =
  \frac1{2T}
  \left|
    \sum_{t=1}^T h(x_t)(\sigma_t-\mu_t)
  \right|
  \le
  \frac{A_H\sqrt{TB_T}+B_H/2}{T},
\]
because $S_T=4TB_T$.
Dividing by $\rho_w$ gives the normalized edge bound when $\rho_w>0$.
\end{proof}

\subsection{The smooth weak-learning condition gives classification boosting}
\label{sec:smooth-boosting}

The smooth weak-learning condition turns the hard-core alternative around.  If
no sufficiently smooth small-edge weighting exists, the algorithm's own
mistake weighting cannot be smooth.

\begin{corollary}[Second-order weak-to-strong rate]
\label{cor:second-order-boosting}
Let $\rho_0,\gamma_0>0$.  If the realized transcript satisfies the
$(\rho_0,\gamma_0)$-smooth weak-learning condition for $\calH$, then the
Brier loss $B_T=T^{-1}\sum_t(Y_t-p_t)^2$ and the randomized classification
error $\rho_w=T^{-1}\sum_t|Y_t-p_t|$ both satisfy
\[
  B_T,\rho_w
  \le
  \max\left\{
    \rho_0,\,
    \frac{4A_H^2}{\gamma_0^2T},\,
    \frac{B_H}{\gamma_0T}
  \right\}.
\]
The deterministic threshold classifier
$\widehat Y_t=\mathbf 1\{p_t\ge1/2\}$, with arbitrary tie-breaking at
$p_t=1/2$, has average classification error at most $2\rho_w$.
\end{corollary}

\paragraph{Proof idea.}
If the mistake weighting is not $\rho_0$-smooth, then its density
$\rho_w$ is already below $\rho_0$, as is $B_T$.  Otherwise, the smooth
weak-learning condition lower-bounds its edge by $\gamma_0$, whereas
Theorem~\ref{thm:main-hard-core} upper-bounds the same edge in terms of
$B_T$ and $\rho_w$.  Solving the two resulting inequalities gives the stated
bounds.  Thresholding adds at most a factor of two because every threshold
mistake has $|Y_t-p_t|\ge1/2$.  The complete calculation is given in
Appendix~\ref{app:full-horizon-derivations}.

Taking $\rho_0$ to be a sufficiently small constant multiple of $\eps$ and
\[
  T
  =
  \Omega\!\left(
    \frac{A_H^2}{\gamma_0^2\eps}
    +
    \frac{B_H}{\gamma_0\eps}
  \right)
\]
gives Brier loss, randomized classification error, and deterministic classification error for the thresholded classifier
at most $\eps$, up to constants.  Thus the smooth
weak-learning condition must hold at smoothness
$\rho_0=O(\eps)$ for a target error $\eps$.  When $B_H$ is logarithmic or
lower order, this is the usual $1/(\gamma_0^2\eps)$ dependence.  The lower
bound of \citet{BeygelzimerKaleLuo2015OptimalAdaptive} shows that this
dependence is unavoidable in their weak-online-learning model, up to
logarithmic and excess-loss terms. 

\section{Boosting on every interval}
\label{sec:interval-boosting}

The preceding guarantees average over the full horizon.  We now give a
strongly adaptive variant: one forecast sequence satisfies the same two
guarantees, up to polylogarithmic factors, on every contiguous interval.
The construction uses a standard second-order specialist reduction.  We
state the reduction first because preserving dependence on the local
residual energy is essential; an ordinary $O(\sqrt{|I|})$ interval-regret
bound would lose the optimal weak-to-strong rate and get a $1/(\gamma_0^2\eps^2)$ dependence instead.

\begin{proposition}[Second-order interval wrapper]
\label{prop:interval-wrapper}
Fix a horizon $T$.  Suppose an online learner $\mathsf B$, whenever started
fresh, outputs $z_t\in[-1,1]$ and, for every comparator sequence
$z^\star=(z_t^\star)_t$ in a fixed class and every coefficient sequence
$c_t\in[-2,2]$, satisfies
\[
  \sum_{t=1}^n c_t(z_t^\star-z_t)
  \le
  a\sqrt{\sum_{t=1}^n c_t^2}+b.
\]
Set
\[
  L_T=\log(4T),
  \qquad
  M_T=2\left\lceil\log_2(2T)\right\rceil.
\]
There is a wrapper $\mathsf{SA}(\mathsf B)$ whose output
$\widetilde z_t\in[-1,1]$ satisfies, simultaneously for every interval
$I\subseteq[T]$ and every comparator $z^\star$,
\begin{equation}
\label{eq:interval-wrapper}
  \sum_{t\in I}c_t(z_t^\star-\widetilde z_t)
  \le
  \alpha_T(a)\sqrt{\sum_{t\in I}c_t^2}+\beta_T(b),
\end{equation}
where, for a universal constant $C_0$,
\[
  \alpha_T(a)=\sqrt{M_T}\bigl(a+C_0\sqrt{L_T}\bigr),
  \qquad
  \beta_T(b)=M_T\bigl(b+C_0L_T\bigr).
\]
The wrapper maintains at most $1+\lceil\log_2T\rceil$ active copies of
$\mathsf B$ per round.
\end{proposition}

The wrapper combines fresh copies of $\mathsf B$ on dyadic intervals with a
second-order confidence-rated experts algorithm.  Its standard proof is
given in Appendix~\ref{app:interval-proofs}.

Apply Proposition~\ref{prop:interval-wrapper} separately to the weak-class
oracle and to the two scalar routines $\mathsf S$ and $\mathsf A$ in
Algorithm~\ref{alg:defensive}; use their aggregate outputs in the same root
rule and feed the wrappers the same coefficients as before.  Call the
resulting forecaster the \emph{strongly adaptive Defensive Booster}.  Let
$a_H^{\rm int},b_H^{\rm int}$ denote the coefficients in
\eqref{eq:interval-wrapper} for the weak-class wrapper, and let
$a_{\rm sc}^{\rm int},b_{\rm sc}^{\rm int}$ denote them for either scalar
wrapper.  Since scalar adaptive OGD has fresh-run constants $a=4$ and $b=8$,
Proposition~\ref{prop:interval-wrapper} gives the explicit values
\[
  a_H^{\rm int}=\alpha_T(a_{\calH}),
  \quad b_H^{\rm int}=\beta_T(b_{\calH}),
  \qquad
  a_{\rm sc}^{\rm int}=\alpha_T(4),
  \quad b_{\rm sc}^{\rm int}=\beta_T(8).
\]

\begin{theorem}[Interval certificate]
\label{thm:interval-certificate}
For every adaptive binary sequence, the strongly adaptive Defensive Booster
satisfies, simultaneously for every interval $I\subseteq[T]$,
\[
  \sup_{h\in\calH}
  \left|\sum_{t\in I}h(x_t)r_t\right|
  \le
  A_H^{\rm int}\sqrt{S_I}+B_H^{\rm int},
  \qquad
  \left|\sum_{t\in I}\mu_tr_t\right|
  \le
  A_S^{\rm int}\sqrt{S_I}+B_S^{\rm int},
\]
where $S_I=\sum_{t\in I}r_t^2$ and
\[
  A_H^{\rm int}=a_H^{\rm int}+a_{\rm sc}^{\rm int},
  \quad B_H^{\rm int}=b_H^{\rm int}+b_{\rm sc}^{\rm int},
  \qquad
  A_S^{\rm int}=2a_{\rm sc}^{\rm int},
  \quad B_S^{\rm int}=2b_{\rm sc}^{\rm int}.
\]
\end{theorem}

The proof repeats the argument of
Theorem~\ref{thm:residual-certificate} using the interval-regret bounds of
Proposition~\ref{prop:interval-wrapper}; details are given in
Appendix~\ref{app:interval-proofs}.

\begin{corollary}[Strongly adaptive boosting]
\label{cor:interval-boosting}
For an interval $I\subseteq[T]$, let $n=|I|$ and define
\[
  B_I=\frac1n\sum_{t\in I}(Y_t-p_t)^2,
  \qquad
  \rho_I=\frac1n\sum_{t\in I}|Y_t-p_t|.
\]
The interval mistake weights $w_t=|Y_t-p_t|$ form a local hard-core
witness: if $\rho_I>0$, then
\[
  \sup_{h\in\calH}
  \frac{\left|\sum_{t\in I}w_t\sigma_th(x_t)\right|}
       {\sum_{t\in I}w_t}
  \le
  \frac{A_H^{\rm int}\sqrt{nB_I}+B_H^{\rm int}/2}{n\rho_I}.
\]
Simultaneously for every interval $I$:
\begin{enumerate}[label=(\roman*),leftmargin=2em]
\item for every $f\in\spanop_\Lambda(\calH)$,
\[
  B_I
  \le
  \frac1n\sum_{t\in I}(Y_t-q_f(x_t))^2
  +
  \frac{\Lambda A_H^{\rm int}+A_S^{\rm int}}{\sqrt n}\sqrt{B_I}
  +
  \frac{\Lambda B_H^{\rm int}+B_S^{\rm int}}{2n};
\]
\item for any $\rho_0,\gamma_0>0$, if every weighting $w\in[0,1]^I$ with
$n^{-1}\sum_{t\in I}w_t\ge\rho_0$ satisfies
\[
  \sup_{h\in\calH}
  \frac{\left|\sum_{t\in I}w_t\sigma_th(x_t)\right|}
       {\sum_{t\in I}w_t}
  \ge\gamma_0,
\]
then
\[
  B_I,\rho_I
  \le
  \max\left\{
    \rho_0,
    \frac{4(A_H^{\rm int})^2}{\gamma_0^2n},
    \frac{B_H^{\rm int}}{\gamma_0n}
  \right\}.
\]
The threshold classifier has error at most $2\rho_I$ on $I$.
\end{enumerate}
\end{corollary}

The corollary follows by applying the proofs of
Theorem~\ref{thm:span-guarantee}, Theorem~\ref{thm:main-hard-core}, and
Corollary~\ref{cor:second-order-boosting} on $I$, with
Theorem~\ref{thm:interval-certificate} in place of the full-horizon
certificate.

Thus, for fixed oracle constants, a target interval error $\eps$ requires
$n=O(\log^2(T)/(\gamma_0^2\eps))$ when the local weak-learning condition
holds with $\rho_0=O(\eps)$, while the span-regret guarantee holds without any
weak-learning condition.  Since the bounds hold simultaneously, the
interval and its span comparator may be selected after observing the
transcript.  No assumption is made about rounds outside $I$.
The price for this simultaneous interval guarantee is the explicit
logarithmic factors in Proposition~\ref{prop:interval-wrapper} and at most
$1+\lceil\log_2T\rceil$ active weak-class oracle copies; the basic Defensive Booster
retains its guarantee while maintaining one weak-class oracle.  Structurally,
one forecast sequence therefore
produces a family of data-dependent local hard-core witnesses: whenever error
remains high on an interval, the mistake weights on that interval identify a
smooth distribution on which the entire weak class has small edge.  This
conclusion goes beyond interval comparator regret by identifying where and
when weak learnability fails; it does not require the algorithm to detect a
change point or explicitly construct a hard subset.
Figure~\ref{fig:adaptive-hard-core-evolution} in
Appendix~\ref{sec:adaptive-diagnostics} visualizes the density and weak-class
edge of these local witnesses across interval endpoints and time scales on a
stream with known change points.

\section{Experiments}
\label{sec:experiments}
We compare the Defensive Booster with online gradient boosting, online
weak-to-strong boosting, and the more naive Brier aggregator strategy on controlled synthetic streams and four naturally
ordered real datasets.  The synthetic streams separately test settings in
which the span contains an informative predictor and settings in which the
smooth weak-learning condition holds.  The real streams test performance on
naturally ordered binary data; Appendix~\ref{sec:regression-experiments}
separately evaluates three chronological regression datasets.  Across the
binary experiments, the Defensive Booster's
Brier loss is competitive with the best baseline on each dataset and often
improves upon it substantially, while the gradient boosting and weak-to-strong boosting ensemble baselines take
$20$--$66\times$ as much time per round.  The Brier-loss aggregator over the
four ensembles is yet more expensive and does not close the gap on the two
real streams where the Defensive Booster performs best.

\paragraph{Protocol.}
We compare eight methods.  Two unboosted controls isolate the benefit of
aggregation by simply running the learning algorithm that the boosting techniques take as input: \textsc{Unboosted reg.} performs online squared-loss regression
over the weak class, while \textsc{Unboosted cls.} runs the online classifier
used as the base learner by the classification boosters.  Four ensemble
baselines represent the two boosting traditions.  \textsc{OGB} is online
gradient boosting
\citep{BeygelzimerHazanKaleLuo2015OnlineGradient}; \textsc{Online BBM} is
the rate-optimal online boost-by-majority algorithm;
\textsc{AdaBoost.OL} is the adaptive logistic-loss algorithm from the same
paper \citep{BeygelzimerKaleLuo2015OptimalAdaptive}; and \textsc{OSBoost} is
online SmoothBoost \citep{ChenLinLu2012OnlineBoosting}.  The \textsc{Brier
aggregator} combines the forecasts of these four ensembles by
exponential weighting under Brier loss.  The Defensive Booster and each
unboosted control maintain one online learner over the weak class; each
boosting baseline maintains an ensemble of $N=100$ such learners, and the
aggregator must run all four ensembles, for a total of $400$.

We evaluate deterministic $0/1$ classification error, Brier loss, and randomized classification error
$T^{-1}\sum_t|Y_t-p_t|$.  Online BBM and \textsc{Unboosted cls.} output hard
labels, which we view as probabilities in $\{0,1\}$ when computing Brier loss
and randomized error.  Consequently, all three metrics coincide for these
two methods.  AdaBoost.OL randomizes over its partial ensembles; we report
its probability of predicting one.  Its randomized-error score is therefore
the expected classification error of the original randomized output, while
its Brier score evaluates that probability directly.

We fix all hyperparameters before examining performance and do not tune them
separately for each stream.  In particular, Online BBM and OSBoost use the
analytically guaranteed weak-learning advantage on controlled synthetic
streams where one is known, and the fixed target classification advantage
$\gamma=.1$ on all other streams.
All synthetic results use $T=3000$ and report means over $20$ seeds.  We
process each real dataset stream once in its recorded order, without shuffling.
Appendix~\ref{sec:experimental-appendix} gives the complete details including algorithm hyperparameter
settings, data generators and preprocessing steps, standard errors, and
runtime tables.  Code, public-data loaders, and exact reproduction commands
are available at \url{https://github.com/aaroth/defensive-boosting}.

We use two synthetic streams to isolate the two guarantees.  To test the
weak-to-strong guarantee, we use the
\emph{binary aggregation stream}.  The weak class contains $200$ binary
hypotheses, arranged as $100$ opposite pairs $\{\pm h_j\}_{j=1}^{100}$.
The algorithms see all $200$
binary predictions in random order and are not told which orientation in
each pair is useful.  Signed labels are balanced and randomly ordered.  There is
a hidden choice of one orientation from each pair such that all $100$ chosen
rules are correct on half the rounds; on each remaining round exactly $58$
are correct and $42$ are incorrect, in a cyclically balanced pattern.  No
single rule is perfect, and averaging all $200$ displayed rules gives zero.
The hidden average, however, has signed margin at least
$2(.58)-1=.16$ on every round and therefore classifies perfectly.  Averaging
over the chosen orientations shows that every nonzero reweighting admits a
displayed rule with edge at least $.16$.  Adding negations and hiding the
orientations does not change the span, so the symmetry calculation in
Proposition~\ref{prop:separation} shows that every fixed affine span score has
Brier loss at least
$(1-.16)^2/(8(1+.16^2))=.0860$.  The stream thus directly instantiates the
separation between span prediction and weak-to-strong aggregation.

To test the span guarantee, we use the \emph{random-label mixture stream}.
Here the contexts are normalized vectors in $\mathbb R^{30}$ and the weak
class is the infinite Euclidean linear class
$\calH=\{x\mapsto\langle u,x\rangle:\|u\|_2\le1\}$.  Independently on each
round, with probability $.65$ the label follows a fixed noisy linear rule and
with probability $.35$ it is an independent random bit.  Uniform weighting
over the random-label rounds is smooth and, with high probability, has low
edge, so the smooth weak-learning condition fails for any constant target edge;
nevertheless, the linear span remains informative on the structured rounds.
These streams are deliberately favorable to different baseline families.  On
the binary aggregation stream, the Bayes classification error is zero and the
weak-to-strong boosters approach it.  On the random-label mixture stream, OGB is
the strongest baseline in Brier loss and approaches the least-squares span
benchmark, whereas the classification boosters incur substantially larger
Brier loss.  Each stream is therefore tailored to one baseline family.  The
 test is whether the Defensive Booster approaches the stronger
baseline on each stream while improving on the other family.
Figure~\ref{fig:core-experiments} shows deterministic classification error on
the binary aggregation stream and Brier loss on the random-label mixture.

\begin{figure}[H]
\centering
\includegraphics[width=.48\linewidth]{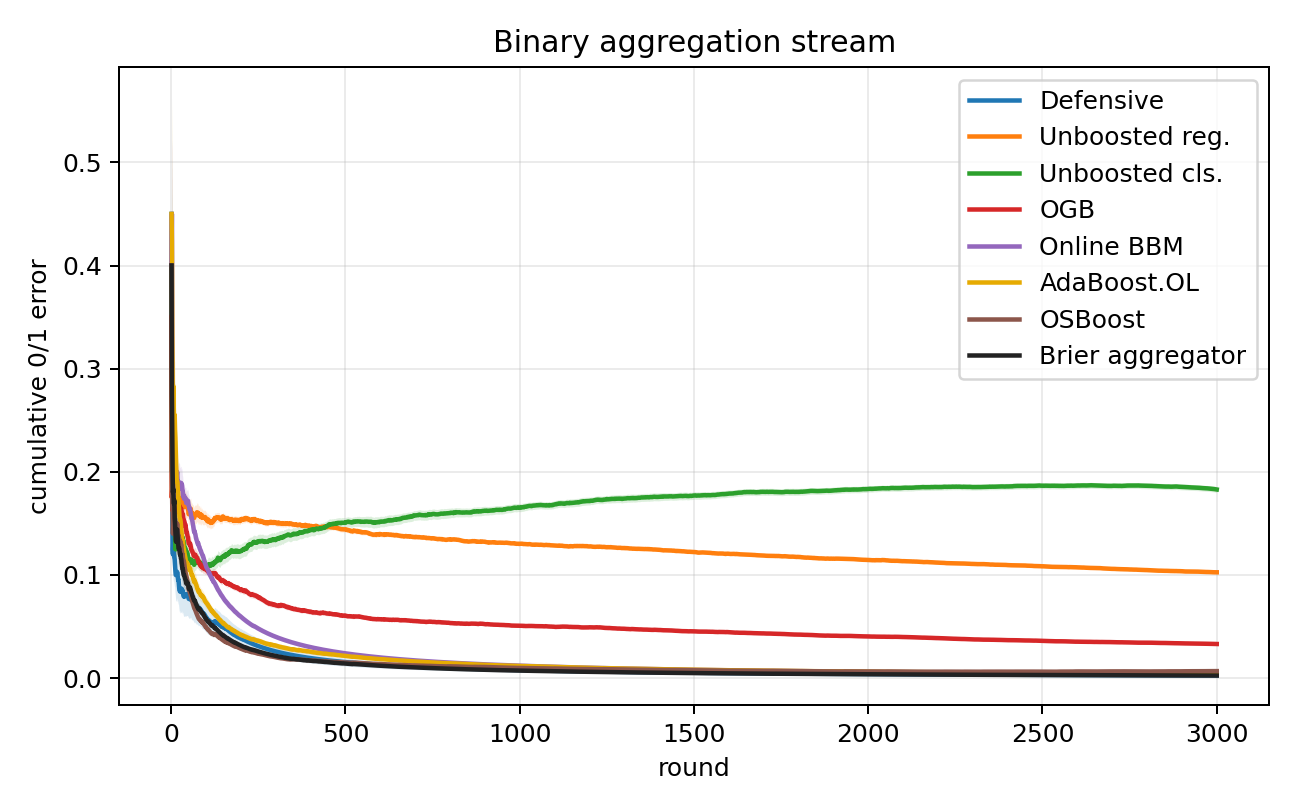}
\hfill
\includegraphics[width=.48\linewidth]{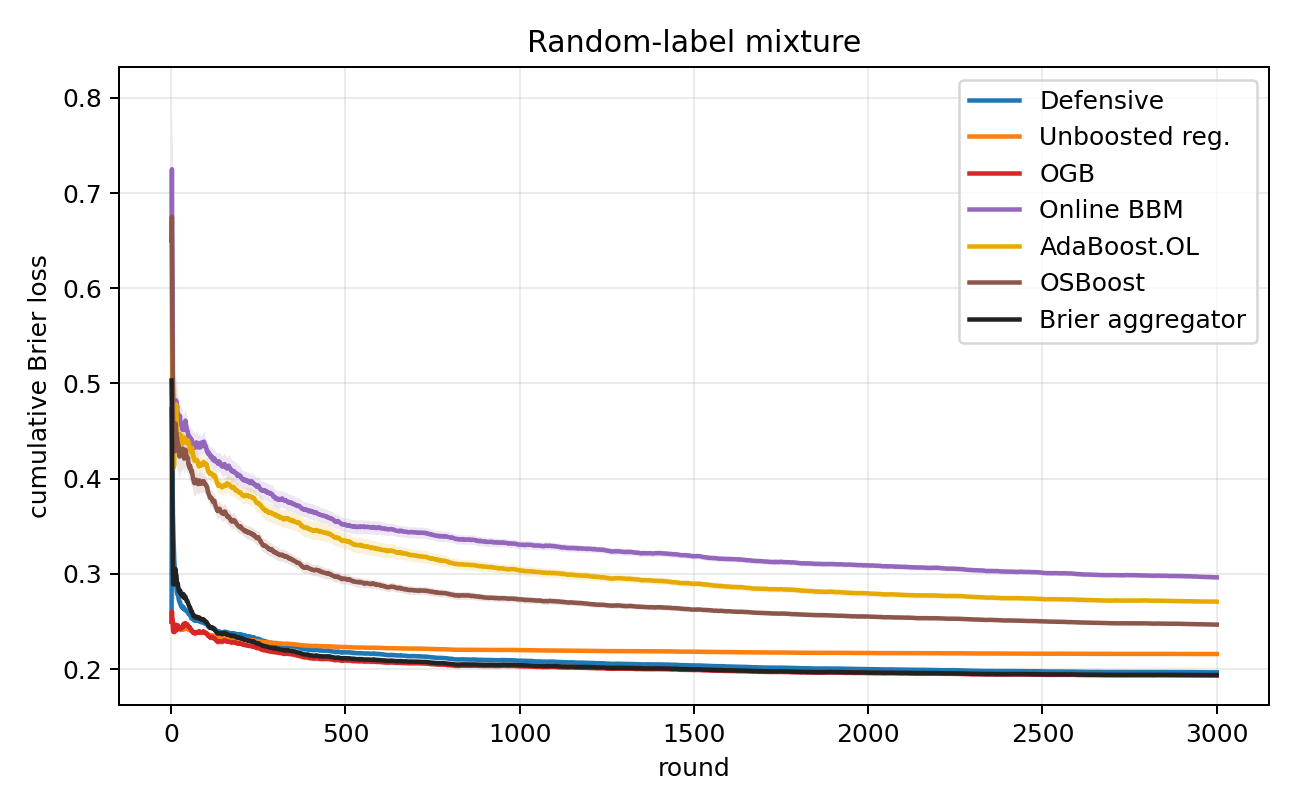}
\caption{The two baseline families excel on complementary streams, and the
Defensive Booster tracks the better family in each.  Left: cumulative $0/1$
error on the binary aggregation stream.  Online BBM, AdaBoost.OL, and OSBoost
approach zero, while OGB retains positive error; the
Defensive Booster also approaches zero and is within $.0001$ of the Brier
aggregator's final mean.  Right:
cumulative Brier loss on the random-label mixture.
OGB has the lowest baseline loss, while the weak-to-strong boosters remain
higher; the Defensive Booster tracks OGB.  The Brier aggregator tracks the
better family in both panels by running all four $100$-learner ensembles,
whereas the Defensive Booster maintains one learner.  Curves are means over $20$ seeds.
The hard-label unboosted classifier is omitted from the Brier panel for scale
and reported in Table~\ref{tab:experiments}.}
\label{fig:core-experiments}
\end{figure}

\paragraph{Results.}
On the binary aggregation stream at $T=3000$, Online BBM reaches
hard-prediction error $.0041$, AdaBoost.OL reaches $.0042$, and OSBoost
reaches $.0068$.  The Defensive Booster reaches $.0026$,
compared with $.0331$ for OGB and $.1829$ for the unboosted classifier, while
using one learner rather than $100$.  Its Brier loss is $.0018$, below every
individual ensemble and the $.0025$ loss of the Brier aggregator.

On the random-label mixture stream, OGB and the Defensive Booster have Brier losses
$.1933$ and $.1965$, respectively, while OSBoost, AdaBoost.OL, and Online BBM have losses
$.2467$, $.2708$, and $.2963$.  The Brier aggregator reaches $.1937$ by running all
four ensembles.  Thus the Defensive Booster remains competitive with OGB
on a sequence where the weak-to-strong guarantee does not apply.  The full
synthetic table in Appendix~\ref{sec:synthetic-details} includes standard
errors and three additional streams: a planted weak rule among decoys, an
infinite linear weak class, and random labels.

The random-label mixture stream also lets us inspect the hard-core guarantee
directly.  Figure~\ref{fig:certificate-diagnostic} tracks
the Defensive Booster's multiaccuracy and self-orthogonality errors, together
with the density of its mistake weighting and the weak class's edge under that
weighting.  The two errors and the class edge decay while the density remains
nontrivial, so the mistake weights form the smooth, low-edge witness predicted
by Theorem~\ref{thm:main-hard-core}.

\begin{figure}[H]
\centering
\includegraphics[width=.92\linewidth]{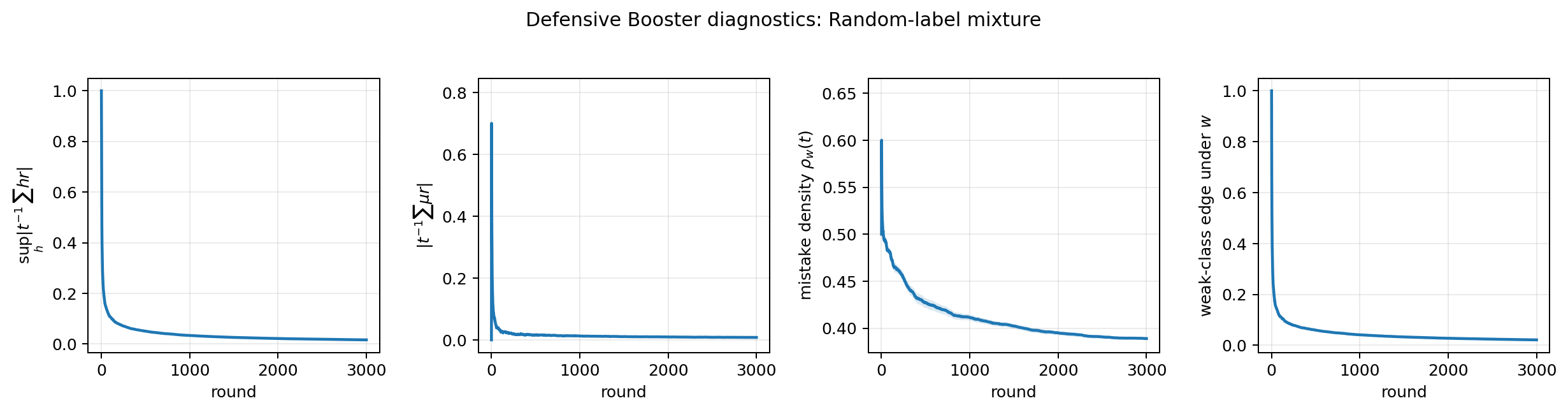}
\caption{Multiaccuracy, self-orthogonality, and hard-core diagnostics on the
random-label mixture stream, averaged over $20$ seeds.  The first two panels
show the cumulative multiaccuracy and self-orthogonality correlations from
Theorem~\ref{thm:residual-certificate}, divided by the current round $t$; the
last two show the density and $\edge_{\calH}(w)$, the weak-class edge under
the mistake weighting $w_t=|Y_t-p_t|$.  Randomized error remains nontrivial
because of the random-label component, but $\edge_{\calH}(w)$ decays: the forecaster's
mistakes are the smooth low-edge witness predicted by
Theorem~\ref{thm:main-hard-core}.}
\label{fig:certificate-diagnostic}
\end{figure}

\paragraph{One learner versus an ensemble.}
Each unboosted control and the Defensive Booster have running time scaling as $C_H+O(1)$ per
round, where $C_H$ is the running time of one weak-learner prediction and update.  OGB, Online
BBM, AdaBoost.OL, and OSBoost have per-round running time $NC_H+O(N)$ with $N$ learners; OSBoost also projects
its combiner onto a simplex.  The Brier aggregator runs all four ensemble boosters. Figure~\ref{fig:compute-sweep} compares
prediction quality as the ensemble size varies.  On this stream, the
Defensive Booster has lower Brier loss and randomized error than each
$N=100$ ensemble and the Brier aggregator.  In our implementation, the
$N=100$ methods take $20$--$66$
times as much wall-clock time per round across the synthetic and real
experiments.  Absolute constants are implementation-dependent, but the
difference in the number of maintained weak learners is part of the
algorithms themselves.

\begin{figure}[H]
\centering
\includegraphics[width=.82\linewidth]{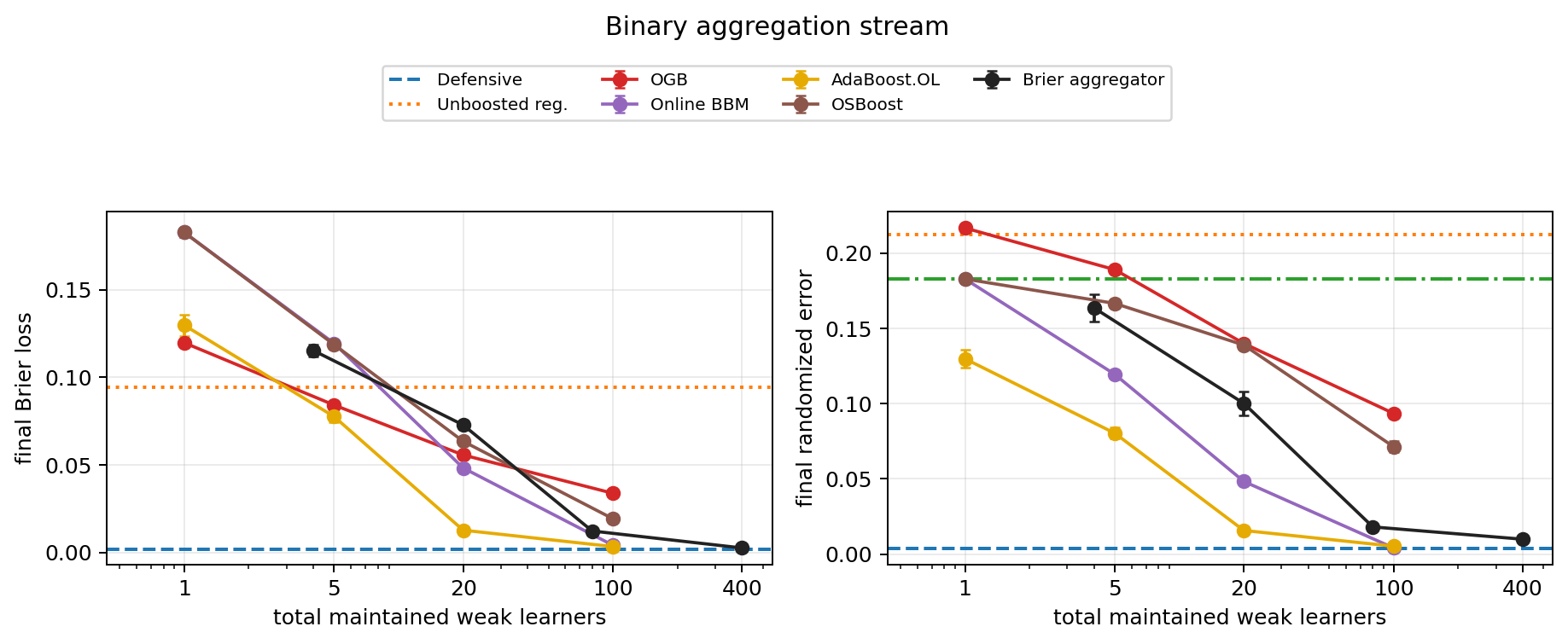}
\caption{Prediction quality as a function of the total number of maintained
weak learners on the binary aggregation stream.  OGB, Online BBM, AdaBoost.OL, and OSBoost use
$N\in\{1,5,20,100\}$ learners; the Brier aggregator at each setting runs all
four ensembles and is therefore plotted at $4N$.  The Defensive Booster and
the unboosted controls each maintain one learner, so their performance appears
as a horizontal line.  The hard-label unboosted classifier is omitted from
the Brier panel for scale.  With one learner, the Defensive Booster achieves
Brier loss and randomized error below the 100-learner ensembles and their
aggregator.}
\label{fig:compute-sweep}
\end{figure}

\paragraph{Real-world data streams.}
We next evaluate the algorithms on four public binary data streams, each
processed in its recorded order.  Bank Marketing predicts whether a client
subscribes to a term deposit \citep{MoroRitaCortez2012BankMarketing};
Electricity predicts price movement in the New South Wales electricity market
\citep{Harries1999Electricity,MOADatasets}; Airlines predicts flight delays
\citep{MOADatasets}; and Occupancy predicts whether an office is occupied from
contemporaneous sensor measurements
\citep{CandanedoFeldheim2016Occupancy}.  Figure~\ref{fig:intro-experiment}
compares each method's final
Brier loss with the best observed loss on that stream, and
Table~\ref{tab:real-brier-main} reports the absolute averages.  The
Defensive Booster has the lowest Brier loss on Electricity and Occupancy by a
wide margin, and it also has the lowest deterministic error on Occupancy.
AdaBoost.OL has the lowest classification errors on Electricity and the
lowest randomized error on Occupancy.  On Bank, the
Brier aggregator is best by $.0010$ over the Defensive Booster.  On Airlines,
the Defensive Booster, OGB, and the aggregator differ by less than
$6\cdot10^{-5}$.  Both unboosted controls are substantially worse on
Electricity and Occupancy, showing that the gains come from the Defensive
Booster's aggregation rather than merely from maintaining fewer learners.  Complete
preprocessing, cumulative curves, classification and randomized errors, and
runtimes appear in
Appendix~\ref{sec:real-data-details}.

\begin{table}[H]
\centering
\scriptsize
\begin{tabular}{lrrrrrrrr}
\hline
Dataset & Defensive & Unboosted reg. & Unboosted cls. & OGB & BBM & AdaBoost.OL & OSBoost & Brier agg. \\
\hline
Bank & $.0800$ & $.0845$ & $.1091$ & $.0791$ & $.1026$ & $.1046$ & $.1679$ & $\mathbf{.0790}$ \\
Electricity & $\mathbf{.0772}$ & $.1957$ & $.3669$ & $.1516$ & $.2010$ & $.1007$ & $.1167$ & $.1006$ \\
Airlines & $.2094$ & $.2190$ & $.3901$ & $\mathbf{.2094}$ & $.3411$ & $.3417$ & $.2354$ & $.2094$ \\
Occupancy & $\mathbf{.0071}$ & $.0396$ & $.0961$ & $.0159$ & $.0138$ & $.0103$ & $.0406$ & $.0101$ \\
\hline
\end{tabular}
\caption{Average online Brier loss on four real binary streams processed in
recorded order; lower is better.  No dataset is shuffled.  The hard-label methods
\textsc{Unboosted cls.} and BBM are scored as $0/1$ probability forecasts.
Bold marks the lowest unrounded value in each row.  The Brier aggregator runs OGB, Online BBM, AdaBoost.OL, and OSBoost, for a total of $400$
weak learners.  The Defensive Booster is best on Electricity and Occupancy,
essentially ties OGB and the aggregator on Airlines, and is within $.0010$ of
the aggregator on Bank.}
\label{tab:real-brier-main}
\end{table}

\paragraph{Regression beyond binary outcomes.}
Appendix~\ref{sec:regression-experiments} evaluates the bounded-outcome
extension on three chronological regression datasets.  Relative to 100-stage
OGB, the Defensive Booster lowers normalized mean squared error by $18\%$ on
Appliance Energy, $29\%$ on Bike Demand, and $17\%$ on Interstate Traffic,
while maintaining one weak learner rather than $100$; OGB takes
$65$--$70\times$ as much wall-clock time per round in our implementation.

The synthetic streams isolate the strengths of the two baseline families:
OGB is strongest when an informative span comparator is available, while the
classification boosters are strongest when the smooth weak-learning condition
holds.  We use the same tuning protocol for every method rather than retuning
each method on each stream.  The Defensive
Booster remains competitive on both streams, uses one online learner, and
its mistake weights expose the smooth, low-edge witness measured in
Figure~\ref{fig:certificate-diagnostic}.  Appendix~\ref{sec:adaptive-diagnostics}
evaluates the strongly adaptive variant of the Defensive Booster from
Section~\ref{sec:interval-boosting}.  On the four original real streams, its
16--20 weak-class learners, active at different time scales, make it slower
than the basic Defensive Booster.  It remains $3$--$10\times$ faster than the
100-learner ensembles and
further reduces both forecasting and classification error on Electricity,
Airlines, and Occupancy
(Table~\ref{tab:real-data}).  On the INSECTS optical-sensor benchmark, whose
released streams have controlled abrupt, gradual, incremental, and recurring
distribution shifts, the same adaptive variant improves both errors on four of
five drift patterns and essentially ties the basic method on the fifth
(Table~\ref{tab:insects-drift}).

\section{Related work}
\label{sec:related-work}
Our work connects to multiple streams of prior work. The most directly relevant to our application is prior work on online boosting, which is where the baseline algorithms in our experiments are drawn from:

\paragraph{Online gradient boosting.}
\citet{BeygelzimerHazanKaleLuo2015OnlineGradient} start from an online
linear-loss learner for $\calH$ and use $N$ copies to compete with
$\conv(\calH)$ or a norm-bounded span under smooth convex losses.
\citet{HuSunVenkatramanHebertBagnell2017Streaming} study gradient boosting
on i.i.d. data streams and extend their analysis to adversarial streams under
a stronger edge assumption; \citet{HazanSingh2021BoostingOCO} use a
multiplicative weak learner to obtain regret to a convex hull in online convex
optimization.  Our Brier/span guarantee is closest to the first of these in the
squared-loss case: it competes with unrestricted real-valued scores in a
norm-bounded span.  Under the same online-linear-oracle primitive, however,
the Defensive Booster uses one weak-class learner rather than an $N$-stage
ensemble. 

\paragraph{Online weak-to-strong boosting.}
\citet{OzaRussell2001OnlineBaggingBoosting} initiated work on practical online
bagging and boosting methods.  The closest classification predecessor to our
work is \citet{ChenLinLu2012OnlineBoosting}, who adapt SmoothBoost to online
binary boosting using smooth distributions.
\citet{BeygelzimerKaleLuo2015OptimalAdaptive} give the rate-optimal online
boosting algorithm, Online BBM, under weak online learnability assumptions.
They also prove
matching lower bounds: in their model, the optimal sample-complexity
dependence for error $\eps$ is $1/(\gamma^2\eps)$ up to logarithmic and
excess-loss terms.  Our algorithm matches this $\gamma,\eps$ dependence.
It also uses only one online linear oracle over $\calH$ rather than
many parallel weak learners.
These papers make different weak-oracle assumptions.  Chen et al.\ and
Beygelzimer et al.\ assume that the online learner's own predictions have a
fixed positive edge over random guessing on every admissible stream---with
smooth importance weights in the former case---up to an excess-loss term.
\citet{BrukhimChenHazanMoran2020OnlineAgnostic} instead assume a
multiplicative agnostic oracle that obtains a fixed fraction of the best
correlation in $\calH$, and boost it to regret against the best $h\in\calH$.
Our primitive is instead a no-regret learning algorithm for $\calH$ under
linear losses: it need not have any absolute edge, but it competes with every
$h\in\calH$ on the realized residual losses.  Thus $\calH$ is the final
comparator class in the agnostic framework of Brukhim et al., whereas here it
supplies weak directions that are aggregated into forecasts competing with
$\spanop(\calH)$; positive edge enters separately through our ex-post smooth
weak-learning condition.

\paragraph{AdaBoost as loss optimization.}
In the offline setting, weak-to-strong boosting algorithms such as AdaBoost
have also been analyzed through the lens of loss minimization.
\citet{MasonBaxterBartlettFrean2000BoostingGradientDescent} place boosting
inside the broader view of functional gradient descent.
\citet{MukherjeeRudinSchapire2011RateAdaBoost} show that AdaBoost converges
to the infimum empirical exponential loss over additive combinations of weak
hypotheses, without assuming weak learnability or a finite minimizer.  This
is analogous to our span-regret guarantee: both retain a span-optimization
interpretation when weak learning assumptions fail.  The objective and
setting differ: their guarantee is batch optimization of exponential margin
loss over scores, while ours is an online pathwise Brier-regret guarantee
for probability forecasts.  Exponential loss rewards large margins and does
not by itself produce calibrated probabilities; empirically,
\citet{NiculescuMizilCaruana2012CalibratedBoosting} show that boosted
outputs can have poor squared error and cross-entropy because they are not
well-calibrated posterior probabilities.

\paragraph{Strong adaptivity.}
Strongly adaptive online learning asks for low regret on every contiguous
interval.  Generic geometric-cover reductions obtain this guarantee from a
standard online learner with $O(\log T)$ active copies
\citep{DanielyGonenShalevShwartz2015StronglyAdaptive}; second-order
confidence bounds preserve dependence on local gradient energy
\citep{GaillardStoltzVanErven2014SecondOrder,Cutkosky2020StronglyAdaptive}.
The ``adaptive'' online booster of
\citet{BeygelzimerKaleLuo2015OptimalAdaptive} is parameter-free rather than
strongly adaptive in this interval sense.  Section~\ref{sec:interval-boosting}
applies the strongly adaptive machinery inside the Defensive Booster's
auditors, preserving both interval span regret and the guarantee that
persistent interval error yields a smooth, low-edge mistake weighting.  The
defensive-forecasting construction therefore extends to strong adaptivity
through standard online-learning machinery.

\paragraph{Smooth boosting and hard-core sets.}
Smooth distributions are central in smooth boosting
\citep{Servedio2003SmoothBoosting}; \citet{Gavinsky2003OptimallySmooth}
develops smooth adaptive boosting in the agnostic setting.  The minimax view
of boosting goes back to \citet{FreundSchapire1996Game}.  The connection
between boosting and hard-core construction starts from
\citet{Impagliazzo1995HardCore} and was made algorithmic by
\citet{KlivansServedio2003BoostingHardCore} and
\citet{BarakHardtKale2009UniformHardcore}.  Our reweightings are the online
transcript analogue of these smooth distributions.  

\paragraph{Multicalibration, multiaccuracy, and loss minimization.}
Multicalibration was introduced by
\citet{HebertJohnsonKimReingoldRothblum2018}; multiaccuracy was isolated as
a black-box correction criterion by \citet{KimGhorbaniZou2019}.  Outcome
indistinguishability and omniprediction turn stronger prediction certificates
into simultaneous downstream loss guarantees
\citep{DworkKimReingoldRothblumYona2021,GopalanKalaiReingoldSharanWieder2022,GopalanHuKimReingoldWieder2023}.
\citet{GlobusHarrisHarrisonKearnsRothSorrell2023} characterize when batch
multicalibration boosts squared-loss regression to Bayes optimality.
\citet{KearnsRothRyu2025Networked} use the weaker pair of conditions
used here for loss minimization: multiaccuracy and self-orthogonality.

\paragraph{Multicalibration and hard-core measures.}
Our weak-to-strong guarantee uses the connection between multiaccuracy and
hard-core measures that appears in the complexity-theoretic regularity lemma
of \citet{TrevisanTulsianiVadhan2009Regularity}.  Stronger variants derive
hard-core measures from multicalibration
\citep{CasacubertaDworkVadhan2024} or calibrated multiaccuracy
\citep{CasacubertaGopalanKanadeReingold2025}.  Our proof requires only
multiaccuracy. This is important online: adversarial
sequential calibration error cannot generally be bounded at the $O(\sqrt{T})$ scale achieved here
\citep{QiaoValiant2021Sidestepping,DaganDaskalakisFishelsonGolowichKleinbergOkoroafor2025Calibration,CollinaLuNoarovRoth2026Lower}.

\paragraph{Defensive forecasting.}
Our algorithm is developed in the defensive forecasting framework which  chooses probabilities that prevent
continuous ``skeptic'' strategies from increasing their capital by betting
against the forecasts
\citep{VovkTakemuraShafer2005DefensiveForecasting,VovkNouretdinovTakemuraShafer2005LinearProtocols}.
\citet{Vovk2007OptimalExpertAdvice} shows that defensive forecasting also
handles continuous ``second-guessing'' experts whose advice depends on the
learner's current forecast; the affine test $F_t$ in our root rule has this
form.
Many online calibration and multicalibration algorithms can be interpreted in
this framework
\citep{GuptaJungNoarovPaiRoth2022,bastani2022practical,NoarovRoth2023StatisticalScope,noarovhigh,GargJungReingoldRoth2024,GhugeMuthukumarSingla2025,perdomo2025defense,HuLuoSenapatiSharan2026Swap}.
In a recent general result, \citet{FarinaPerdomo2026BlackBox} give generic
reductions from online multicalibration to a no-regret learner plus an expected
variational-inequality solver and recover traditional defensive-forecasting
algorithms as special cases.  The affine root step in
Algorithm~\ref{alg:defensive} is a deterministic one-dimensional instance of
their framework.  To our knowledge, ours is the first online boosting theorem
obtained this way.  Its additional structure yields the span guarantee and
the hard-core mistake weighting.

\subsection*{Acknowledgments}
The authors used AI tools, specifically GPT 5.6 Pro, and GPT 5.6 in the Codex environment in the development of this paper. All of the final theorems and proofs are written and verified by the authors. The code for the empirical evaluation was written via GPT 5.6 Codex.

\bibliographystyle{plainnat}
\bibliography{online_hardcore_refs}

\clearpage
\appendix

\section{Deferred proofs}
\label{app:deferred-proofs}

This appendix contains proofs of the standard online-learning tools and
routine consequences used in the main text.

\subsection{Scalar second-order regret}
\label{app:scalar-regret}

\begin{proof}[Proof of Lemma~\ref{lem:scalar-ogd}]
Let $\eta_t=1/\sqrt{V_{t-1}}$.  Nonexpansiveness of projection gives
\[
  g_t(a-a_t)
  \le
  \frac{(a-a_t)^2-(a-a_{t+1})^2}{2\eta_t}
  +
  \frac{\eta_tg_t^2}{2}.
\]
Because $\eta_t$ is nonincreasing and the diameter of $[-1,1]$ is $2$, the
sum of the first terms is at most
$2/\eta_T\le2\sqrt{V_T}$.  Moreover, $g_t^2\le4\le V_{t-1}$, and hence
\[
  \frac{g_t^2}{\sqrt{V_{t-1}}}
  =
  \left(1+\sqrt{\frac{V_t}{V_{t-1}}}\right)
  \left(\sqrt{V_t}-\sqrt{V_{t-1}}\right)
  \le
  \frac52\left(\sqrt{V_t}-\sqrt{V_{t-1}}\right).
\]
Thus the sum of the second terms is at most
$\frac54\sqrt{V_T}$.  Adding the two contributions gives
\[
  \sum_{t=1}^T(a-a_t)g_t
  \le
  \left(2+\frac54\right)\sqrt{V_T}
  =
  \frac{13}{4}\sqrt{V_T}
  \le
  4\sqrt{V_T}.
\]
Finally,
$\sqrt{4+x}\le2+\sqrt x$ gives the final inequality.
\end{proof}

\subsection{Consequences of the full-horizon certificate}
\label{app:full-horizon-derivations}

\begin{proof}[Proof of Corollary~\ref{cor:low-loss-span}]
Theorem~\ref{thm:span-guarantee} gives
\[
  B_T\le B_f+\frac{C}{\sqrt T}\sqrt{B_T}+\frac{D}{2T}.
\]
Let $x=B_T$, $a=C/\sqrt T$, and $d=D/(2T)$.  For $u=\sqrt x$,
\[
  u^2\le B_f+au+d,
  \qquad\text{so}\qquad
  \left(u-\frac a2\right)^2\le B_f+d+\frac{a^2}{4}.
\]
Taking the positive square root gives
\[
  u
  \le
  \frac a2+\sqrt{B_f+d+\frac{a^2}{4}}.
\]
Hence
\[
  x
  \le
  B_f+d+a^2+a\sqrt{B_f+d}
  \le
  B_f+a\sqrt{B_f}+\frac32a^2+\frac32d,
\]
where the last inequality uses
$\sqrt{B_f+d}\le\sqrt{B_f}+\sqrt d$ and
$a\sqrt d\le(a^2+d)/2$.  Substituting the definitions of $a$ and $d$
proves the claim.
\end{proof}

\begin{proof}[Proof of Corollary~\ref{cor:second-order-boosting}]
Let $w_t=|Y_t-p_t|$.  If $\rho_w<\rho_0$, then
Theorem~\ref{thm:main-hard-core} gives $B_T\le\rho_w<\rho_0$, and the
randomized-error bound is immediate.  Now suppose $\rho_w\ge\rho_0$.  The
smooth weak-learning condition applies to $w$, while
Theorem~\ref{thm:main-hard-core} gives
\[
  \gamma_0
  \le
  \edge_{\calH}(w)
  \le
  \frac{A_H\sqrt{TB_T}+B_H/2}{T\rho_w}.
\]
Since $B_T\le\rho_w$, we have
\[
  \gamma_0
  \le
  \frac{A_H}{\sqrt{TB_T}}
  +
  \frac{B_H}{2TB_T}.
\]
The denominator is nonzero: $\rho_w\ge\rho_0>0$ implies that some $w_t>0$,
and hence $B_T=T^{-1}\sum_t w_t^2>0$.
If both
$B_T>4A_H^2/(\gamma_0^2T)$ and
$B_T>B_H/(\gamma_0T)$ held, then the two terms on the right would each be
strictly smaller than $\gamma_0/2$, a contradiction.  This proves the stated
bound on $B_T$.

For the randomized-error bound, combine the edge inequality with
$B_T\le\rho_w$:
\[
  \gamma_0
  \le
  \frac{A_H\sqrt{T\rho_w}+B_H/2}{T\rho_w}
  =
  \frac{A_H}{\sqrt{T\rho_w}}
  +
  \frac{B_H}{2T\rho_w}.
\]
Writing $u=\sqrt{\rho_w}$, this becomes
\[
  u^2
  \le
  \frac{A_H}{\gamma_0\sqrt T}u
  +
  \frac{B_H}{2\gamma_0T}.
\]
If both
$u^2>4A_H^2/(\gamma_0^2T)$ and
$u^2>B_H/(\gamma_0T)$ held, the right side would be strictly smaller than
$u^2$, again a contradiction.  This proves the randomized-error bound.

Finally, if $\widehat Y_t\ne Y_t$ then $|Y_t-p_t|\ge1/2$.  Averaging shows
that the deterministic threshold error is at most $2\rho_w$.
\end{proof}

\subsection{Strongly adaptive extension}
\label{app:interval-proofs}

\begin{proof}[Proof of Proposition~\ref{prop:interval-wrapper}]
Pad $[T]$ to the next power of two and let $\mathcal J$ be its canonical
family of dyadic intervals.  The family has fewer than $4T$ members, at most
$1+\lceil\log_2T\rceil$ of which contain any round, and every
interval $I\subseteq[T]$ is a disjoint union of at most $M_T$ members of
$\mathcal J$ \citep{DanielyGonenShalevShwartz2015StronglyAdaptive}.  Start
one copy $\mathsf B_J$ at the left endpoint of each $J\in\mathcal J$ and
run it only on $J$.

Aggregate the active copies with a second-order confidence-rated experts
algorithm, treating membership in $J$ as expert $J$'s confidence.  To see
that the standard guarantee applies, let $q_{J,t}$ be the algorithm's
weights on the active intervals and set
$\widetilde z_t=\sum_Jq_{J,t}z_{J,t}$.  Map the linear gain to the loss
\[
  \ell_{J,t}=\frac{2-c_tz_{J,t}}4\in[0,1].
\]
The confidence-regret reduction and second-order excess-loss bound of
\citet{GaillardStoltzVanErven2014SecondOrder} give, for every
$J\in\mathcal J$,
\[
  \sum_{t\in J}c_t(z_{J,t}-\widetilde z_t)
  \le
  C_0\left(
    \sqrt{L_T\sum_{t\in J}c_t^2}+L_T
  \right),
\]
for a universal constant $C_0$.  Here we assign a uniform prior to the
fewer than $4T$ dyadic specialists.  A geometric grid of learning rates, if
needed, contributes only $O(\log\log T)$, which is absorbed by $C_0L_T$.
The loss range used in the second-order bound is valid because
$|\ell_{J,t}-\sum_{J'}q_{J',t}\ell_{J',t}|\le |c_t|/2$.

Now partition an arbitrary $I$ into
$J_1,\ldots,J_m\in\mathcal J$.  On each block, insert the prediction of
$\mathsf B_{J_j}$ between the comparator and the wrapper.  The fresh-run
guarantee and the preceding confidence-regret bound give
\[
\begin{split}
  \sum_{t\in I}c_t(z_t^\star-\widetilde z_t)
  &\le
  \sum_{j=1}^m
  \left(
    a\sqrt{\sum_{t\in J_j}c_t^2}+b
    +C_0\sqrt{L_T\sum_{t\in J_j}c_t^2}+C_0L_T
  \right) \\
  &\le
  \sqrt{M_T}\bigl(a+C_0\sqrt{L_T}\bigr)
  \sqrt{\sum_{t\in I}c_t^2}
  +M_T\bigl(b+C_0L_T\bigr),
\end{split}
\]
where the last step uses Cauchy--Schwarz and $m\le M_T$.  Only the
$1+\lceil\log_2T\rceil$ copies associated with intervals containing the
current round are active.  This proves \eqref{eq:interval-wrapper}.

Advance knowledge of $T$ is not essential.  Partition time into epochs
$[2^j,2^{j+1}-1]$ and run the fixed-horizon construction afresh in each
epoch.  Any interval up to time $T$ meets at most
$1+\lceil\log_2T\rceil$ epochs.  Summing the fixed-horizon bounds over those
pieces and applying Cauchy--Schwarz adds a factor
$\sqrt{1+\lceil\log_2T\rceil}$ to the coefficient of the second-order term
and a factor $1+\lceil\log_2T\rceil$ to the additive term.  At any time only
the wrapper for the current epoch is active.
\end{proof}

\begin{proof}[Proof of Theorem~\ref{thm:interval-certificate}]
Fix an interval $I$.  The root sign property holds separately on every
round, so the aggregated auditor gain is nonpositive on $I$.  Interval
regret for the wrapped $\mathsf A$ against its two endpoint comparators
therefore gives
\[
  \sum_{t\in I}z_{H,t},\quad \sum_{t\in I}z_{S,t}
  \le
  a_{\rm sc}^{\rm int}\sqrt{S_I}+b_{\rm sc}^{\rm int}.
\]
For every $h\in\calH$, interval regret for the wrapped weak-class oracle
then gives
\[
  \sum_{t\in I}h(x_t)r_t
  \le
  \sum_{t\in I}z_{H,t}
  +a_H^{\rm int}\sqrt{S_I}+b_H^{\rm int}
  \le
  A_H^{\rm int}\sqrt{S_I}+B_H^{\rm int}.
\]
Symmetry of $\calH$ gives the absolute value.  Similarly, for
$\theta\in\{-1,1\}$, interval regret for the wrapped $\mathsf S$ gives
\[
  \sum_{t\in I}\theta\mu_tr_t
  \le
  \sum_{t\in I}z_{S,t}
  +a_{\rm sc}^{\rm int}\sqrt{\sum_{t\in I}\mu_t^2r_t^2}
  +b_{\rm sc}^{\rm int}
  \le
  A_S^{\rm int}\sqrt{S_I}+B_S^{\rm int}.
\]
Taking both signs proves the self-orthogonality inequality.
\end{proof}

\section{Separation: the guarantees are incomparable}
\label{sec:separation}

The Brier/span and weak-to-strong guarantees do not imply one another.
Proposition~\ref{prop:separation} gives a transcript on which every
reweighting has positive weak-class edge, yet every real-valued score induced
by the span has constant squared loss.  The construction exploits a basic
difference between the guarantees: squared loss depends on the numerical
values of a score, whereas weak-to-strong boosting can exploit its sign.
Proposition~\ref{prop:converse-separation} gives the converse separation: a
small subset of uninformative rounds defeats the smooth weak-learning
condition even though a span comparator has small squared loss.

\begin{proposition}[Separation]
\label{prop:separation}
For every $\gamma\in(0,1)$ and $\eta>0$, there are a number
$\delta\in[\gamma,\min\{\gamma+\eta,1\})$, a symmetric binary-valued class
$\calH$, and a transcript $(x_t,Y_t)_{t\le T}$ such that:
\begin{enumerate}[label=(\roman*),leftmargin=2em]
\item every reweighting $w\in[0,1]^T$ with $\sum_tw_t>0$ has
      $\edge_\calH(w)\ge\delta$; in particular the
      $(\rho,\gamma)$-smooth weak-learning condition holds for every
      $\rho>0$;
\item no single hypothesis is correct on every round, but a uniform average
      of hypotheses in $\calH$ has positive signed margin on every round and
      therefore classifies the transcript perfectly;
\item every affine score $q_f(x)=(1+f(x))/2$ induced by the span of
      $\calH$, with no range or norm constraint, has
      average squared loss at least
      $(1-\delta)^2/(8(1+\delta^2))$.
\end{enumerate}
\end{proposition}

\begin{proof}
Choose an integer $k$ large enough that
$2/k<\min\{\eta,1-\gamma\}$, and set
\[
  r=\left\lceil\frac{(1+\gamma)k}{2}\right\rceil,
  \qquad
  \delta=\frac{2r-k}{k}.
\]
Then $k/2<r<k$ and
$\gamma\le\delta<\min\{\gamma+\eta,1\}$.  Let
$h_1,\ldots,h_k$ be binary-valued hypotheses and include their negations in
$\calH$.  All labels equal one, so $\sigma_t=1$ on every round.

The transcript has two equally large parts.  On every context in the first
part, set $h_j(x_t)=1$ for all $j$.  The second part contains one context for
each $r$-element subset $A\subseteq[k]$, and on that context set
\[
  h_j(x_A)=
  \begin{cases}
    1,&j\in A,\\
    -1,&j\notin A.
  \end{cases}
\]
Repeat the all-positive context $\binom{k}{r}$ times so that the two parts
have equal size.  At every context,
\[
  \frac1k\sum_{j=1}^k \sigma_t h_j(x_t)
  \ge \delta.
\]
Consequently, for every reweighting $w$ with positive mass,
\[
  \frac1k\sum_{j=1}^k
  \frac{\sum_{t=1}^Tw_t\sigma_th_j(x_t)}{\sum_{t=1}^Tw_t}
  \ge\delta.
\]
Some $h_j$ must therefore have weighted edge at least $\delta$, proving
(i).  The uniform average $k^{-1}\sum_jh_j$ equals $1$ on the first part
and $\delta$ on the second, so its sign is always correct.  On the other
hand, every $h_j$ equals $-1$ on a positive fraction of the second part
because $r<k$.  This proves (ii).

It remains to minimize squared loss over the span.  Write
$f=\sum_{j=1}^k\alpha_jh_j$.  The transcript and its squared-loss objective
are invariant under permutations of the $k$ coordinates.  Averaging the
coefficient vector over all permutations and using convexity therefore
cannot increase loss.  Hence an optimum has
$\alpha_1=\cdots=\alpha_k=b/k$ for some $b\in\mathbb R$.  Its signed score
equals $b$ on the first part and $\delta b$ on the second, so its Brier loss
is
\[
  \psi(b)
  =\frac{(1-b)^2+(1-\delta b)^2}{8}.
\]
This is minimized at $b^\star=(1+\delta)/(1+\delta^2)$, where it equals
$(1-\delta)^2/(8(1+\delta^2))$.  This proves (iii).
\end{proof}

The separation also survives clipping span scores to the probability range
when the coefficient norm is bounded.  Fix $\Lambda\ge1$, write
$f=\sum_j\alpha_jh_j$ with $\sum_j|\alpha_j|\le\Lambda$, and clip $f$ to
$[-1,1]$ before converting it to a probability.  On the second half of the
construction, the average of $f$ over all $r$-subsets is
$\delta\sum_j\alpha_j\le\delta\Lambda$.  For every
$u\in[-\Lambda,\Lambda]$,
\[
  \operatorname{clip}_{[-1,1]}(u)
  \le \frac{2u+\Lambda-1}{\Lambda+1}.
\]
If $\delta\le1/(2\Lambda)$, the average clipped signed score on this half is
therefore at most $\Lambda/(\Lambda+1)$.  Jensen's inequality shows that
these rounds have average Brier loss at least
$1/(4(\Lambda+1)^2)$, and hence the full transcript has average Brier loss
at least $1/(8(\Lambda+1)^2)$.  In particular, clipping does not remove the
separation for any fixed coefficient-norm budget.  Moreover, perfect clipped
prediction would require $f(x_A)\ge1$ for every $r$-subset $A$; averaging
these inequalities gives
$\delta\sum_j\alpha_j\ge1$, and therefore
$\sum_j|\alpha_j|\ge1/\delta$.

As $\delta\to0$, the span's best squared loss approaches $1/8$, while
Corollary~\ref{cor:second-order-boosting} forces the Defensive Booster's
Brier score and classification error to vanish as $T$ grows.  Thus a
Brier/span-regret guarantee alone can leave constant squared loss on
transcripts where the smooth weak-learning condition forces near-perfect
prediction.

The reverse implication fails for a different reason: a small set of rounds
can support a zero-edge reweighting while contributing little to average
squared loss.

\begin{proposition}[Converse separation]
\label{prop:converse-separation}
For every $T$ and every even $m\in\{2,\ldots,T\}$, there is a symmetric
binary-valued class
$\calH=\{h_1,-h_1,h_2,-h_2\}$ and a transcript
$(x_t,Y_t)_{t\le T}$ such that:
\begin{enumerate}[label=(\roman*),leftmargin=2em]
\item the smooth weak-learning condition fails for every
      $\rho\le m/T$ and every $\gamma>0$;
\item the span comparator $f=(h_1+h_2)/2\in\spanop_1(\calH)$ induces the score
      $q_f=(1+f)/2$ with average squared loss
      $T^{-1}\sum_t(Y_t-q_f(x_t))^2=m/(4T)$.
\end{enumerate}
\end{proposition}

\begin{proof}
Choose distinct contexts $x_1,\ldots,x_T$, arbitrary binary labels, and let
$R\subseteq\{1,\ldots,T\}$ have size $m$.  Split $R$ into equal parts
$R_1$ and $R_2$.  Outside $R$, set
$h_1(x_t)=h_2(x_t)=\sigma_t$.  On $R_1$, set
$(\sigma_th_1(x_t),\sigma_th_2(x_t))=(1,-1)$; on $R_2$, reverse these two
values.  Consider the reweighting
$w_t=\mathbf 1\{t\in R\}$.  It has density $m/T$, and
\[
  \sum_{t=1}^T w_t\sigma_th_1(x_t)
  =
  \sum_{t=1}^T w_t\sigma_th_2(x_t)
  =0
\]
and the same holds for their negations, so $\edge_\calH(w)=0$.  Hence the
$(\rho,\gamma)$-smooth weak-learning condition fails for every
$\rho\le m/T$ and every $\gamma>0$.

The comparator $f=(h_1+h_2)/2$ equals $\sigma_t$ off $R$ and zero on $R$.
It therefore predicts perfectly off $R$ and induces probability $1/2$ on
$R$, so each
round in $R$ contributes squared loss $1/4$.  Its average squared loss is
therefore $m/(4T)$.
\end{proof}

The random-label mixture experiment in Section~\ref{sec:experiments}
instantiates the second separation qualitatively: the randomly labeled rounds
support a smooth reweighting with small weak-class edge, while the structured
rounds retain an informative least-squares span score.

\section{Additional experimental details and results}
\label{sec:experimental-appendix}

This appendix section gives the complete protocol and additional results for
the binary prediction experiments.  The
synthetic streams separately examine the two guarantees: Brier loss and an
offline least-squares span comparator measure competition with the span,
while randomized classification error
$T^{-1}\sum_t|Y_t-p_t|$ measures weak-to-strong performance and equals the
density of the forecaster's mistake weighting.  The real streams evaluate
forecasting and classification performance on examples processed in their
recorded order.  Appendices~\ref{sec:regression-experiments}
and~\ref{sec:adaptive-diagnostics} report the regression and strongly adaptive
experiments, respectively.

\subsection{Protocol and implementation}
\label{sec:experimental-protocol}

\paragraph{Algorithms.}
The main comparison includes eight methods.  \textsc{Defensive} implements the
Defensive Booster
(Algorithm~\ref{alg:defensive}), including the two class-independent scalar
adaptive-OGD states in Definition~\ref{def:scalar-ogd}.
For the real and distribution-shift experiments, we additionally report
\textsc{Adaptive Def.}, the strongly adaptive Defensive Booster from
Section~\ref{sec:interval-boosting}; it maintains one weak learner at each
active dyadic scale.  Section~\ref{sec:adaptive-diagnostics} gives its
implementation details and local diagnostics.

The two unboosted controls each maintain a single learner.
\textsc{Unboosted reg.} performs online squared-loss regression over the same
weak class, using online linear optimization on the squared-loss gradient.
\textsc{Unboosted cls.} runs the online binary classifier used as the weak
learner by Online BBM, AdaBoost.OL, and OSBoost; it isolates the benefit of boosting from
that of the base classification algorithm.

The four ensemble baselines each maintain $100$ weak learners.  \textsc{OGB}
is the online
gradient boosting algorithm of
\citet{BeygelzimerHazanKaleLuo2015OnlineGradient}, specialized to
one-dimensional signed squared loss, equivalently Brier loss up to a factor
of four, and run with $100$ boosting stages and stage step
$\eta=(\log N)/N$, as suggested by their theoretical discussion.  In the
learner-count sweep, the $N=1$ endpoint uses the admissible step $\eta=1$.
\textsc{Online BBM} is the boost-by-majority algorithm of
\citet{BeygelzimerKaleLuo2015OptimalAdaptive}, the optimal-rate binary online
classification booster; we run its importance-weighted version with $100$ weak
learners and the target advantage described below.  Online BBM outputs hard binary predictions, so its Brier score in
our tables is the Brier score of the induced $0/1$ probability forecast.
\textsc{AdaBoost.OL} is the adaptive logistic-loss booster from the same
paper.  We implement the importance-weighted version used in its experiments:
projected online gradient descent learns each weak learner's coefficient, and
exponential weights aggregate the hard predictions of the partial ensembles.
The algorithm's probability of predicting one is used as its forecast.  Thus
its randomized-error score is the expected $0/1$ error of the original
randomized classifier, while its Brier score evaluates that probability
forecast directly.  AdaBoost.OL requires no target-edge parameter.
\textsc{OSBoost} is the online SmoothBoost algorithm of
\citet{ChenLinLu2012OnlineBoosting} with its online-convex-programming
combiner and $100$ weak learners, following the experimental convention in
that paper.  We use the target-selection rule
below and importance-weighted updates as in their experiments.  We include
OSBoost because its smooth reweighting
mechanism is especially close to the smooth-distribution mechanism analyzed
here.

The Online BBM/AdaBoost.OL paper and the OSBoost paper define $\gamma$ as the amount by which weak
classification error improves on $1/2$; equivalently, a binary weak prediction
has signed correlation $2\gamma$.  On a synthetic stream with a known
guaranteed advantage, we supply that value rather than tune it from the
observed results.  Thus Online BBM and OSBoost use $\gamma=.08$ on binary
aggregation, whose guaranteed correlation edge is $.16$, and $\gamma=.1$ on
the other streams.  This convention keeps the baselines' classification
parameter distinct from the correlation edge in our theorem.

Finally, \textsc{Brier aggregator} runs OGB, Online BBM, AdaBoost.OL, and OSBoost in
parallel.  Before observing $Y_t$, it predicts the weighted average of their
four probabilities and, after observing $Y_t$, multiplies each weight by
$\exp\{-(Y_t-p_{t,i})^2/2\}$.  Since Brier loss is $1/2$-exp-concave on
$[0,1]$, its cumulative Brier loss is at most that of the best constituent
plus $2\log 4$: for each binary outcome $y$,
$p\mapsto\exp(-(y-p)^2/2)$ is concave, so the standard exponential-weights
potential argument applies to the weighted-average forecast.  This baseline
maintains all $400$ constituent weak learners.

To place all outputs on the same probability scale, we interpret each
method's signed score $s_t\in[-1,1]$ as the probability
$p_t=(1+s_t)/2$.  OGB projects its aggregate score to $[-1,1]$, OSBoost uses
its simplex-weighted vote, AdaBoost.OL uses its exponential-weights
randomization probability, and \textsc{Unboosted cls.} and Online BBM return
$s_t\in\{-1,1\}$; we apply no post-hoc calibration.  Thus the reported Brier
loss evaluates the probability induced directly by each online algorithm's
output.

For each algorithm family, we use one predeclared tuning rule rather than
optimizing parameters against the observed performance of each stream.  In
\textsc{Defensive}, both scalar states use Definition~\ref{def:scalar-ogd}
with $V_0=4$; no scalar learning rate is tuned.  The linear-loss oracles used
by \textsc{Defensive}, \textsc{Unboosted reg.}, and every OGB stage are
second-order.  For a finite class with $d$ base hypotheses, we run
entropy-FTRL over its symmetric closure with
\[
  \eta_t
  =
  \min\left\{.25,
  \sqrt{\frac{\log(2d)}{4+\sum_{s<t}c_s^2}}\right\},
\]
where $c_s$ is the scalar multiplying the weak prediction in the round-$s$ 
linear gain.  For the Euclidean unit ball, we use projected adaptive gradient
ascent with step $.5/\sqrt{4+\sum_{s<t}\|c_sx_s\|_2^2}$.  These choices give
the coefficient-energy regret bounds required by the theory.  The finite
classification learners used by \textsc{Unboosted cls.}, Online BBM,
AdaBoost.OL, and OSBoost instead share the horizon-aware Hedge scale
$\min\{.5,\sqrt{8\log(d)/T}\}$; their linear-class counterparts share the
same weighted projected-perceptron update with a $1/\sqrt t$ step scale.

All synthetic results average $20$ random seeds with $T=3000$ rounds, and
tables report mean $\pm$ standard error.  The real datasets have fixed
chronological order and deterministic algorithm updates, so we report one
run.  We never shuffle a real stream: each example is predicted in its
recorded order before it is used for the update.
The \texttt{experiments/} directory contains the generators, loaders,
algorithm implementations, and commands used for every reported result.  The
synthetic streams are original controlled constructions rather than
reproductions of prior benchmarks.

\paragraph{Runtime.}
Let $C_H$ be the cost of one prediction/update for the online learner over
the weak class.  Each unboosted control and the Defensive Booster costs
$C_H+O(1)$ per round.  OGB, Online BBM, AdaBoost.OL, and OSBoost each cost
$N C_H+O(N)$ with $N$ weak learners; OSBoost additionally pays an
$O(N\log N)$ simplex projection when the OCP combiner updates.  We use
$N=100$ for all ensemble baselines in the main comparison, and additionally
evaluate each ensemble baseline with $N\in\{1,5,20,100\}$ learners.  The
Brier aggregator runs all four ensembles and therefore costs
$4NC_H+O(N\log N)$ and maintains $4N$ weak learners.

\subsection{Controlled synthetic streams}
\label{sec:synthetic-details}

\paragraph{Synthetic streams.}
We use five streams to separate two questions: whether every smooth
reweighting has a weak rule with positive edge, and whether the span contains
an informative squared-loss comparator.  The planted-decoy and binary aggregation
streams satisfy the smooth weak-learning condition by construction.  The
three linear streams instead test span competition with noisy labels or weak
signal.

For the two finite-class streams, an algorithm uses a context only through
the vector of weak-rule values, so we generate that vector directly.  Round
$t$ is represented by $(h_1(x_t),\ldots,h_d(x_t))$, and the linear-loss
oracle uses the symmetric closure of the coordinate rules.  For the other
three streams, the weak class is the Euclidean unit ball
$\calH=\{x\mapsto\langle u,x\rangle:\|u\|_2\le1\}$.

In \emph{planted decoy}, $\sigma_t$ is uniform on $\{-1,1\}$, the useful rule is
$h_1(x_t)=s_t\sigma_t$ for an independent uniform
$s_t\in\{.12,1\}$, and the other $199$ rule scores are independent uniform
signs.  Thus the $200$ base rules (and their symmetric closure) contain a
sign-perfect rule with edge at least $.12$ under every reweighting.

In \emph{binary aggregation}, the displayed weak class consists of $100$
opposite pairs $\{\pm h_j\}_{j=1}^{100}$, and signed labels are balanced and
randomly ordered.  One latent orientation from each pair is designated useful.
On half the rounds all $100$ useful orientations equal the signed label.  On
the remaining rounds, we cycle through $100$ binary patterns: each pattern
has $58$ useful orientations correct and $42$ incorrect, and every useful
orientation occupies each position equally often.  Independent sign flips
and a column permutation hide the useful orientation in each pair; we also
randomly permute the rounds.  The algorithms receive only the resulting
vector of $200$ weak-rule values.  On every round,
\[
  \frac1{100}\sum_{j=1}^{100}\sigma_t h_j(x_t)\ge .16.
\]
Here the sum uses the latent useful orientation from each pair.  After
multiplying by any nonzero weights and averaging, at least one displayed rule
has edge $.16$.  No individual rule is perfect, and the average of all $200$
displayed rules is identically zero, whereas the hidden average above has
positive margin on every round.

Adding the opposite rules, applying sign flips, and permuting the columns do
not change the span.  The same symmetry therefore yields an exact span
obstruction.  Averaging any coefficient vector over cyclic shifts cannot
increase squared loss, so an optimal fixed span score assigns every latent
useful orientation the same coefficient.  Its signed score is $b$ on the
first half of the stream and $.16b$ on the second half.  Minimizing the
resulting Brier loss gives
\[
  \min_{b\in\mathbb R}
  \frac{(1-b)^2+(1-.16b)^2}{8}
  =\frac{(1-.16)^2}{8(1+.16^2)}=.0860.
\]
Thus the stream is a finite cyclic version of the binary separation in
Proposition~\ref{prop:separation}: aggregation classifies perfectly, but every
fixed affine span score has constant Brier loss.

For the three linear streams, draw $x_t\sim N(0,I_d)$ and normalize it to
unit Euclidean norm; draw and normalize a fixed $\beta\sim N(0,I_d)$.  In
\emph{linear span}, $d=40$ and
$\sigma_t=\operatorname{sign}(\langle\beta,x_t\rangle+\xi_t)$ for
$\xi_t\sim N(0,.02^2)$.  In \emph{random-label mixture}, $d=30$ and the same model
uses noise $N(0,.05^2)$, but independently on $35\%$ of rounds its label is
replaced by a uniform random sign.  The reweighting supported on the replaced
labels has density close to $.35$, and because those labels are independent
of the contexts, its weak-class edge tends to zero as $T$ grows.  The
structured rounds still admit an informative linear score.
In \emph{random labels}, $d=30$ and every signed label is an independent
uniform sign.  All draws are independent except where the construction
explicitly shares $\sigma_t$ or $\beta$.

\paragraph{Offline span diagnostic.}
For each realized synthetic transcript, let
$\widehat\beta\in\arg\min_\beta\sum_t(\sigma_t-\langle\beta,x_t\rangle)^2$.
The induced real-valued prediction is
$q_{\widehat\beta}(x)=(1+\langle\widehat\beta,x\rangle)/2$, whose average
squared loss is exactly
\[
  \frac1T\sum_{t=1}^T(Y_t-q_{\widehat\beta}(x_t))^2
  =
  \frac1{4T}\sum_{t=1}^T
  (\sigma_t-\langle\widehat\beta,x_t\rangle)^2.
\]
We report this loss and a norm
$\Lambda_{\rm LS}$ witnessing membership in the comparator class of
Theorem~\ref{thm:span-guarantee}: $\|\widehat\beta\|_1$ for a finite
coordinate class and $\|\widehat\beta\|_2$ for the Euclidean linear ball.
The score is not clipped, so it is an unrestricted span comparator covered by
Theorem~\ref{thm:span-guarantee}.  Because $\widehat\beta$ is fit after observing the full
transcript, it is a diagnostic benchmark rather than an online algorithm.

Figure~\ref{fig:weak-experiments} reports randomized classification error on
the two streams satisfying the smooth weak-learning condition.
Figure~\ref{fig:brier-experiments} reports Brier loss on the three linear
streams.  Tables~\ref{tab:span-diagnostics} and~\ref{tab:experiments} give the
corresponding span benchmarks and complete numerical results.

\begin{figure}[H]
\centering
\includegraphics[width=.48\linewidth]{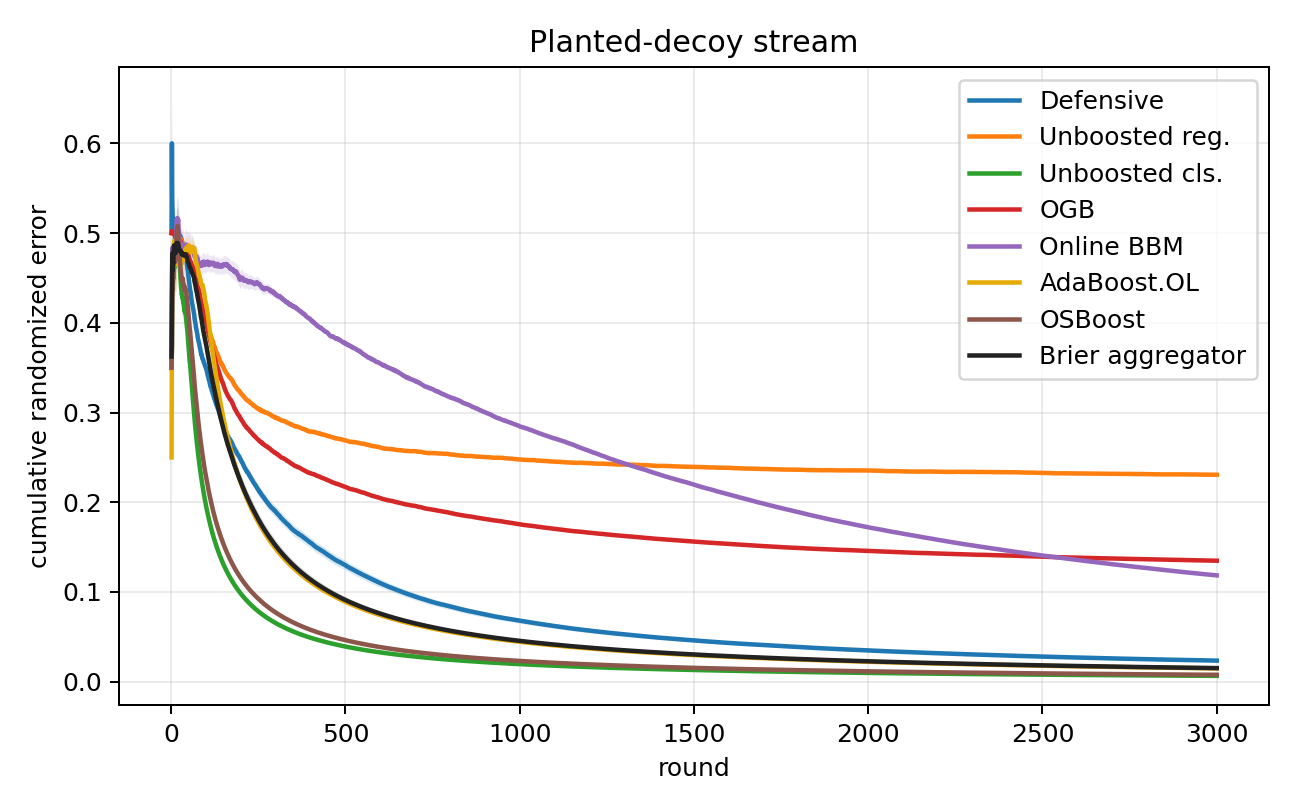}
\hfill
\includegraphics[width=.48\linewidth]{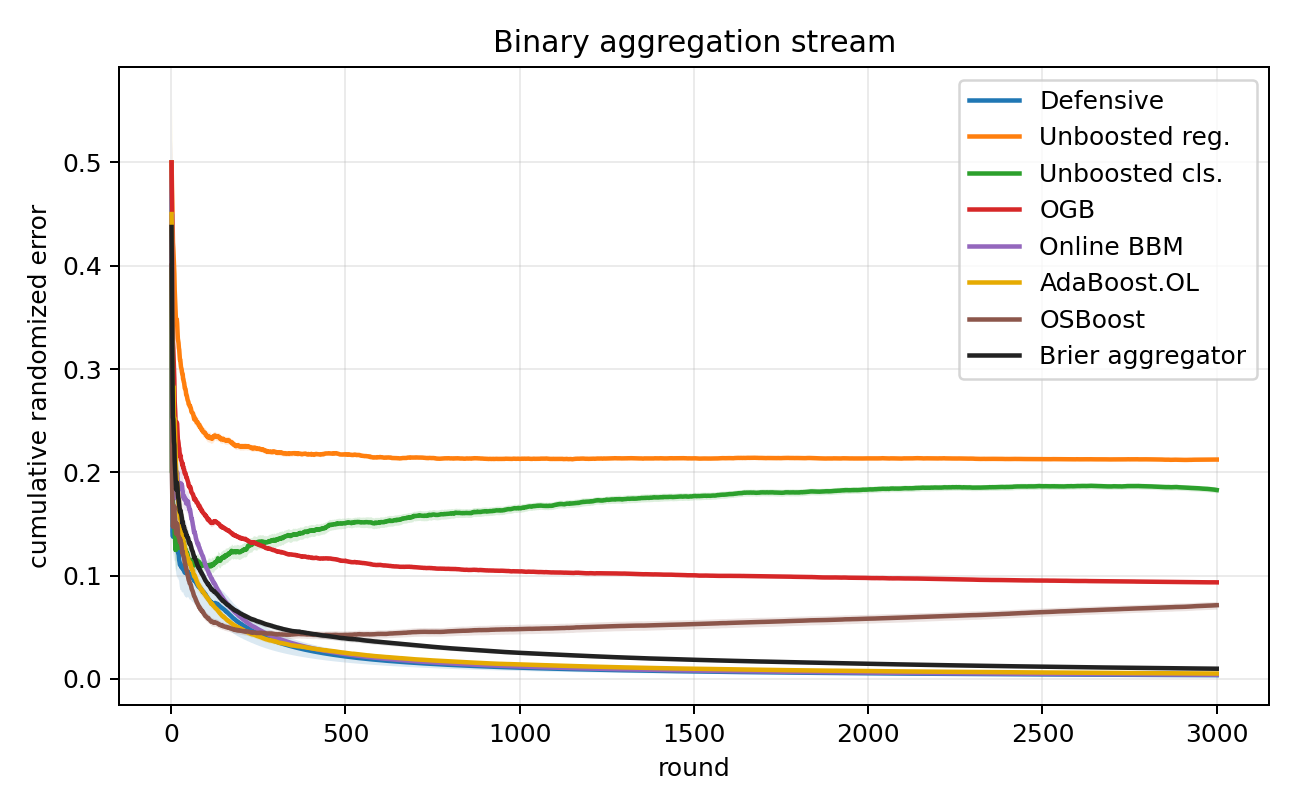}
\caption{Randomized classification error on the streams satisfying the smooth
weak-learning condition.  Left: on planted decoy, one base rule has the
correct sign on every round but is hidden among $199$ decoys; the unboosted
classifier and OSBoost identify it fastest, and the Defensive Booster also
reaches low error.  Right: on binary aggregation, no individual rule is
perfect and the average of all displayed rules is zero, but a hidden choice
of one orientation from each opposite pair has positive margin.  The
Defensive Booster and the three weak-to-strong ensembles all attain low
randomized error; the Defensive Booster uses one weak learner.}
\label{fig:weak-experiments}
\end{figure}

\begin{figure}[H]
\centering
\includegraphics[width=.48\linewidth]{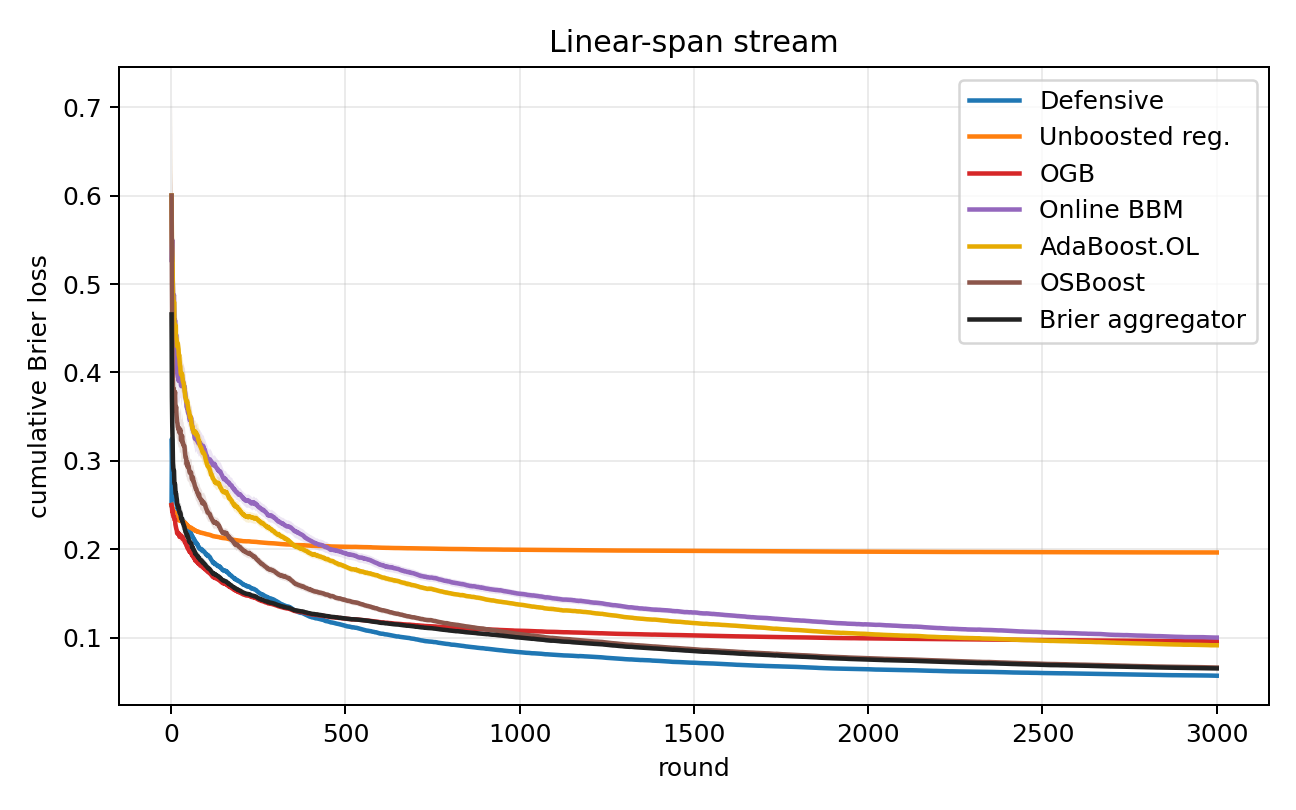}
\hfill
\includegraphics[width=.48\linewidth]{figures/random_label_mixture_brier_loss.png}

\vspace{.6em}
\includegraphics[width=.48\linewidth]{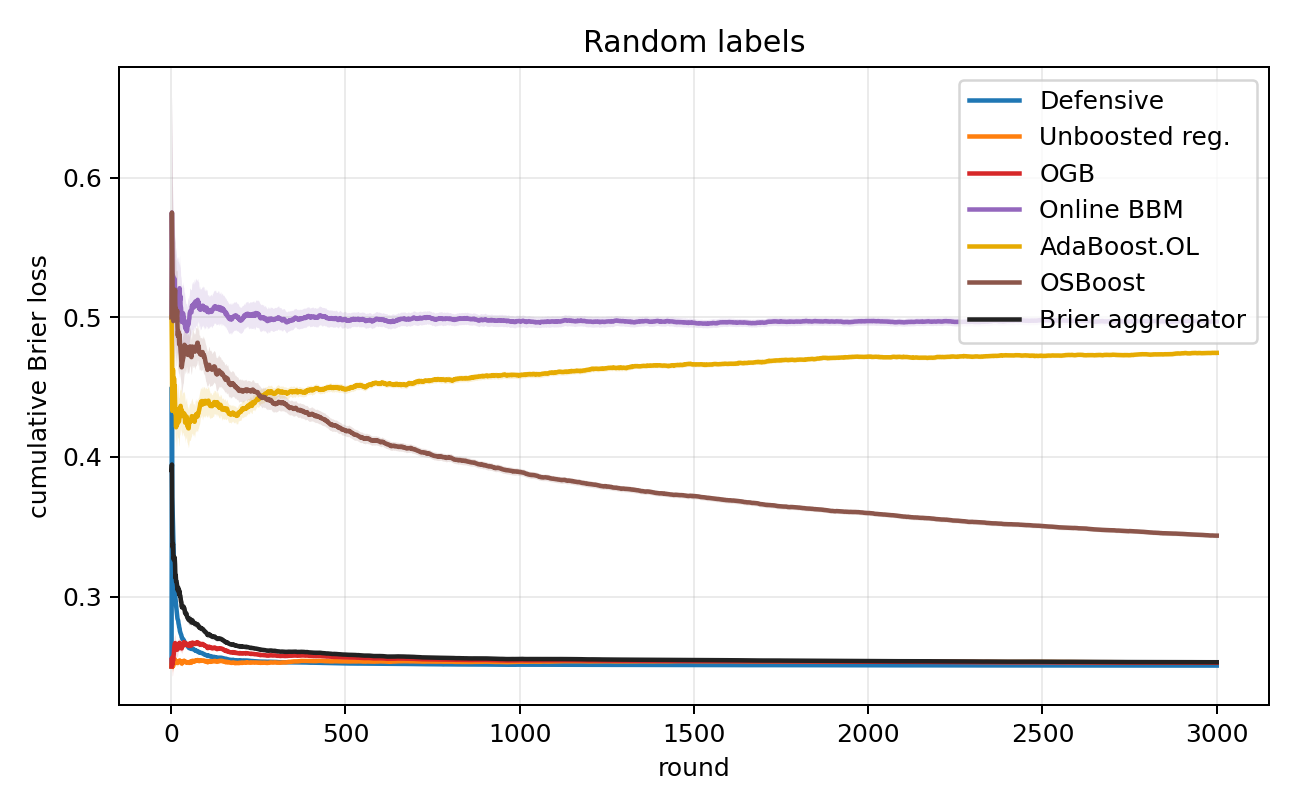}
\caption{Brier loss on the three linear streams.  Top left: on linear span,
the span comparator is informative even though the smooth weak-learning
condition can fail on near-margin examples.  Top right: on random-label
mixture, the structured rounds retain an informative span comparator; OGB,
unboosted regression, and the Defensive Booster outperform the
classification-boosting baselines.  Bottom: on random labels, the Defensive Booster, unboosted
regression, and OGB remain near the $p=1/2$ Brier baseline.  The hard-label
unboosted classifier is omitted from these plots for scale and reported in
Table~\ref{tab:experiments}.}
\label{fig:brier-experiments}
\end{figure}

\begin{table}[H]
\centering
\small
\begin{tabular}{lcc}
\hline
Stream & LS span squared loss & $\Lambda_{\rm LS}$ \\
\hline
Planted decoy & $.0887\pm.0004$ & $2.982\pm.020$ \\
Binary aggregation & $.0860\pm.0000$ & $1.131\pm.000$ \\
Linear span & $.0900\pm.0003$ & $5.084\pm.017$ \\
Random-label mixture & $.1847\pm.0007$ & $2.794\pm.018$ \\
Random labels & $.2472\pm.0001$ & $.576\pm.013$ \\
\hline
\end{tabular}
\caption{Offline least-squares span diagnostics, averaged over the same
$20$ seeds used for the online results.  The score is not clipped.  The
reported $\Lambda_{\rm LS}$ is the $\ell_1$ norm of the fitted coefficient
vector for a finite class or the
$\ell_2$ norm for the Euclidean linear class.  Hence each row is an actual
$\spanop_{\Lambda_{\rm LS}}(\calH)$ comparator covered by
Theorem~\ref{thm:span-guarantee}.}
\label{tab:span-diagnostics}
\end{table}

\clearpage

\begin{table}[H]
\centering
\scriptsize
\begin{tabular}{llccc}
\hline
Stream & Algorithm & 0/1 & Brier & Rand. err. \\
\hline
Planted decoy & Defensive & $.0166\pm.0008$ & $.0119\pm.0005$ & $.0235\pm.0010$ \\
Planted decoy & Unboosted reg. & $.0072\pm.0004$ & $.0999\pm.0004$ & $.2306\pm.0009$ \\
Planted decoy & Unboosted cls. & $.0065\pm.0003$ & $.0065\pm.0003$ & $.0065\pm.0003$ \\
Planted decoy & OGB & $.0162\pm.0005$ & $.0390\pm.0002$ & $.1348\pm.0005$ \\
Planted decoy & BBM & $.1185\pm.0008$ & $.1185\pm.0008$ & $.1185\pm.0008$ \\
Planted decoy & AdaBoost.OL & $.0147\pm.0004$ & $.0135\pm.0004$ & $.0148\pm.0004$ \\
Planted decoy & OSBoost & $.0073\pm.0003$ & $.0071\pm.0002$ & $.0077\pm.0003$ \\
Planted decoy & Brier agg. & $.0096\pm.0005$ & $.0069\pm.0003$ & $.0152\pm.0006$ \\
\hline
Binary aggregation & Defensive & $.0026\pm.0006$ & $.0018\pm.0005$ & $.0036\pm.0009$ \\
Binary aggregation & Unboosted reg. & $.1026\pm.0013$ & $.0944\pm.0001$ & $.2124\pm.0002$ \\
Binary aggregation & Unboosted cls. & $.1829\pm.0028$ & $.1829\pm.0028$ & $.1829\pm.0028$ \\
Binary aggregation & OGB & $.0331\pm.0009$ & $.0338\pm.0002$ & $.0935\pm.0003$ \\
Binary aggregation & BBM & $.0041\pm.0002$ & $.0041\pm.0002$ & $.0041\pm.0002$ \\
Binary aggregation & AdaBoost.OL & $.0042\pm.0004$ & $.0033\pm.0002$ & $.0052\pm.0005$ \\
Binary aggregation & OSBoost & $.0068\pm.0007$ & $.0192\pm.0014$ & $.0713\pm.0037$ \\
Binary aggregation & Brier agg. & $.0025\pm.0002$ & $.0025\pm.0001$ & $.0099\pm.0007$ \\
\hline
Linear span & Defensive & $.0744\pm.0007$ & $.0570\pm.0004$ & $.1120\pm.0008$ \\
Linear span & Unboosted reg. & $.0952\pm.0010$ & $.1964\pm.0001$ & $.4403\pm.0002$ \\
Linear span & Unboosted cls. & $.1233\pm.0015$ & $.1233\pm.0015$ & $.1233\pm.0015$ \\
Linear span & OGB & $.0866\pm.0010$ & $.0960\pm.0003$ & $.2511\pm.0006$ \\
Linear span & BBM & $.1002\pm.0014$ & $.1002\pm.0014$ & $.1002\pm.0014$ \\
Linear span & AdaBoost.OL & $.0966\pm.0011$ & $.0915\pm.0009$ & $.0971\pm.0010$ \\
Linear span & OSBoost & $.0849\pm.0011$ & $.0663\pm.0008$ & $.0953\pm.0011$ \\
Linear span & Brier agg. & $.0815\pm.0011$ & $.0653\pm.0008$ & $.1327\pm.0028$ \\
\hline
Random-label mixture & Defensive & $.2667\pm.0017$ & $.1965\pm.0007$ & $.3890\pm.0015$ \\
Random-label mixture & Unboosted reg. & $.2671\pm.0017$ & $.2158\pm.0003$ & $.4576\pm.0003$ \\
Random-label mixture & Unboosted cls. & $.2789\pm.0018$ & $.2789\pm.0018$ & $.2789\pm.0018$ \\
Random-label mixture & OGB & $.2656\pm.0013$ & $.1933\pm.0007$ & $.3859\pm.0013$ \\
Random-label mixture & BBM & $.2963\pm.0021$ & $.2963\pm.0021$ & $.2963\pm.0021$ \\
Random-label mixture & AdaBoost.OL & $.2758\pm.0022$ & $.2708\pm.0021$ & $.2757\pm.0022$ \\
Random-label mixture & OSBoost & $.3260\pm.0024$ & $.2467\pm.0012$ & $.3655\pm.0018$ \\
Random-label mixture & Brier agg. & $.2658\pm.0013$ & $.1937\pm.0007$ & $.3859\pm.0014$ \\
\hline
Random labels & Defensive & $.4980\pm.0020$ & $.2506\pm.0000$ & $.4999\pm.0001$ \\
Random labels & Unboosted reg. & $.4970\pm.0016$ & $.2529\pm.0001$ & $.4995\pm.0002$ \\
Random labels & Unboosted cls. & $.4978\pm.0015$ & $.4978\pm.0015$ & $.4978\pm.0015$ \\
Random labels & OGB & $.4968\pm.0019$ & $.2527\pm.0001$ & $.4995\pm.0002$ \\
Random labels & BBM & $.4971\pm.0025$ & $.4971\pm.0025$ & $.4971\pm.0025$ \\
Random labels & AdaBoost.OL & $.4983\pm.0020$ & $.4745\pm.0014$ & $.4995\pm.0019$ \\
Random labels & OSBoost & $.4994\pm.0024$ & $.3437\pm.0015$ & $.4977\pm.0016$ \\
Random labels & Brier agg. & $.4966\pm.0018$ & $.2530\pm.0001$ & $.4994\pm.0002$ \\
\hline
\end{tabular}
\caption{Average online performance over $20$ random seeds, reported as
mean $\pm$ standard error.  ``Rand. err.''
is $T^{-1}\sum_t |Y_t-p_t|$, the error of the randomized classifier induced
by the probability forecast.  \textsc{Unboosted cls.} and BBM output hard
labels, so their Brier and randomized-error entries equal their $0/1$ error.
\textsc{Brier agg.} combines the forecasts of OGB, BBM, AdaBoost.OL, and OSBoost before
observing the current label, then updates its weights after the label is
revealed; it therefore maintains $400$ weak learners.  AdaBoost.OL's Brier
entry scores its probability of predicting one, while its randomized-error
entry is the expected $0/1$ loss of its randomized classifier.
Entries are rounded to four decimals.}
\label{tab:experiments}
\end{table}

Online BBM is designed to output a hard majority prediction.  To check whether
its Brier results are merely an artifact of that convention, we also score the
normalized raw vote $(1+N^{-1}\sum_i h_{t,i}(x_t))/2$ as a probability.  This
diagnostic is not the output analyzed by the Online BBM theorem.
Figure~\ref{fig:bbm-vote-diagnostic} shows that the normalized vote lowers
planted-decoy Brier loss from $.1185$ to $.0838$, but remains far above the
best forecasting methods.  On binary aggregation, the hard output has Brier
loss and randomized error $.0041$, while the normalized vote has Brier loss
$.0520$ and randomized error $.1562$.  Thus a softer output helps on planted
decoy but hurts substantially on binary aggregation; it does not account for
the main comparisons.

\begin{figure}[H]
\centering
\includegraphics[width=.78\linewidth]{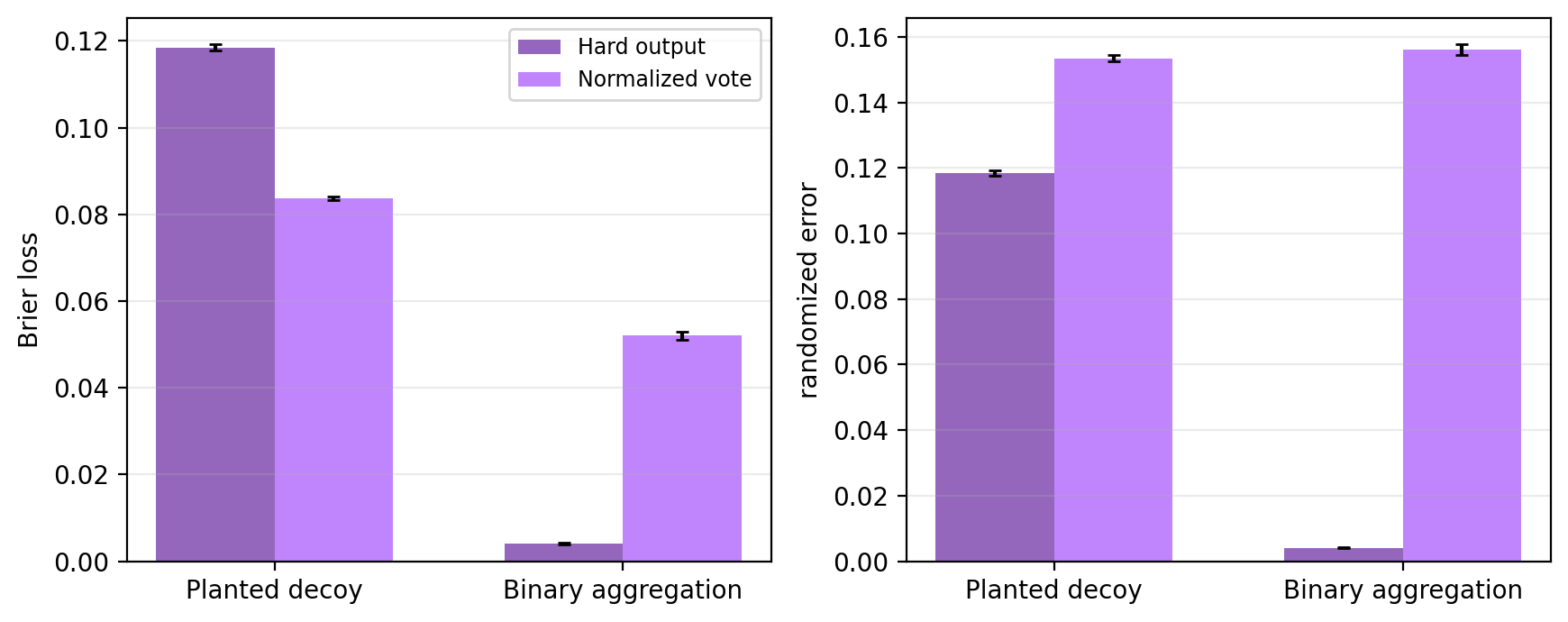}
\caption{Online BBM's specified hard output versus its normalized raw vote,
averaged over $20$ seeds.  The vote is scored directly as a probability, with
no post-hoc calibration.  It improves Brier loss on planted decoy but worsens
randomized error there; on binary aggregation, the specified hard output is
substantially better under both metrics.}
\label{fig:bbm-vote-diagnostic}
\end{figure}

\begin{table}[H]
\centering
\small
\begin{tabular}{lrrr}
\hline
Algorithm & Weak learners & Seconds per stream & Microseconds per round \\
\hline
Unboosted reg. & 1 & .032 & 11 \\
Unboosted cls. & 1 & .030 & 10 \\
Defensive & 1 & .046 & 15 \\
OGB & 100 & 2.959 & 986 \\
BBM & 100 & 1.478 & 493 \\
AdaBoost.OL & 100 & 2.292 & 764 \\
OSBoost & 100 & 1.753 & 584 \\
Brier aggregator & 400 & 8.482 & 2827 \\
\hline
\end{tabular}
\caption{Runtime summary for the implementation used in the experiments.
Wall-clock times are averages over all stream/seed runs, each with
$T=3000$ rounds.  The exact constants are implementation-dependent, but the
difference in maintained learners is structural: both unboosted controls and
the Defensive Booster maintain one online learner, whereas the ensemble
baselines maintain many weak learners in parallel.  The Brier aggregator's
time is the sum of its four constituent ensembles because all must be run.}
\label{tab:runtime}
\end{table}

\paragraph{Sensitivity to baseline parameters.}
The main experiments use one setting for each algorithm family on every
stream.  Figure~\ref{fig:parameter-sensitivity} varies, one at a time, the OGB
stage step, the classification learner's Hedge rate, and the target-edge
parameter on the binary aggregation stream.  Each value is $.25$, $.5$, $1$, $2$, or $4$ times
the reported setting, and each point averages $10$ seeds.  OGB improves
steadily with its stage step but has higher Brier loss than the Defensive
Booster at every tested value.  Online BBM is stable across both sweeps, and
AdaBoost.OL has slightly lower randomized error than the Defensive Booster at
the smallest tested learning rate.  OSBoost is substantially more sensitive
to both its classification learning rate and target edge.  Thus the main
comparison does not depend on a single narrow baseline setting, although the
relative ordering of the lowest-error classification methods can change.

\begin{figure}[H]
\centering
\includegraphics[width=.98\linewidth]{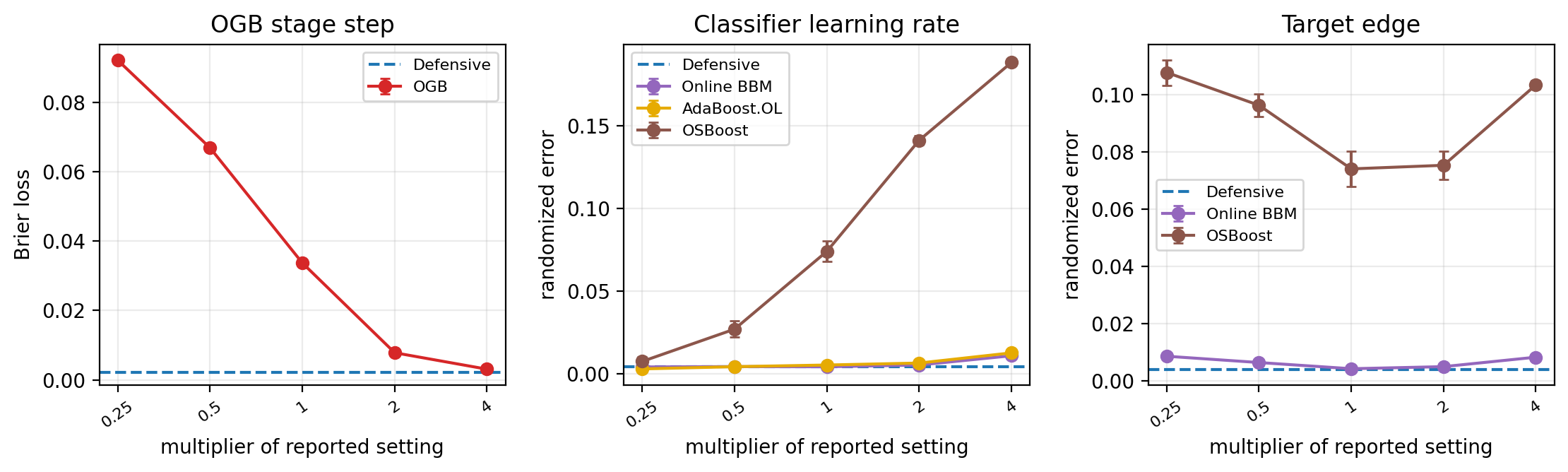}
\caption{One-at-a-time parameter sensitivity on the binary aggregation stream.  The
dashed blue line is the parameter-free Defensive Booster.  The middle and
right panels show randomized classification error; the left panel shows Brier
loss.  Error bars are standard errors over $10$ seeds.  A target-edge
multiplier above one deliberately overstates the edge guaranteed by the
construction and is included as a misspecification check.}
\label{fig:parameter-sensitivity}
\end{figure}

Figure~\ref{fig:real-parameter-sensitivity} repeats the same sixteen-fold
sweep on Electricity and Occupancy, the two real streams on which the
Defensive Booster has the largest advantage.  It retains the lowest Brier
loss under every tested setting.  This check does not tune the reported
results: the main tables continue to use multiplier one for every dataset.

\begin{figure}[H]
\centering
\includegraphics[width=.98\linewidth]{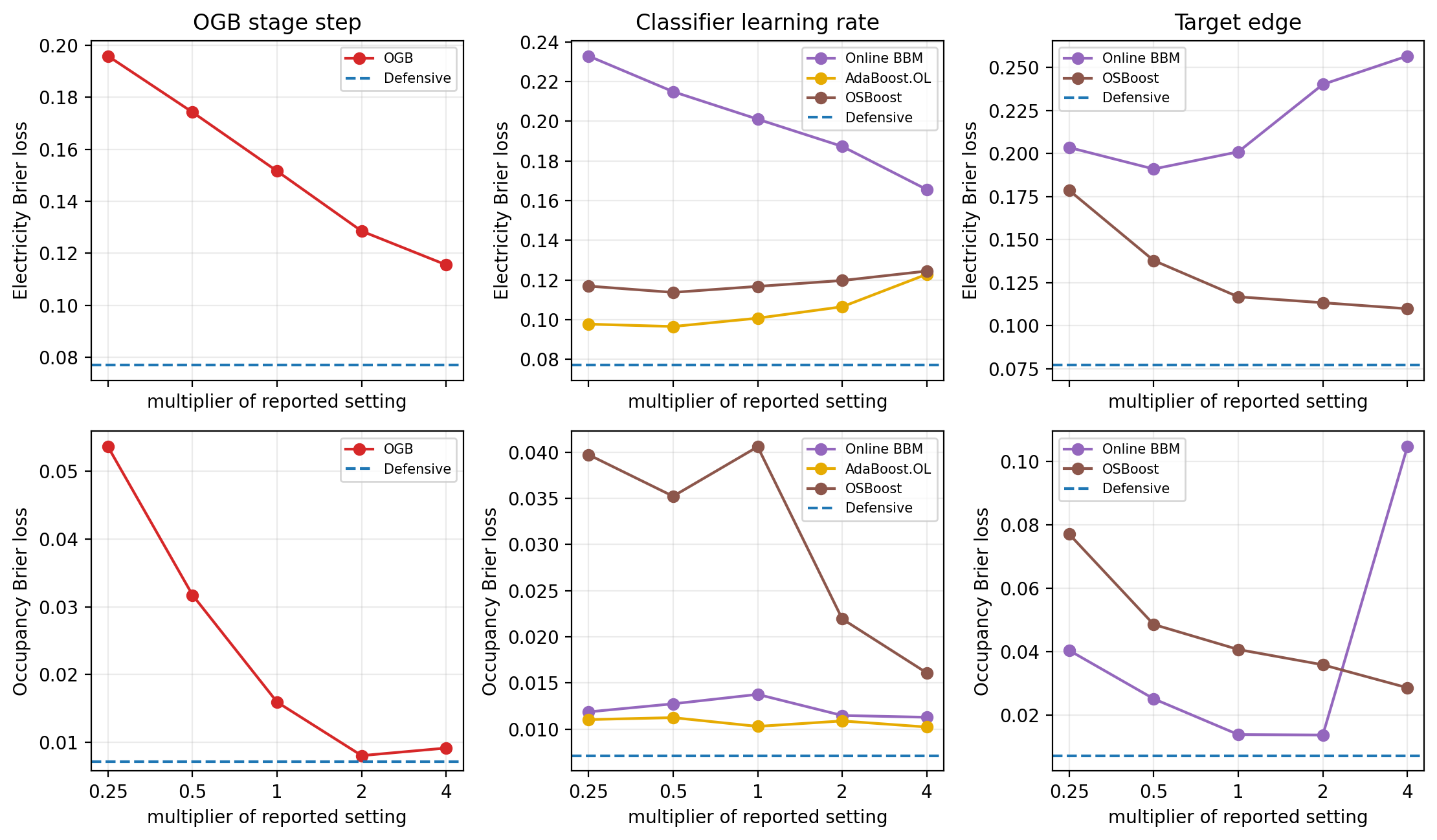}
\caption{One-at-a-time parameter sensitivity on the complete Electricity
(top) and Occupancy (bottom) streams.  Multipliers range from $.25$ to $4$.
The dashed blue line is the parameter-free Defensive Booster.  Every panel
shows final Brier loss; lower is better.  AdaBoost.OL has no target-edge
parameter and therefore appears only in the middle column.}
\label{fig:real-parameter-sensitivity}
\end{figure}

\paragraph{Takeaways.}
The planted-decoy stream tests identification within a large weak class, but
does not require boosting: one base rule already has the correct sign on every
round.  Accordingly, the unboosted classifier and OSBoost have the smallest
errors, and the Defensive Booster also attains low Brier and randomized error.
Unboosted regression has low threshold error but poor Brier and randomized
error, illustrating why probability forecasting is stricter than threshold
accuracy.  The unrestricted least-squares span score has loss $.0887$, whereas
the Defensive Booster reaches $.0119$; this weak-to-strong behavior is not
explained by span fitting.

The binary aggregation stream requires genuine aggregation: no individual
binary rule is perfect, and averaging all displayed rules gives zero.  A
hidden choice of one orientation from each pair nevertheless classifies every
round correctly.  The unboosted classifier has error $.1829$, while the
Defensive Booster reaches Brier loss $.0018$ and randomized error $.0036$
using one weak learner, despite the $.0860$ loss floor for fixed affine span
scores.  At $N=100$, all three weak-to-strong ensembles also reach low hard
error, but each has higher Brier loss and randomized error than the Defensive
Booster.  The ensemble-size comparison in Figure~\ref{fig:compute-sweep}
shows how these values compare with each ensemble as its number of weak
learners grows.

The three linear streams in Figure~\ref{fig:brier-experiments} examine Brier
loss when the smooth weak-learning condition need not hold.  On linear span,
the offline span score has loss $.0900$, and the Defensive Booster attains
$.0570$.  On random-label mixture, the randomly labeled rounds provide a
smooth reweighting with small weak-class edge, while the structured rounds
retain an informative span comparator.  OGB, the unboosted regressor, and the
Defensive Booster have substantially lower Brier loss than the unboosted
classifier and the classification-boosting baselines.  OSBoost and Online BBM remain more competitive in randomized error
than in Brier loss because randomized error depends linearly on their signed
margins.  Figure~\ref{fig:certificate-diagnostic} examines the
random-label mixture more closely: the Defensive Booster's multiaccuracy and
self-orthogonality errors and the weak-class edge under its mistake weighting
all decay, although the weighting retains nontrivial density.
Finally, on random labels, no method has a real signal; the Defensive
Booster, unboosted regressor, and OGB stay near the $p=1/2$ Brier-loss
baseline, while the unboosted classifier and the classification boosters have
larger Brier loss.  The weak-class edge under the Defensive Booster's mistake
weighting remains small.

Under these globally fixed tuning rules, the classification boosters are strongest
on the binary aggregation stream, while OGB is strongest on the random-label
mixture.  The Defensive Booster is competitive with the better family in both
comparisons.  Table~\ref{tab:runtime} shows that it does so while maintaining
one online weak learner rather than an ensemble of $100$ learners.

\subsection{Naturally ordered real streams}
\label{sec:real-data-details}

\paragraph{Data and preprocessing.}
We evaluate four public binary prediction streams in their recorded order.
UCI Bank Marketing predicts whether a client subscribes to a term deposit
\citep{MoroRitaCortez2012BankMarketing}; we use the date order supplied by the
full dataset.  The MOA Electricity stream \citep{MOADatasets} predicts price
movement in the New South Wales electricity market; the data originate in
Harries's electricity-pricing study \citep{Harries1999Electricity}.
The MOA Airlines stream predicts whether a flight is delayed
\citep{MOADatasets}.  UCI Occupancy Detection
\citep{CandanedoFeldheim2016Occupancy} predicts whether an office is occupied
from minute-level sensor measurements, which we merge by recorded timestamp.
We do not shuffle any dataset: on
each round the algorithm receives the next context, predicts, and then
observes its label.

We use a $128$-dimensional deterministic signed-hash representation and the
Euclidean unit-ball weak class for all four datasets.  Each row includes a
bias feature and is normalized to unit norm.  Categorical values are hashed as
indicators; numeric values are standardized using unlabeled covariates from
earlier rounds, clipped to five running standard deviations, and then hashed.
The current numeric value is transformed using the mean and variance of that
feature among preceding contexts and is incorporated into those statistics
only afterward.  Thus no feature or label from a future round enters the
current representation.  For Bank Marketing we drop call duration, which
is unavailable before the outcome.  Occupancy uses contemporaneous
temperature, humidity, light, and CO$_2$ measurements, the derived humidity
ratio, and cyclic encodings of the recorded time and weekday.  The Bank,
Electricity, Airlines, and Occupancy runs use all $41{,}188$, $45{,}312$,
$539{,}383$, and $20{,}560$ examples, respectively.

Figure~\ref{fig:real-data} plots cumulative average Brier loss over each
stream.  Table~\ref{tab:real-data} reports the final Brier, deterministic, and
randomized classification errors, together with runtime.

\begin{figure}[H]
\centering
\includegraphics[width=.48\linewidth]{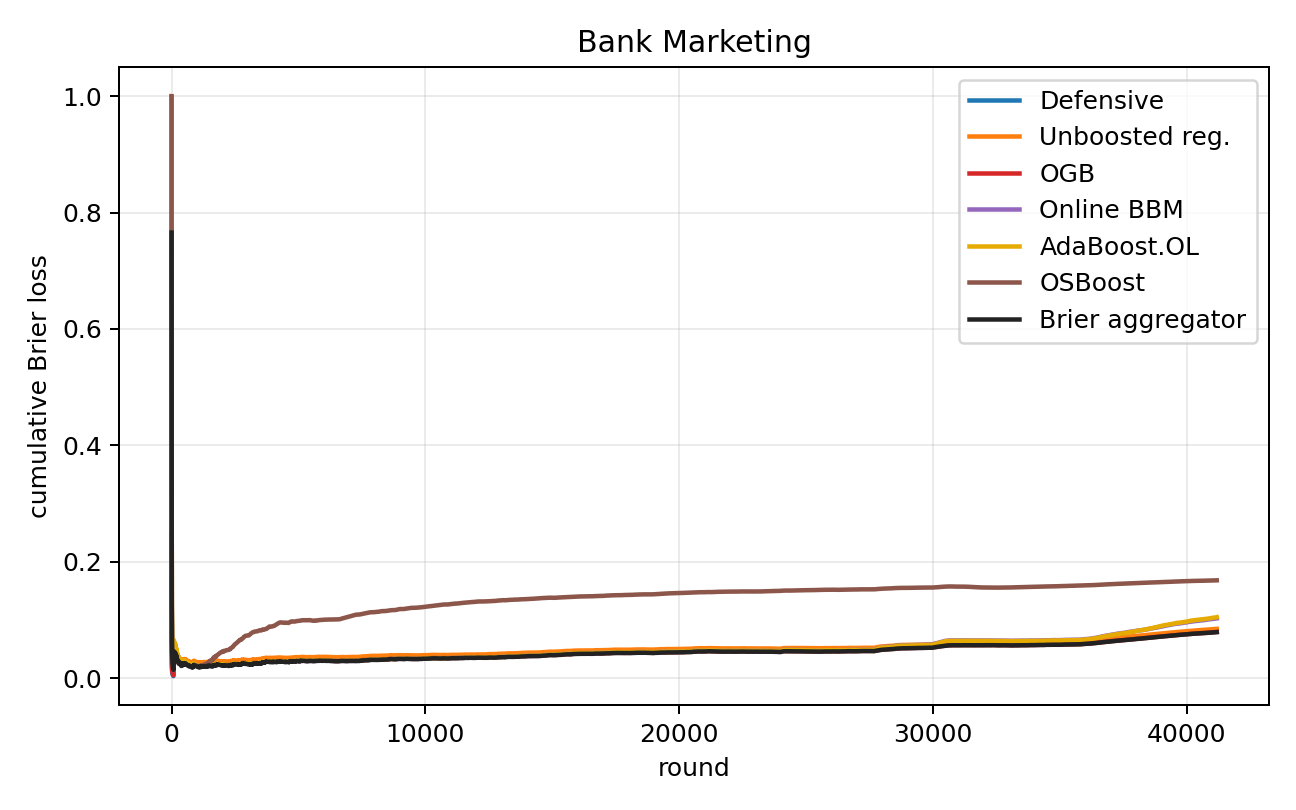}
\hfill
\includegraphics[width=.48\linewidth]{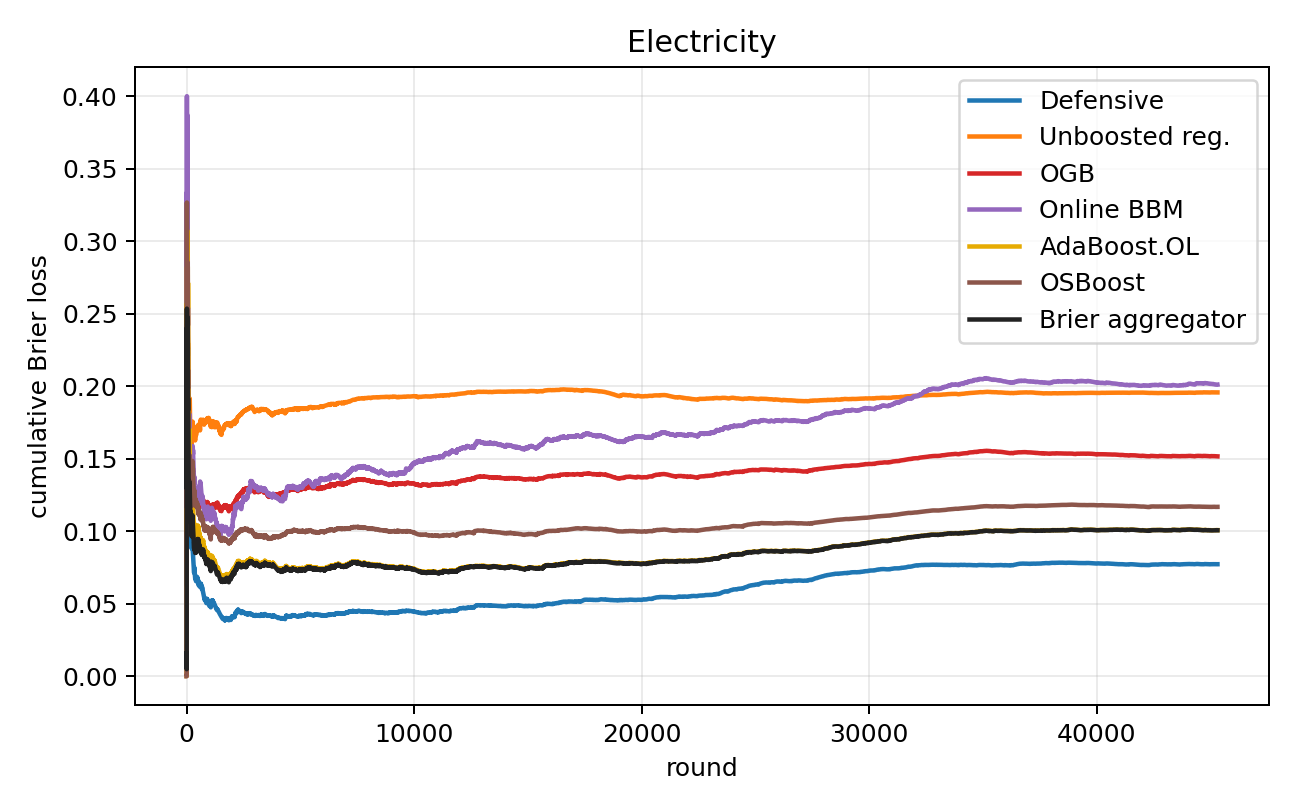}

\vspace{.6em}
\includegraphics[width=.48\linewidth]{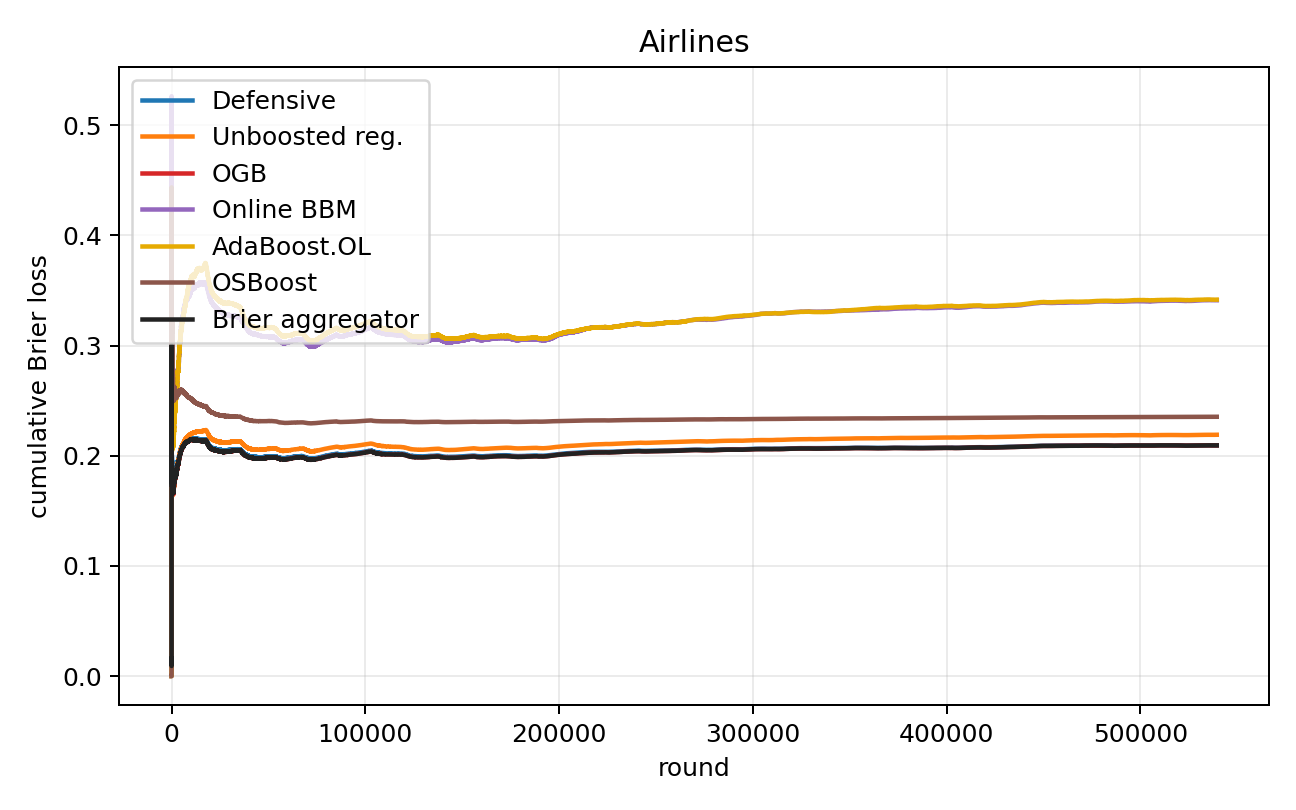}
\hfill
\includegraphics[width=.48\linewidth]{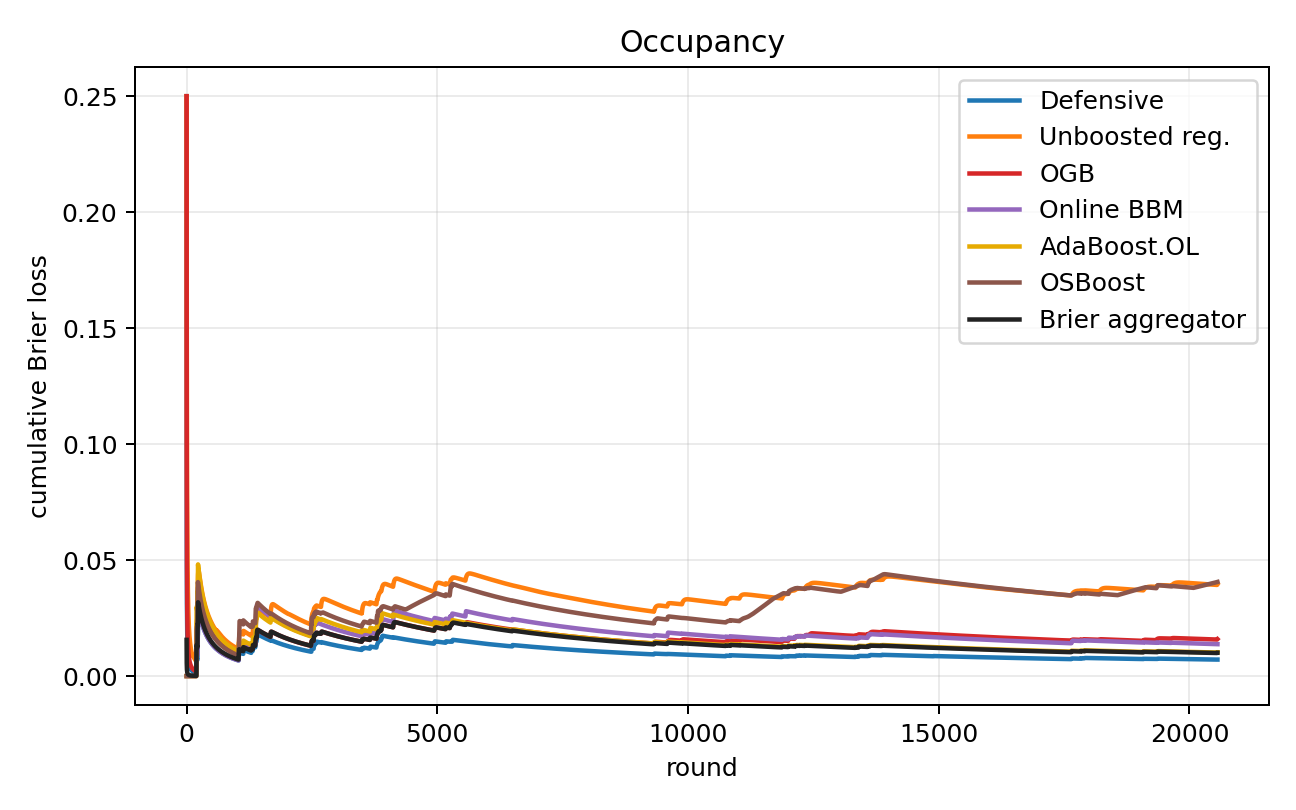}
\caption{Cumulative average Brier loss on the real-data streams.  Top left:
on Bank Marketing, OGB
and the Defensive Booster are close.  Top right: on Electricity, the
Defensive Booster has substantially lower Brier loss than all six plotted
baselines.  Bottom left: on Airlines, the Defensive Booster, OGB, and the
Brier aggregator are nearly indistinguishable.  Bottom right: on Occupancy,
the Defensive Booster has the lowest Brier loss.
Each curve is the average loss incurred up to that point while processing the
dataset in recorded order.  The hard-label
unboosted classifier is omitted for scale and reported in
Table~\ref{tab:real-data}.}
\label{fig:real-data}
\end{figure}

\begin{table}[H]
\centering
\small
\begin{tabular}{llrrrrr}
\hline
Dataset & Algorithm & 0/1 & Brier & Base & Rand. err. & $\mu$s/round \\
\hline
Bank & Defensive & .103 & .080 & .100 & .159 & 14 \\
Bank & Adaptive Def. & .103 & .081 & .100 & .159 & 86 \\
Bank & Unboosted reg. & .103 & .085 & .100 & .212 & 10 \\
Bank & Unboosted cls. & .109 & .109 & .100 & .109 & 8 \\
Bank & OGB & .102 & .079 & .100 & .157 & 803 \\
Bank & BBM & .103 & .103 & .100 & .103 & 300 \\
Bank & AdaBoost.OL & .106 & .105 & .100 & .106 & 522 \\
Bank & OSBoost & .131 & .168 & .100 & .387 & 356 \\
Bank & Brier agg. & .102 & .079 & .100 & .148 & 1981 \\
\hline
Electricity & Defensive & .108 & .077 & .244 & .154 & 15 \\
Electricity & Adaptive Def. & .085 & .064 & .244 & .129 & 86 \\
Electricity & Unboosted reg. & .273 & .196 & .244 & .422 & 11 \\
Electricity & Unboosted cls. & .367 & .367 & .244 & .367 & 8 \\
Electricity & OGB & .219 & .152 & .244 & .310 & 807 \\
Electricity & BBM & .201 & .201 & .244 & .201 & 329 \\
Electricity & AdaBoost.OL & .102 & .101 & .244 & .102 & 526 \\
Electricity & OSBoost & .111 & .117 & .244 & .280 & 349 \\
Electricity & Brier agg. & .102 & .101 & .244 & .102 & 2011 \\
\hline
Airlines & Defensive & .332 & .209 & .247 & .419 & 16 \\
Airlines & Adaptive Def. & .324 & .207 & .247 & .413 & 89 \\
Airlines & Unboosted reg. & .348 & .219 & .247 & .451 & 11 \\
Airlines & Unboosted cls. & .390 & .390 & .247 & .390 & 10 \\
Airlines & OGB & .332 & .209 & .247 & .419 & 826 \\
Airlines & BBM & .341 & .341 & .247 & .341 & 352 \\
Airlines & AdaBoost.OL & .342 & .342 & .247 & .342 & 543 \\
Airlines & OSBoost & .370 & .235 & .247 & .480 & 362 \\
Airlines & Brier agg. & .332 & .209 & .247 & .419 & 2082 \\
\hline
Occupancy & Defensive & .009 & .007 & .178 & .014 & 14 \\
Occupancy & Adaptive Def. & .007 & .007 & .178 & .013 & 85 \\
Occupancy & Unboosted reg. & .052 & .040 & .178 & .120 & 9 \\
Occupancy & Unboosted cls. & .096 & .096 & .178 & .096 & 8 \\
Occupancy & OGB & .022 & .016 & .178 & .033 & 838 \\
Occupancy & BBM & .014 & .014 & .178 & .014 & 288 \\
Occupancy & AdaBoost.OL & .011 & .010 & .178 & .011 & 536 \\
Occupancy & OSBoost & .015 & .041 & .178 & .088 & 382 \\
Occupancy & Brier agg. & .012 & .010 & .178 & .016 & 2044 \\
\hline
\end{tabular}
\caption{Average online performance on the real-data streams.  ``Base'' is
the Brier loss of the best constant probability forecast on the evaluated
dataset.
\textsc{Unboosted cls.} and BBM report the Brier score of their hard binary
predictions.  \textsc{Brier agg.} runs all four ensemble baselines.  Runtime is
reported for the same implementation and machine as Table~\ref{tab:runtime}.
The absolute times are implementation-dependent.  The basic Defensive
Booster maintains one weak-class learner, \textsc{Adaptive Def.} maintains
one at each active dyadic scale (20 for Airlines, 17 for Bank and Electricity,
and 16 for Occupancy), each ensemble baseline maintains $100$, and the Brier
aggregator maintains $400$.}
\label{tab:real-data}
\end{table}

Among the eight methods in the main comparison, the Defensive Booster has the
lowest Brier loss on Electricity and Occupancy.  On Occupancy its Brier loss
$.007$ is less than half OGB's $.016$; it also has the lowest deterministic
error.  AdaBoost.OL has the lowest randomized error on Occupancy and the
lowest two classification errors on Electricity.  These distinctions are
consistent with the algorithms' objectives: the Defensive Booster is
designed to forecast probabilities, while AdaBoost.OL directly optimizes
classification.  The Brier gains are
not explained merely by maintaining fewer learners: both unboosted controls
are substantially worse on Electricity and Occupancy.  The Brier aggregator
is best on Bank by $.0010$ over the Defensive Booster.  On Airlines, OGB has
the numerically smallest Brier loss, but it, the aggregator, and the Defensive
Booster differ by less than $6\cdot10^{-5}$.  The unboosted regressor is also
competitive on these two streams, while the classification methods are worse
on Brier loss.  On Electricity and Occupancy, however, the Defensive Booster's
Brier loss is respectively $61\%$ and $82\%$ lower than the unboosted
regressor's.  The Brier aggregator maintains
$400$ weak learners.
On the four real streams, the $N=100$ ensembles take $20$--$60\times$ as much
wall-clock time per round as the Defensive Booster, which maintains one weak
learner; the Brier aggregator takes about $130$--$150\times$ as much.  The
separate Brier and classification columns matter here: a boosted margin method
can classify accurately without producing the most accurate probability
forecasts.

\section{Extension to bounded real-valued outcomes}
\label{sec:bounded-outcomes}

It is evident that the Defensive Booster, as it is defined, does not require that the outcomes be binary; it can be applied essentially without modification in a regression setting with bounded scalar labels. We will now briefly state and discuss this natural extension: in a nutshell, the Defensive Booster's span guarantee is satisfied in the exact same way as in the binary setting.
Then, we will
evaluate this extension on three
chronological regression streams --- in these experiments, the Defensive
Booster has $17$--$29\%$ lower normalized mean squared error than 100-stage online
gradient boosting while maintaining one weak-class learner; OGB takes
$65$--$70\times$ as much wall-clock time per round in our implementation.

To begin, suppose that $Y_t\in[0,1]$ and that $p_t\in[0,1]$
is interpreted as a prediction of the bounded outcome.  Retain the affine
encoding
\[
  \sigma_t=2Y_t-1,
  \qquad
  \mu_t=2p_t-1,
  \qquad
  r_t=2(Y_t-p_t).
\]
The Defensive Booster is unchanged; also note that outcomes in any other fixed bounded
interval reduce to this setting by affine rescaling.
It is now easy to see that the span guarantee remains the same.
\begin{proposition}[Certificate and span guarantee for bounded outcomes]
\label{prop:bounded-certificate-span}
For every adaptive sequence with $Y_t\in[0,1]$, the Defensive Booster
satisfies the multiaccuracy and self-orthogonality guarantees of
Theorem~\ref{thm:residual-certificate}, with
\[
  S_T=4\sum_{t=1}^T(Y_t-p_t)^2.
\]
In particular, writing
$L_{2,T}=T^{-1}\sum_t(Y_t-p_t)^2$, for every
$f\in\spanop_\Lambda(\calH)$,
\[
  L_{2,T}
  \le
  \frac1T\sum_{t=1}^T(Y_t-q_f(x_t))^2
  +
  \frac{\Lambda A_H+A_S}{\sqrt T}\sqrt{L_{2,T}}
  +
  \frac{\Lambda B_H+B_S}{2T}.
\]
The low-loss conclusion of Corollary~\ref{cor:low-loss-span} also holds with
$B_T$ and $B_f$ interpreted as the corresponding average squared errors.
\end{proposition}

\begin{proof}
Lemma~\ref{lem:root-sign} already permits every
$\sigma_t\in[-1,1]$.  The proof of
Theorem~\ref{thm:residual-certificate} therefore applies unchanged.  The
proof of Theorem~\ref{thm:span-guarantee} then uses only that certificate and
the same squared-loss expansion, so it also applies unchanged.
\end{proof}

\subsection{Regression experiments with bounded outcomes}
\label{sec:regression-experiments}

We test the Defensive Booster beyond binary outcomes on three public
regression datasets, processed in timestamp order.  \emph{Appliance energy}
records household appliance use
every ten minutes together with indoor and outdoor sensor measurements
\citep{Candanedo2017AppliancesDataset}.  \emph{Bike demand} records hourly
Capital Bikeshare rentals with calendar and weather variables
\citep{FanaeeT2013BikeDataset}.  \emph{Interstate traffic} records hourly
westbound I-94 traffic volume with calendar and weather variables
\citep{Hogue2019MetroDataset}.

We normalize each target by a fixed upper bound $R$: $2000$ Wh for Appliance
Energy, $2000$ rentals per hour for Bike Demand, and $10000$ vehicles per
hour for Interstate Traffic.  We map a target $Y_t\in[0,R]$ to
$Y_t/R\in[0,1]$ and map a normalized forecast $p_t$ back to $Rp_t$.  Thus
\[
  (Y_t/R-p_t)^2=(Y_t-Rp_t)^2/R^2,
\]
so normalized MSE and root mean squared error in the original units differ
only by the fixed factor $R$.  Every observed target lies in its stated
interval, so this normalization clips no outcome.

We use the first $10\%$ of each stream as a common chronological
initialization prefix and report losses on the remaining $90\%$.  Thereafter,
each algorithm receives the current context, predicts, observes the target,
and updates.  Numeric features are standardized using statistics from
strictly earlier rows.  Each context contains calendar variables and the
available sensor or weather measurements.  Bike and Traffic additionally use
targets observed exactly one hour, one day, and one week earlier when those
timestamps exist; Appliance Energy uses lags of ten minutes, one hour, and one
day.  We omit simultaneous light consumption and two random decoy columns
from Appliance Energy.  We omit the casual and registered rental counts from
Bike because they sum to the target, and collapse duplicate weather reports
at a Traffic timestamp.

All learned methods (except the past-outcome mean baseline) receive the same 128-dimensional signed feature-hashed context,
rescaled to have Euclidean norm at most one, and use the same unit-ball linear
weak class.  Algorithmic hyperparameters are fixed across datasets.  The
Defensive Booster maintains one second-order linear oracle; OGB maintains
$N=100$ such oracles
and uses the stage step $\eta=(\log N)/N$ from the binary experiments.  Two
controls use either one unboosted squared-loss learner or the mean of targets
observed before the current round.  Figure~\ref{fig:regression-mse}
plots cumulative mean squared error on the normalized targets, and
Table~\ref{tab:regression-results} reports final normalized MSE and root mean
squared error in the original units.

\begin{figure}[H]
\centering
\includegraphics[width=.98\linewidth]{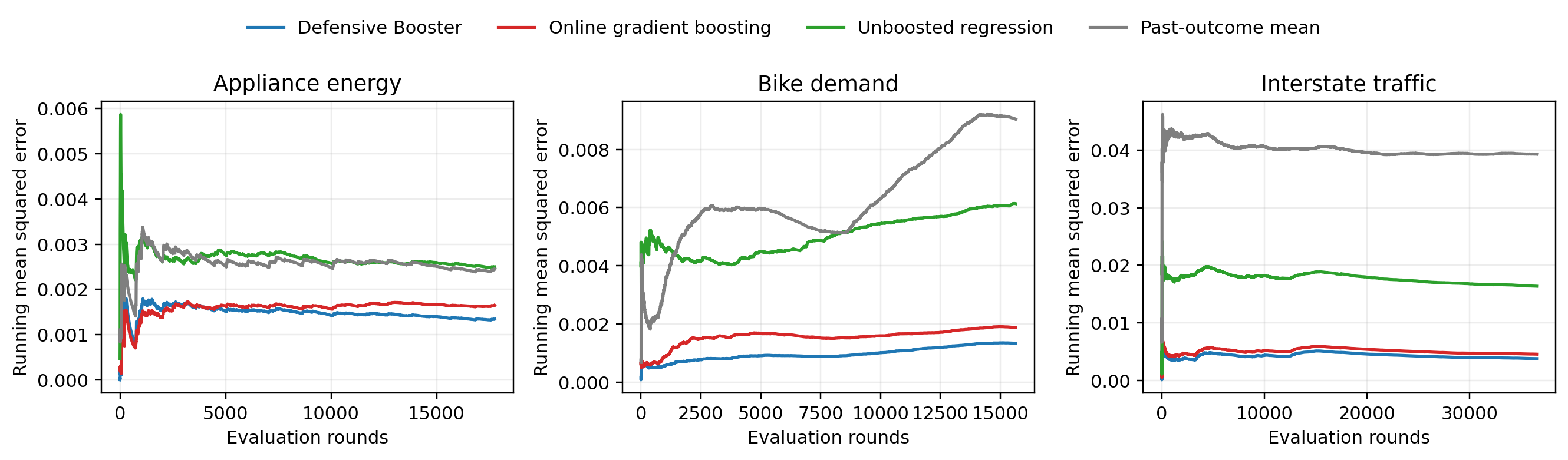}
\caption{Cumulative mean squared error after the common chronological
initialization prefix.  The Defensive Booster and unboosted control each
maintain one weak-class learner, whereas online gradient boosting maintains
$100$.  The Defensive Booster has the lowest final loss on all three streams.}
\label{fig:regression-mse}
\end{figure}

\begin{table}[H]
\centering
\small
\begin{tabular}{lrrrrrr}
\hline
& & \multicolumn{2}{c}{Normalized MSE} & \multicolumn{2}{c}{Raw-unit RMSE} & OGB time \\
Dataset & Rounds & Defensive & OGB & Defensive & OGB & / Defensive \\
\hline
Appliance energy & $17{,}762$ & $\mathbf{.0013443}$ & $.0016461$ & $\mathbf{73.330}$ & $81.145$ & $64.7\times$ \\
Bike demand & $15{,}642$ & $\mathbf{.0013440}$ & $.0018824$ & $\mathbf{73.322}$ & $86.773$ & $65.9\times$ \\
Interstate traffic & $36{,}518$ & $\mathbf{.0038040}$ & $.0045762$ & $\mathbf{616.769}$ & $676.473$ & $69.6\times$ \\
\hline
\end{tabular}
\caption{Final regression performance after the common chronological
initialization prefix; lower error is better.  Raw-unit RMSE is measured in
Wh, rentals per hour, and vehicles per hour, respectively.  The final column
is the ratio of OGB's wall-clock time per round to the Defensive Booster's on
the same machine.  Absolute runtimes depend on the implementation, whereas
OGB's $100$ maintained weak-class learners versus the Defensive Booster's one
is an algorithmic difference.}
\label{tab:regression-results}
\end{table}

The Defensive Booster reduces normalized MSE relative to OGB by $18\%$ on
Appliance Energy, $29\%$ on Bike Demand, and $17\%$ on Interstate Traffic.
On Bike and Traffic, both controls are substantially worse, so
the improvement does not come merely from predicting the running mean or from
applying the shared weak learner once.  Thus the empirical advantage extends
beyond binary outcomes: on each stream, the Defensive Booster obtains lower
squared error than the $100$-learner OGB ensemble while maintaining one
weak-class learner.

\section{Strongly adaptive experiments}
\label{sec:adaptive-diagnostics}
We next compare the basic Defensive Booster with its strongly adaptive
variant from
Section~\ref{sec:interval-boosting}.  The implementation uses the canonical
dyadic interval family and the Adapt-ML-Prod second-order aggregation rule
with the
sleeping-expert confidence reduction of
\citet{GaillardStoltzVanErven2014SecondOrder}.  The weak-class and scalar
states are vectorized across scales, which reduces implementation overhead
without changing their updates.  This variant introduces no dataset-specific
parameter: the same second-order weak oracle and scalar routines are used at
every scale.

The \textsc{Adaptive Def.} rows in Table~\ref{tab:real-data} report its full-stream
performance.  On Electricity, strong adaptivity lowers Brier loss from
$.0772$ to $.0644$, deterministic error from $.1077$ to $.0851$, and
randomized error from $.1538$ to $.1289$.  It also lowers Airlines Brier loss
from $.2094$ to $.2066$ and Occupancy Brier loss from $.0071$ to $.0069$; on
Bank it increases Brier loss from $.0800$ to $.0807$.  The adaptive
implementation takes $85$--$89$ microseconds per round, about six times the basic forecaster but
still substantially less than the 100-learner ensembles in
Table~\ref{tab:real-data}.
Figure~\ref{fig:adaptive-real} compares the methods' trailing-window losses
and shows when these full-stream improvements occur.

\begin{figure}[H]
\centering
\includegraphics[width=.96\linewidth]{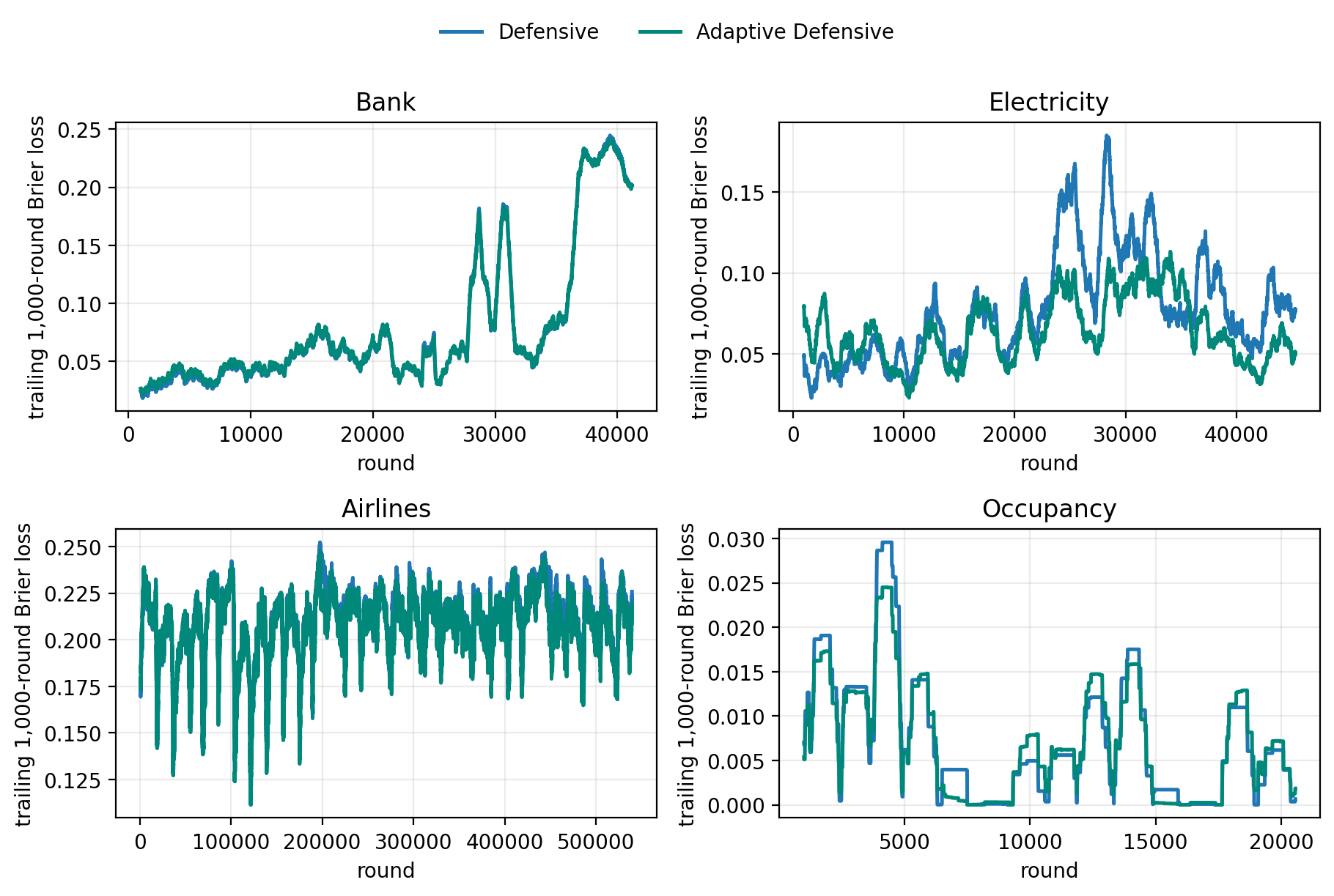}
\caption{Local performance of the basic and strongly adaptive Defensive
Boosters on all four real streams.  Each curve is the trailing
$1{,}000$-round Brier loss; the same window is fixed for every dataset.  The
adaptive variant tracks the basic forecaster closely on Bank and Airlines,
while on Electricity it improves substantially during the later high-loss
portions of the stream.  Occupancy is mixed locally but favors the adaptive variant in
full-stream loss.  Table~\ref{tab:real-data} reports the corresponding
full-stream averages.}
\label{fig:adaptive-real}
\end{figure}

\paragraph{Controlled drift benchmark.}
To isolate adaptation to distribution shift, we additionally use the INSECTS
optical-sensor benchmark of
\citet{SouzaDosReisMaletzkeBatista2020Benchmarking}.  Each example is an
optical-sensor recording of mosquito flight.  The benchmark orders these
examples using a hidden temperature variable to produce known abrupt,
incremental-gradual, and recurring drift patterns.  We use all five balanced
variants and preserve each released order.  Because our setting is binary, we fix one target for
the entire benchmark: recognize \emph{Aedes albopictus} (either sex) versus
\emph{Aedes aegypti} or \emph{Culex quinquefasciatus}.  Each of the six source
classes has equal frequency, so the binary target has positive rate $1/3$.
The 33 numeric signal features receive the same prefix-only standardization
and row normalization as the other real streams.  Neither algorithm is given
the temperature, the drift type, or the change points.
Table~\ref{tab:insects-drift} reports full-stream performance for all five
released orderings.

\begin{table}[H]
\centering
\small
\begin{tabular}{lrrrrr}
\hline
Drift pattern & Rounds & \multicolumn{2}{c}{Brier loss} & \multicolumn{2}{c}{$0/1$ error} \\
 & & Basic & Adaptive & Basic & Adaptive \\
\hline
Abrupt & $52{,}848$ & $.1307$ & $\mathbf{.1198}$ & $.1905$ & $\mathbf{.1697}$ \\
Incremental-gradual & $24{,}150$ & $.0905$ & $\mathbf{.0877}$ & $.1321$ & $\mathbf{.1235}$ \\
Incremental-abrupt recurring & $79{,}986$ & $.0871$ & $\mathbf{.0813}$ & $.1231$ & $\mathbf{.1148}$ \\
Incremental recurring & $79{,}986$ & $.0829$ & $\mathbf{.0769}$ & $.1151$ & $\mathbf{.1064}$ \\
Incremental & $57{,}018$ & $\mathbf{.1700}$ & $.1711$ & $\mathbf{.2535}$ & $.2547$ \\
\hline
\end{tabular}
\caption{Average online performance on five controlled-drift INSECTS streams;
lower is better.  ``Basic'' is the Defensive Booster, which maintains one
weak learner; ``Adaptive'' is its strongly adaptive variant.  The binary task
and all algorithmic choices are fixed across rows.}
\label{tab:insects-drift}
\end{table}

The adaptive variant lowers both metrics on the four streams that combine
abrupt, gradual, or recurring shifts.  On the continuously incremental stream,
the methods differ by at most $.0012$.  The adaptive implementation takes
$83$--$85$ microseconds per round.  On the four original real streams,
Table~\ref{tab:real-data} shows that the adaptive method
is slower than the basic forecaster but remains $3$--$10\times$ faster than
the 100-learner ensembles.
Figure~\ref{fig:adaptive-insects} shows the local Brier losses around the
published change points for two representative streams.

\begin{figure}[H]
\centering
\includegraphics[width=.96\linewidth]{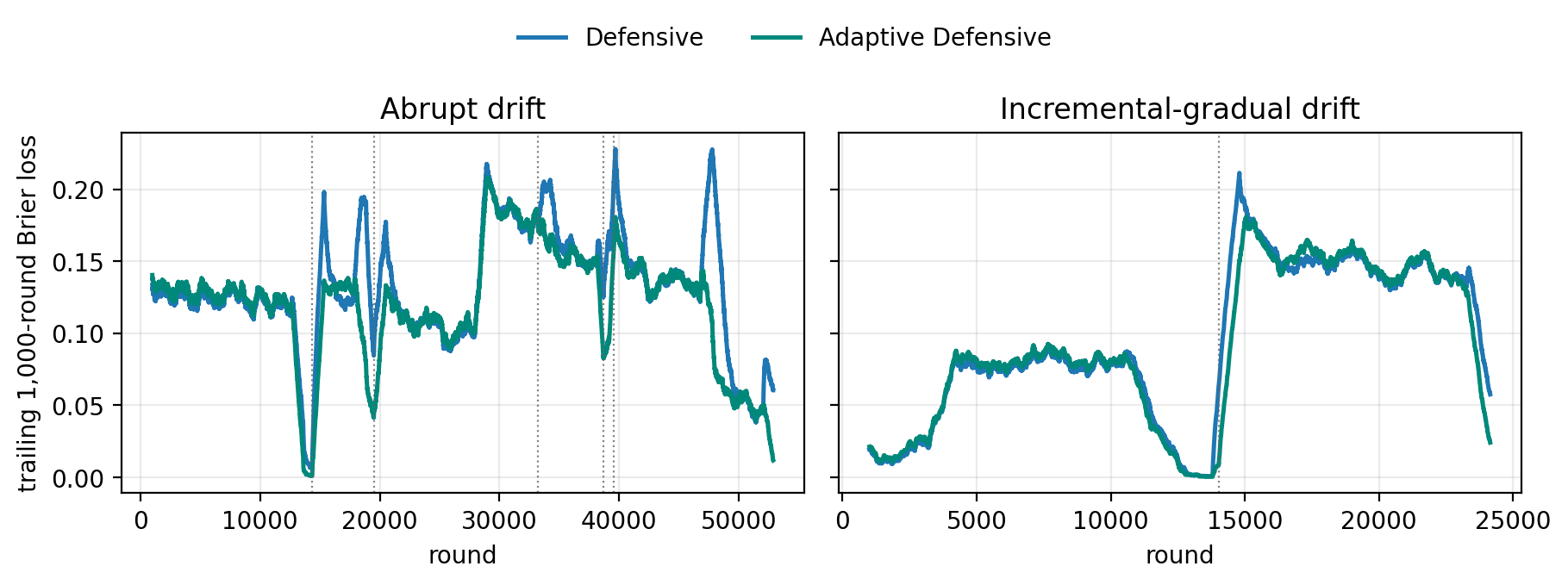}
\caption{Local Brier loss on the abrupt and incremental-gradual INSECTS
streams.  Curves are trailing $1{,}000$-round averages; dotted lines mark the
change points published with the benchmark.  The adaptive forecaster reduces
several of the largest post-shift loss spikes.}
\label{fig:adaptive-insects}
\end{figure}

The interval guarantee also produces local hard-core witnesses.
Figure~\ref{fig:adaptive-hard-core-evolution} examines these witnesses on the
abrupt INSECTS stream.  For each selected endpoint
$t$ and dyadic length $L$, it computes the mistake weighting on the trailing
interval $I=[t-L+1,t]$ and reports both quantities that define its hard-core
quality: density and normalized weak-class edge.

\begin{figure}[H]
\centering
\includegraphics[width=.98\linewidth]{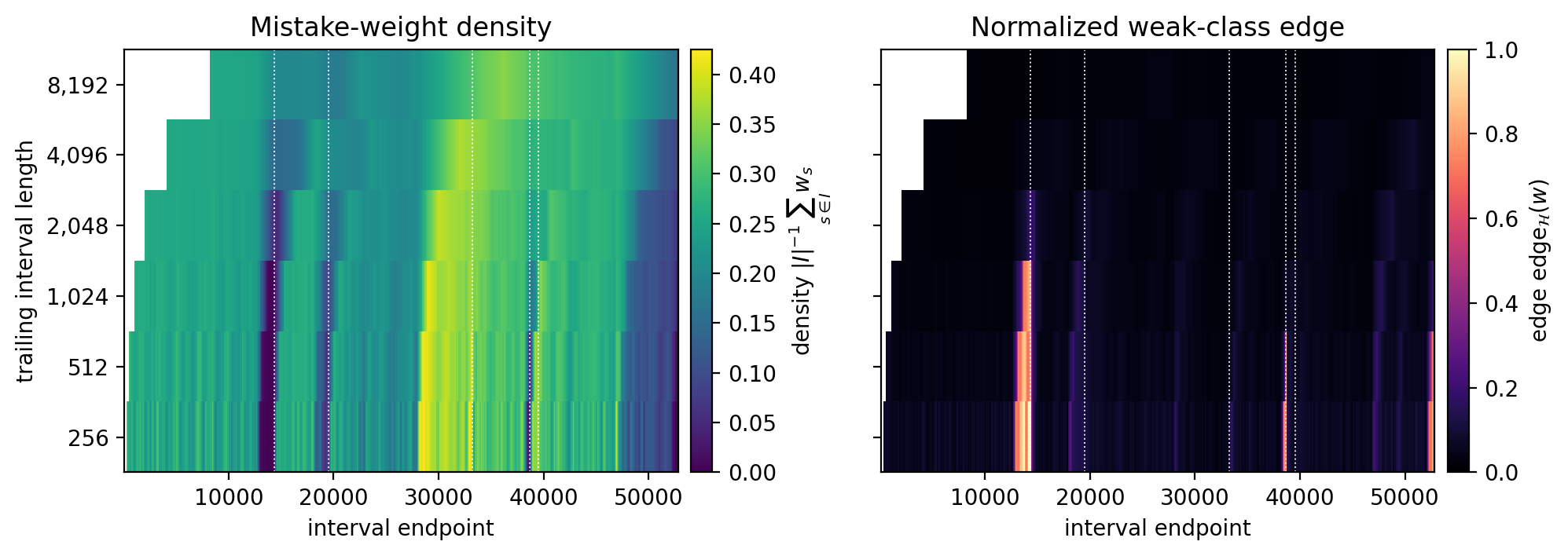}
\caption{Evolution of local hard-core witnesses for the strongly adaptive
Defensive Booster on the abrupt INSECTS stream.  Each pixel represents the
mistake weights $w_s=|Y_s-p_s|$ on one trailing interval
$I=[t-L+1,t]$: the horizontal axis is its endpoint $t$, and the vertical axis
is its dyadic length $L$.  The left panel gives the density
$|I|^{-1}\sum_{s\in I}w_s$; the right gives $\edge_{\calH}(w)$.  Bright regions
on the left paired with dark regions on the right are smooth, low-edge local
hard-core witnesses.  Dotted lines mark the five published abrupt change
points; white regions precede the first complete interval at a given scale.}
\label{fig:adaptive-hard-core-evolution}
\end{figure}

\end{document}